\documentclass[]{bytedance_seed}

\usepackage{airesearch}

\usepackage{wrapfig}
\setlist[itemize]{leftmargin=*,topsep=1pt,itemsep=1pt}
\usetikzlibrary{positioning, calc, shapes.geometric, arrows.meta, backgrounds, fit, matrix, decorations.pathreplacing, shadows.blur}

\definecolor{appleBlue}{RGB}{0, 122, 255}
\definecolor{appleLightBlue}{RGB}{235, 245, 255}
\definecolor{appleRed}{RGB}{255, 59, 48}
\definecolor{appleLightRed}{RGB}{255, 235, 235}
\definecolor{appleGreen}{RGB}{52, 199, 89}
\definecolor{appleLightGreen}{RGB}{235, 255, 235}
\definecolor{appleGray}{RGB}{142, 142, 147}
\definecolor{appleLightGray}{RGB}{242, 242, 247}
\definecolor{appleDarkGray}{RGB}{44, 44, 46}
\definecolor{appleDark}{RGB}{28, 28, 30}
\definecolor{applePurple}{RGB}{175, 82, 222}
\definecolor{appleLightPurple}{RGB}{243, 233, 255}
\definecolor{appleOrange}{RGB}{255, 149, 0}
\definecolor{appleLightOrange}{RGB}{255, 235, 205}
\definecolor{appleTeal}{RGB}{48, 176, 199}
\definecolor{appleLightTeal}{RGB}{213, 240, 245}
\usepackage{booktabs}
\usepackage{colortbl,xcolor}
\usepackage{siunitx}
\usepackage{array}     % for \newcolumntype
\usepackage{collcell}  % for \collectcell

\newcommand{\AvgCell}[1]{\cellcolor{red!8}\tablenum{#1}}
\newcolumntype{A}{>{\collectcell\AvgCell}c<{\endcollectcell}}

\usepackage{algorithm}
\usepackage{algpseudocode} % Provides the algorithmic commands like \Require, \State, etc.

\makeatletter
\newenvironment{breakablealgorithm}
  {%
    \par\addvspace{0.5\baselineskip}
    \noindent\refstepcounter{algorithm}%
    \hrule height .8pt\relax
    \kern 2pt
    \renewcommand{\caption}[2][\relax]{%
      {\raggedright\textbf{Algorithm \thealgorithm} ##2\par}%
      \ifx\relax##1\relax
        \addcontentsline{loa}{algorithm}{\protect\numberline{\thealgorithm}##2}%
      \else
        \addcontentsline{loa}{algorithm}{\protect\numberline{\thealgorithm}##1}%
      \fi
      \kern 2pt\hrule\kern 2pt
    }%
  }{%
    \kern 2pt\hrule
    \par\addvspace{0.5\baselineskip}
  }
\makeatother
\usepackage{dsfont}

\tikzset{
  section_label/.style={
    font=\sffamily\bfseries\small,
    text=appleGray,
    align=center
  },
  base_node/.style={
    draw=appleGray!55,
    fill=white,
    rounded corners=6pt,
    inner sep=5pt,
    outer sep=0pt,
    align=center,
    font=\sffamily\small
  },
  state/.style={
    base_node,
    draw=appleGray!60,
    fill=appleLightGray!25,
    minimum height=8mm,
    inner xsep=6pt,
    font=\sffamily\bfseries\scriptsize
  },
  tensor/.style={
    base_node,
    draw=appleGray!55,
    fill=white,
    rounded corners=5pt,
    minimum height=7mm,
    inner xsep=4pt,
    font=\small
  },
  value_tensor/.style={
    tensor,
    draw=appleRed!70!black,
    fill=appleRed!8
  },
  signal_arrow/.style={
    -{Latex[length=2.2mm,width=1.6mm]},
    line width=0.9pt,
    draw=appleBlue!75
  },
  arrow/.style={
    -{Latex[length=2.2mm,width=1.6mm]},
    line width=0.9pt,
    draw=appleGray!75
  }
}

\newcommand{\Falcon}{Falcon}

\title{Fast Weight Attention for Continual Learning}

\author[1,2,*]{Yifan Zhang}
\author[2]{Steve Ta}
\author[5]{Jasper Zhang}
\author[2]{Jichen Feng}
\author[2]{Shuzhen Li}
\author[1]{\\Yongxin Zhang}
\author[1]{Yifeng Liu}
\author[1]{Huizhuo Yuan}
\author[2, \dagger]{Mengdi Wang}
\author[1, \dagger]{Quanquan Gu}
\author[3, \dagger]{Andrew Chi-Chih Yao}

\affiliation[1]{Bytedance Seed}
\affiliation[2]{Princeton University}
\affiliation[3]{Tsinghua University}
\affiliation[4]{UCLA}
\affiliation[5]{Hyperbolic Labs}

\contribution[*]{Work was partially done when Yifan was at ByteDance Seed}
\contribution[\dagger]{Corresponding authors}

\abstract{
    Recurrent fast-weight memories and selective state-space models compress an expanding context into a fixed-size recurrent state, making the state transition an online learning rule. We study this rule under read-after-write autoregressive semantics. For the prefix-prediction objective considered here, the local fast-memory example revealed at step $t$ is the prefix-aligned pair $(\mathbf{x}_t,\mathbf{y}_t)=(\phi(\mathbf{k}_{t-1}),\mathbf{v}_t)$. The common same-step association $(\phi(\mathbf{k}_t),\mathbf{v}_t)$ remains causal, but optimizes a different internal objective. We derive normalized first-order updates for squared-error regression and negative inner-product objectives. The regression family comprises \Falcon-1 (a scalar NLMS update), \Falcon-2 (its per-column extension), and \Falcon-3 (a sliding-window mini-batch update); \Falcon-1A/\Falcon-2A/\Falcon-3A are the corresponding inner-product variants. We provide recurrent, masked-parallel, and chunk-parallel forms, together with numerically stable positive-decay renormalization. Representative variants remain competitive in language modeling and improve length extrapolation on variable-digit addition. This framework separates temporal alignment, plasticity, forgetting, and bounded rehearsal in recurrent sequence models.
}

\date{March 9, 2026}
\checkdata[Project Page]{\url{https://github.com/yifanzhang-pro/fast-weight-attention}}

\begin{document}
\maketitle

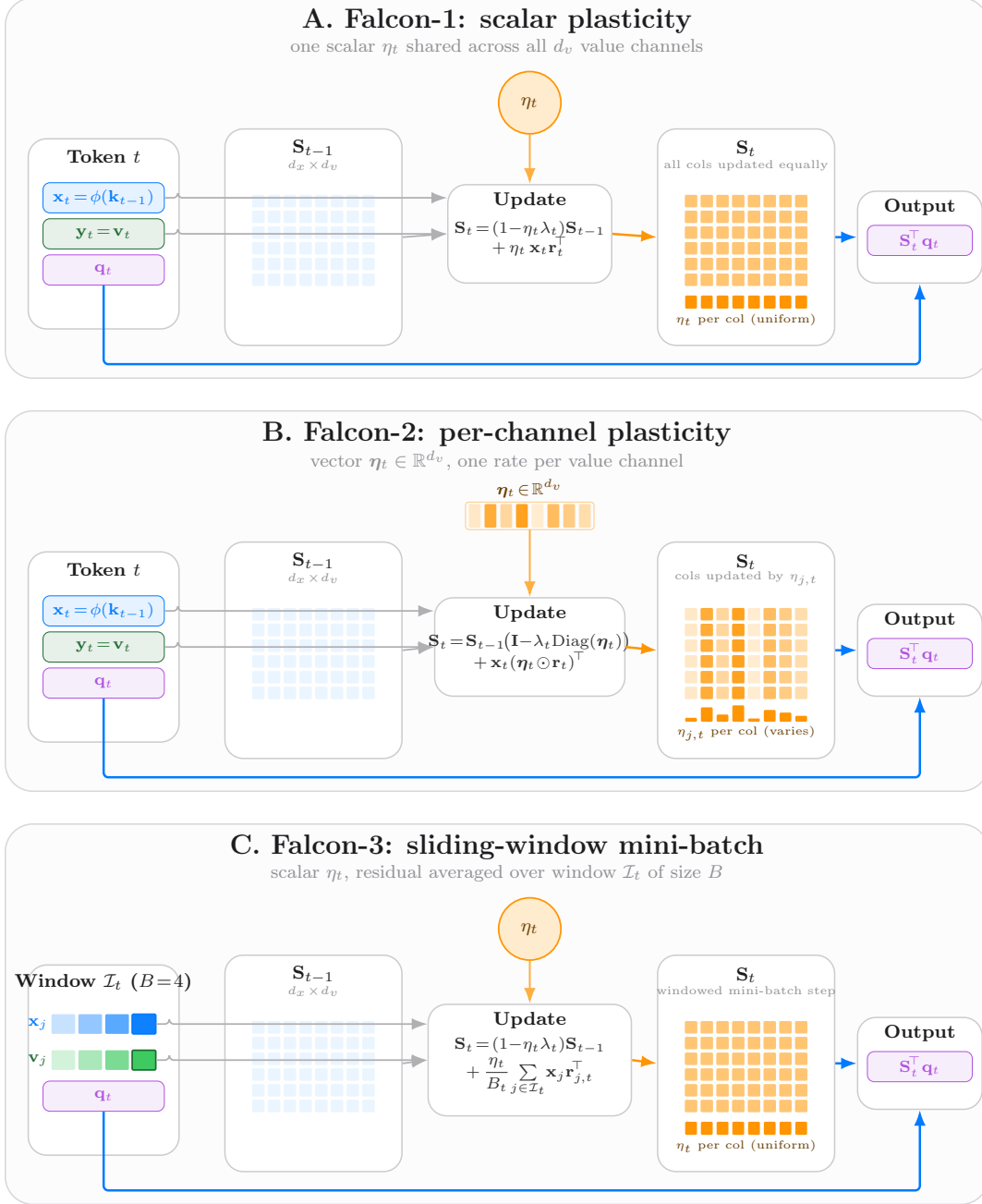
\begin{figure}[ht!]
\centering
\resizebox{0.85\linewidth}{!}{%
\begin{tikzpicture}[
    font=,
    >={Latex[length=2.4mm,width=1.8mm]},
    card/.style={rounded corners=8pt, fill=white, draw=appleGray!45, line width=0.6pt},
    pill/.style 2 args={rounded corners=3pt, fill=#1!10, draw=#1, line width=0.4pt,
                        minimum width=#2, minimum height=0.46cm,
                        font=\scriptsize, text=#1, inner sep=2pt},
    arr/.style={->, line width=0.7pt, draw=appleGray!70, rounded corners=2pt},
    arrW/.style={->, line width=1pt, draw=appleOrange, rounded corners=2pt},
    arrR/.style={->, line width=1pt, draw=appleBlue, rounded corners=2pt},
    arrEta/.style={->, line width=0.8pt, draw=appleOrange!75, rounded corners=2pt},
    panelbg/.style={rounded corners=12pt, fill=appleLightGray!22, draw=appleGray!35, line width=0.6pt},
]

% ===== Panel backgrounds =====
\begin{scope}[on background layer]
  \node[panelbg, minimum width=15cm, minimum height=5.8cm] at (7.5,  4.0) {}; % A
  \node[panelbg, minimum width=15cm, minimum height=5.8cm] at (7.5, -2.3) {}; % B
  \node[panelbg, minimum width=15cm, minimum height=5.8cm] at (7.5, -8.6) {}; % C
\end{scope}

% ============================================================
% PANEL A: FALCON-1 -- scalar plasticity
% ============================================================
\node[font=\large\bfseries, text=appleDark] at (7.5, 6.55)
    {A. Falcon-1: scalar plasticity};
\node[font=\footnotesize, text=appleGray] at (7.5, 6.15)
    {one scalar $\eta_t$ shared across all $d_v$ value channels};

\node[card, minimum width=2.3cm, minimum height=2.9cm] (inputA) at (1.5, 3.3) {};
\node[font=\footnotesize\bfseries, text=appleDark] at ($(inputA.north)+(0,-0.28)$) {Token $t$};
\node[pill={appleBlue}{1.85cm}] (xtA) at ($(inputA.center)+(0,0.55)$)
    {$\mathbf{x}_t\!=\!\phi(\mathbf{k}_{t-1})$};
\node[pill={appleGreen!55!black}{1.85cm}] (ytA) at ($(inputA.center)+(0,0.00)$)
    {$\mathbf{y}_t\!=\!\mathbf{v}_t$};
\node[pill={applePurple}{1.85cm}] (qtA) at ($(inputA.center)+(0,-0.55)$)
    {$\mathbf{q}_t$};

\def\cw{0.24}\def\rh{0.24}
\pgfmathsetmacro{\matAX}{4.7}\pgfmathsetmacro{\matAY}{3.25}
\node[card, minimum width=2.7cm, minimum height=3.3cm] (stateAprev) at (\matAX, \matAY) {};
\node[font=\footnotesize\bfseries, text=appleDark] at (\matAX, \matAY+1.35) {$\mathbf{S}_{t-1}$};
\node[font=\tiny, text=appleGray] at (\matAX, \matAY+1.10) {$d_x\!\times\!d_v$};
\pgfmathsetmacro{\gridAleft}{\matAX - 4*\cw}
\pgfmathsetmacro{\gridAbot}{\matAY - 0.05 - 3*\rh}
\foreach \r in {0,...,5}{\foreach \c in {0,...,7}{%
  \pgfmathsetmacro{\cx}{\gridAleft + \c*\cw}\pgfmathsetmacro{\cy}{\gridAbot + \r*\rh}%
  \fill[appleLightBlue, rounded corners=0.6pt]
    (\cx+0.025, \cy+0.025) rectangle (\cx+\cw-0.025, \cy+\rh-0.025);}}

\node[circle, draw=appleOrange, fill=appleLightOrange, line width=0.7pt,
      minimum size=0.95cm, font=\footnotesize\bfseries, text=appleOrange!40!black]
    (etaA) at (8.0, 5.30) {$\eta_t$};

\node[card, minimum width=2.5cm, minimum height=1.5cm] (formulaA) at (8.0, 3.3) {};
\node[font=\footnotesize\bfseries, text=appleDark] at ($(formulaA.north)+(0,-0.25)$) {Update};
\node[font=\scriptsize, text=appleDarkGray, align=center] at ($(formulaA.center)+(0,-0.07)$)
    {$\mathbf{S}_t\!=\!(1{-}\eta_t\lambda_t)\mathbf{S}_{t-1}$\\[-1pt]
     $+\,\eta_t\,\mathbf{x}_t\mathbf{r}_t^{\!\top}$};

\pgfmathsetmacro{\matBX}{11.3}
\node[card, minimum width=2.7cm, minimum height=3.3cm] (stateAcur) at (\matBX, \matAY) {};
\node[font=\footnotesize\bfseries, text=appleDark] at (\matBX, \matAY+1.35) {$\mathbf{S}_t$};
\node[font=\tiny, text=appleGray] at (\matBX, \matAY+1.10) {all cols updated equally};
\pgfmathsetmacro{\gridBleft}{\matBX - 4*\cw}
\foreach \r in {0,...,5}{\foreach \c in {0,...,7}{%
  \pgfmathsetmacro{\cx}{\gridBleft + \c*\cw}\pgfmathsetmacro{\cy}{\gridAbot + \r*\rh}%
  \fill[appleOrange!55, rounded corners=0.6pt]
    (\cx+0.025, \cy+0.025) rectangle (\cx+\cw-0.025, \cy+\rh-0.025);}}

\pgfmathsetmacro{\barAY}{\gridAbot - 0.32}
\foreach \c in {0,...,7}{\pgfmathsetmacro{\cx}{\gridBleft + \c*\cw}%
  \fill[appleOrange, rounded corners=0.6pt]
    (\cx+0.03, \barAY) rectangle (\cx+\cw-0.03, \barAY+0.20);}
\node[font=\tiny, text=appleOrange!40!black] at (\matBX, \barAY-0.17)
    {$\eta_t$ per col (uniform)};

\node[card, minimum width=1.9cm, minimum height=1.4cm] (outputA) at (13.95, \matAY) {};
\node[font=\footnotesize\bfseries, text=appleDark] at ($(outputA.north)+(0,-0.25)$) {Output};
\node[pill={applePurple}{1.6cm}] at ($(outputA.center)+(0,-0.05)$)
    {$\mathbf{S}_t^{\!\top}\mathbf{q}_t$};

\draw[arr]    (xtA.east) -- ++(0.20,0) |- (formulaA.west |- xtA);
\draw[arr]    (ytA.east) -- ++(0.20,0) |- (formulaA.west |- ytA);
\draw[arr]    (stateAprev.east) -- (formulaA.west);
\draw[arrW]   (formulaA.east)   -- (stateAcur.west);
\draw[arrR]   (stateAcur.east)  -- (outputA.west);
\draw[arrEta] (etaA.south)      -- (formulaA.north);
\draw[arrR]   (qtA.south) -- ++(0,-1.20) -| (outputA.south);

% ============================================================
% PANEL B: FALCON-2 -- per-channel plasticity (yshift=-6.3cm)
% ============================================================
\begin{scope}[yshift=-6.3cm]

\node[font=\large\bfseries, text=appleDark] at (7.5, 6.55)
    {B. Falcon-2: per-channel plasticity};
\node[font=\footnotesize, text=appleGray] at (7.5, 6.15)
    {vector $\boldsymbol{\eta}_t\in\mathbb{R}^{d_v}$, one rate per value channel};

\node[card, minimum width=2.3cm, minimum height=2.9cm] (inputB) at (1.5, 3.3) {};
\node[font=\footnotesize\bfseries, text=appleDark] at ($(inputB.north)+(0,-0.28)$) {Token $t$};
\node[pill={appleBlue}{1.85cm}] (xtB) at ($(inputB.center)+(0,0.55)$)
    {$\mathbf{x}_t\!=\!\phi(\mathbf{k}_{t-1})$};
\node[pill={appleGreen!55!black}{1.85cm}] (ytB) at ($(inputB.center)+(0,0.00)$)
    {$\mathbf{y}_t\!=\!\mathbf{v}_t$};
\node[pill={applePurple}{1.85cm}] (qtB) at ($(inputB.center)+(0,-0.55)$)
    {$\mathbf{q}_t$};

\node[card, minimum width=2.7cm, minimum height=3.3cm] (stateBprev) at (\matAX, \matAY) {};
\node[font=\footnotesize\bfseries, text=appleDark] at (\matAX, \matAY+1.35) {$\mathbf{S}_{t-1}$};
\node[font=\tiny, text=appleGray] at (\matAX, \matAY+1.10) {$d_x\!\times\!d_v$};
\foreach \r in {0,...,5}{\foreach \c in {0,...,7}{%
  \pgfmathsetmacro{\cx}{\gridAleft + \c*\cw}\pgfmathsetmacro{\cy}{\gridAbot + \r*\rh}%
  \fill[appleLightBlue, rounded corners=0.6pt]
    (\cx+0.025, \cy+0.025) rectangle (\cx+\cw-0.025, \cy+\rh-0.025);}}

% Vector \eta_t shown as a horizontal strip with per-channel intensities
\def\etaBY{5.30}
\node[font=\scriptsize\bfseries, text=appleOrange!40!black]
    at (8.0, \etaBY+0.42) {$\boldsymbol{\eta}_t\!\in\!\mathbb{R}^{d_v}$};
\foreach \i/\op in {0/22, 1/78, 2/38, 3/88, 4/18, 5/62, 6/48, 7/32}{%
  \pgfmathsetmacro{\cx}{8.0 - 4*\cw + \i*\cw}%
  \fill[appleOrange!\op, rounded corners=0.6pt]
    (\cx+0.03, \etaBY-0.18) rectangle (\cx+\cw-0.03, \etaBY+0.18);}
\pgfmathsetmacro{\etaBL}{8.0-4*\cw-0.02}
\pgfmathsetmacro{\etaBR}{8.0+4*\cw+0.02}
\pgfmathsetmacro{\etaBT}{\etaBY+0.21}
\pgfmathsetmacro{\etaBB}{\etaBY-0.21}
\draw[draw=appleOrange!50, line width=0.4pt, rounded corners=1.5pt]
    (\etaBL, \etaBB) rectangle (\etaBR, \etaBT);

\node[card, minimum width=2.9cm, minimum height=1.5cm] (formulaB) at (8.0, 3.3) {};
\node[font=\footnotesize\bfseries, text=appleDark] at ($(formulaB.north)+(0,-0.25)$) {Update};
\node[font=\scriptsize, text=appleDarkGray, align=center] at ($(formulaB.center)+(0,-0.07)$)
    {$\mathbf{S}_t\!=\!\mathbf{S}_{t-1}\!\bigl(\mathbf{I}{-}\lambda_t\!\operatorname{Diag}(\boldsymbol{\eta}_t)\bigr)$\\[-1pt]
     $+\,\mathbf{x}_t(\boldsymbol{\eta}_t\!\odot\!\mathbf{r}_t)^{\!\top}$};

\node[card, minimum width=2.7cm, minimum height=3.3cm] (stateBcur) at (\matBX, \matAY) {};
\node[font=\footnotesize\bfseries, text=appleDark] at (\matBX, \matAY+1.35) {$\mathbf{S}_t$};
\node[font=\tiny, text=appleGray] at (\matBX, \matAY+1.10) {cols updated by $\eta_{j,t}$};
\foreach \c/\op in {0/22, 1/78, 2/38, 3/88, 4/18, 5/62, 6/48, 7/32}{%
  \foreach \r in {0,...,5}{%
    \pgfmathsetmacro{\cx}{\gridBleft + \c*\cw}\pgfmathsetmacro{\cy}{\gridAbot + \r*\rh}%
    \fill[appleOrange!\op, rounded corners=0.6pt]
      (\cx+0.025, \cy+0.025) rectangle (\cx+\cw-0.025, \cy+\rh-0.025);}}

\foreach \c/\h in {0/0.06, 1/0.22, 2/0.11, 3/0.25, 4/0.05, 5/0.18, 6/0.14, 7/0.09}{%
  \pgfmathsetmacro{\cx}{\gridBleft + \c*\cw}%
  \fill[appleOrange, rounded corners=0.6pt]
    (\cx+0.03, \barAY) rectangle (\cx+\cw-0.03, \barAY+\h);}
\node[font=\tiny, text=appleOrange!40!black] at (\matBX, \barAY-0.17)
    {$\eta_{j,t}$ per col (varies)};

\node[card, minimum width=1.9cm, minimum height=1.4cm] (outputB) at (13.95, \matAY) {};
\node[font=\footnotesize\bfseries, text=appleDark] at ($(outputB.north)+(0,-0.25)$) {Output};
\node[pill={applePurple}{1.6cm}] at ($(outputB.center)+(0,-0.05)$)
    {$\mathbf{S}_t^{\!\top}\mathbf{q}_t$};

\draw[arr]    (xtB.east) -- ++(0.20,0) |- (formulaB.west |- xtB);
\draw[arr]    (ytB.east) -- ++(0.20,0) |- (formulaB.west |- ytB);
\draw[arr]    (stateBprev.east) -- (formulaB.west);
\draw[arrW]   (formulaB.east)   -- (stateBcur.west);
\draw[arrR]   (stateBcur.east)  -- (outputB.west);
\draw[arrEta] (8.0, \etaBB)     -- (formulaB.north);
\draw[arrR]   (qtB.south) -- ++(0,-1.20) -| (outputB.south);

\end{scope}

% ============================================================
% PANEL C: FALCON-3 -- sliding-window mini-batch (yshift=-12.6cm)
% ============================================================
\begin{scope}[yshift=-12.6cm]

\node[font=\large\bfseries, text=appleDark] at (7.5, 6.55)
    {C. Falcon-3: sliding-window mini-batch};
\node[font=\footnotesize, text=appleGray] at (7.5, 6.15)
    {scalar $\eta_t$, residual averaged over window $\mathcal{I}_t$ of size $B$};

% Input card -- Window with explicit B mini-cells
\node[card, minimum width=2.3cm, minimum height=2.9cm] (inputC) at (1.5, 3.3) {};
\node[font=\footnotesize\bfseries, text=appleDark] at ($(inputC.north)+(0,-0.28)$)
    {Window $\mathcal{I}_t$ ($B\!=\!4$)};

% Window-cell geometry (4 cells, 0.36cm wide, 0.05cm gap)
\pgfmathsetmacro{\miniW}{0.36}
\pgfmathsetmacro{\miniG}{0.05}
\pgfmathsetmacro{\stripL}{1.5 - 2*\miniW - 1.5*\miniG}
\pgfmathsetmacro{\xLastC}{\stripL + 3*(\miniW+\miniG)}

% --- X-window: 4 blue mini-cells, gradient (older -> newer) ---
\foreach \i/\op in {0/22, 1/42, 2/65, 3/95}{%
  \pgfmathsetmacro{\xL}{\stripL + \i*(\miniW+\miniG)}%
  \fill[appleBlue!\op, rounded corners=0.6pt]
    (\xL, 3.70) rectangle (\xL+\miniW, 4.00);}
% Latest cell (j=t): thicker border
\draw[appleBlue, line width=0.7pt, rounded corners=0.6pt]
  (\xLastC, 3.70) rectangle (\xLastC+\miniW, 4.00);
% Phantom anchor matching strip right edge
\node[draw=none, fill=none, inner sep=0pt,
      minimum width=1.59cm, minimum height=0.30cm]
  (xtC) at (1.5, 3.85) {};
\node[font=\scriptsize\bfseries, text=appleBlue] at (0.50, 3.85) {$\mathbf{x}_j$};

% --- V-window: 4 green mini-cells, same gradient ---
\foreach \i/\op in {0/22, 1/42, 2/65, 3/95}{%
  \pgfmathsetmacro{\xL}{\stripL + \i*(\miniW+\miniG)}%
  \fill[appleGreen!\op, rounded corners=0.6pt]
    (\xL, 3.15) rectangle (\xL+\miniW, 3.45);}
\draw[appleGreen!55!black, line width=0.7pt, rounded corners=0.6pt]
  (\xLastC, 3.15) rectangle (\xLastC+\miniW, 3.45);
\node[draw=none, fill=none, inner sep=0pt,
      minimum width=1.59cm, minimum height=0.30cm]
  (ytC) at (1.5, 3.30) {};
\node[font=\scriptsize\bfseries, text=appleGreen!55!black] at (0.50, 3.30) {$\mathbf{v}_j$};

% q_t pill (single, regular)
\node[pill={applePurple}{1.85cm}] (qtC) at (1.5, 2.75) {$\mathbf{q}_t$};

% S_{t-1} grid
\node[card, minimum width=2.7cm, minimum height=3.3cm] (stateCprev) at (\matAX, \matAY) {};
\node[font=\footnotesize\bfseries, text=appleDark] at (\matAX, \matAY+1.35) {$\mathbf{S}_{t-1}$};
\node[font=\tiny, text=appleGray] at (\matAX, \matAY+1.10) {$d_x\!\times\!d_v$};
\foreach \r in {0,...,5}{\foreach \c in {0,...,7}{%
  \pgfmathsetmacro{\cx}{\gridAleft + \c*\cw}\pgfmathsetmacro{\cy}{\gridAbot + \r*\rh}%
  \fill[appleLightBlue, rounded corners=0.6pt]
    (\cx+0.025, \cy+0.025) rectangle (\cx+\cw-0.025, \cy+\rh-0.025);}}

% scalar \eta_t
\node[circle, draw=appleOrange, fill=appleLightOrange, line width=0.7pt,
      minimum size=0.95cm, font=\footnotesize\bfseries, text=appleOrange!40!black]
    (etaC) at (8.0, 5.30) {$\eta_t$};

% Update formula C (windowed mini-batch)
\node[card, minimum width=3.1cm, minimum height=1.7cm] (formulaC) at (8.0, 3.3) {};
\node[font=\footnotesize\bfseries, text=appleDark] at ($(formulaC.north)+(0,-0.25)$) {Update};
\node[font=\scriptsize, text=appleDarkGray, align=center] at ($(formulaC.center)+(0,-0.10)$)
    {$\mathbf{S}_t\!=\!(1{-}\eta_t\lambda_t)\mathbf{S}_{t-1}$\\[1pt]
     $+\,\dfrac{\eta_t}{B_t}\!\sum\limits_{j\in\mathcal{I}_t}\!\mathbf{x}_j\mathbf{r}_{j,t}^{\!\top}$};

% S_t (uniform orange columns, mini-batch)
\node[card, minimum width=2.7cm, minimum height=3.3cm] (stateCcur) at (\matBX, \matAY) {};
\node[font=\footnotesize\bfseries, text=appleDark] at (\matBX, \matAY+1.35) {$\mathbf{S}_t$};
\node[font=\tiny, text=appleGray] at (\matBX, \matAY+1.10) {windowed mini-batch step};
\foreach \r in {0,...,5}{\foreach \c in {0,...,7}{%
  \pgfmathsetmacro{\cx}{\gridBleft + \c*\cw}\pgfmathsetmacro{\cy}{\gridAbot + \r*\rh}%
  \fill[appleOrange!55, rounded corners=0.6pt]
    (\cx+0.025, \cy+0.025) rectangle (\cx+\cw-0.025, \cy+\rh-0.025);}}

% Per-column \eta bar (uniform, scalar)
\foreach \c in {0,...,7}{\pgfmathsetmacro{\cx}{\gridBleft + \c*\cw}%
  \fill[appleOrange, rounded corners=0.6pt]
    (\cx+0.03, \barAY) rectangle (\cx+\cw-0.03, \barAY+0.20);}
\node[font=\tiny, text=appleOrange!40!black] at (\matBX, \barAY-0.17)
    {$\eta_t$ per col (uniform)};

% Output C
\node[card, minimum width=1.9cm, minimum height=1.4cm] (outputC) at (13.95, \matAY) {};
\node[font=\footnotesize\bfseries, text=appleDark] at ($(outputC.north)+(0,-0.25)$) {Output};
\node[pill={applePurple}{1.6cm}] at ($(outputC.center)+(0,-0.05)$)
    {$\mathbf{S}_t^{\!\top}\mathbf{q}_t$};

% Arrows C
\draw[arr]    (xtC.east) -- ++(0.20,0) |- (formulaC.west |- xtC);
\draw[arr]    (ytC.east) -- ++(0.20,0) |- (formulaC.west |- ytC);
\draw[arr]    (stateCprev.east) -- (formulaC.west);
\draw[arrW]   (formulaC.east)   -- (stateCcur.west);
\draw[arrR]   (stateCcur.east)  -- (outputC.west);
\draw[arrEta] (etaC.south)      -- (formulaC.north);
\draw[arrR]   (qtC.south) -- ++(0,-1.20) -| (outputC.south);

\end{scope}

\end{tikzpicture}}
\caption{\textbf{Falcon-1 vs.\ Falcon-2 vs.\ Falcon-3.}
Each panel shows one fast-weight update step on the fixed-size matrix state
$\mathbf{S}_{t-1}\!\in\!\mathbb{R}^{d_x\times d_v}$. We write
$\mathbf{q}_t:=\phi(\qb_t)$ for the query feature. In panels (A)--(B),
$\mathbf{r}_t:=\mathbf{v}_t-\mathbf{S}_{t-1}^{\!\top}\mathbf{x}_t$.
\textbf{(A)~Falcon-1} applies one scalar $\eta_t$ to all $d_v$ value channels;
every column of $\mathbf{S}_t$ receives the same plasticity.
\textbf{(B)~Falcon-2} promotes the step size to a vector
$\boldsymbol{\eta}_t\!\in\!\mathbb{R}^{d_v}$ shown as the orange strip above the
update card: each column of $\mathbf{S}_t$ is updated with its own gain
$\eta_{j,t}$, while the write feature $\mathbf{x}_t$ and residual
$\mathbf{r}_t$ remain shared.
\textbf{(C)~Falcon-3} uses a scalar $\eta_t$ and a sliding window
$\mathcal{I}_t$ of $B$ recent causal pairs, with window residuals
$\mathbf{r}_{j,t}:=\mathbf{v}_j-\mathbf{S}_{t-1}^{\!\top}\mathbf{x}_j$;
the blue/green strips show four
$\{\mathbf{x}_j,\mathbf{v}_j\}_{j\in\mathcal{I}_t}$ entries. The read
$\mathbf{o}_t=\mathbf{S}_t^{\!\top}\mathbf{q}_t$ uses the updated state in all three cases.}
\label{fig:falcon123-overview}
\end{figure}

%%%%%%%%%%%%%%%%%%%%%%%%%%%%%%
\section{Introduction}

Transformers \citep{vaswani2017attention} dominate modern language modeling because self-attention effectively captures global dependencies. Their main limitation is cost: standard attention scales quadratically, $\mathcal{O}(N^2)$, in sequence length $N$. For long contexts, both the attention matrix and the memory traffic of the key-value (KV) cache become major bottlenecks.

Beyond efficiency, long-context modeling is also a continual-learning problem: the model must bind new evidence online without catastrophic interference. Transformers externalize this fast memory as a growing KV cache, whereas SSMs and fast-weight models compress it into a fixed-size recurrent state. In such architectures, the state-update rule acts as a local learning rule, and its temporal alignment determines whether the fast memory is trained on information that was actually available at prediction time~\citep{sun2024learning, liu2024longhorn, behrouz2024titans, wang2025test}.

Linear Attention~\citep{katharopoulos2020transformers}, Fast Weight Programmers and Delta Networks~\citep{schlag2021linear}, RWKV~\citep{peng2023rwkv}, and Mamba~\citep{gu2023mamba} show that recurrent alternatives can be competitive while maintaining $\mathcal{O}(N)$ training and $\mathcal{O}(1)$ per-step inference. Their state-update equations are often presented at the architectural level, leaving the local objective and temporal alignment implicit.

Once read-after-write semantics are fixed, there is a convention mismatch: many recurrences bind the same-step pair $(\phi(\mathbf{k}_t),\mathbf{v}_t)$ (or $(\mathbf{k}_t,\mathbf{v}_t)$ when $\phi$ is the identity), whereas the prefix-prediction objective studied here pairs the newly revealed target with the prefix feature that was available when it was predicted, namely $(\phi(\mathbf{k}_{t-1}),\mathbf{v}_t)$, or equivalently $(\phi(\mathbf{k}_i),\mathbf{v}_{i+1})$ under standard indexing. The same-step pairing is still causal, but it optimizes a different internal fast-memory objective.

We therefore recast state-based sequence modeling as \textit{autoregressive next-latent prediction}. The recurrent state $\Sbb_t$ acts as a fast linear predictor from the prefix write feature $\mathbf{x}_t:=\phi(\mathbf{k}_{t-1})$ to the newly revealed target $\vb_t$. Building on Titans and ATLAS-style internal-memory views~\citep{behrouz2024titans, behrouz2025atlas}, we make this alignment explicit and derive normalized first-order updates that remain compatible with SSD-style chunk-parallel training~\citep{dao2024transformers}. Our contributions are as follows:

\begin{itemize}
\item We identify the fast-memory training pair induced by read-after-write autoregressive modeling and separate it from the common same-step association.
\item We derive \Falcon-1, \Falcon-2, and \Falcon-3 as normalized first-order regression writes. Scalar or per-column gains $\beta$ control plasticity, $\lambda_t$ controls shrinkage, and the realized step sizes $\eta$ are matched to the appropriate local smoothness scale.
\item We derive inner-product counterparts, \Falcon-1A, \Falcon-2A, and \Falcon-3A, using the same plasticity/forgetting semantics but interpreting the denominator as write-magnitude normalization rather than curvature normalization. Here \Falcon-2A is the per-column inner-product analogue of \Falcon-2.
\item We give chunk-parallel implementations for the resulting recurrences and evaluate representative regression and inner-product variants in language modeling and arithmetic extrapolation.
\end{itemize}

\section{Background}

\subsection{State Space Models}
\label{sec:background_ssm}

\noindent
State-space representations have a long history in classical control, filtering, realization theory, and system identification.
Twentieth-century foundations include Kalman filtering, state-variable realization, optimal filtering, and linear-systems theory~\citep{kalman1960new,kalman1961new,ho1966effective, kung1978new, kung1981state}.
Modern neural State Space Models (SSMs)~\citep{gu2023mamba,dao2024transformers} process a sequence of inputs $x(t) \in \mathbb{R}^{d_{\text{in}}}$ through a compressed latent state $\hb(t) \in \mathbb{R}^{n}$, producing outputs $y(t)\in\mathbb{R}^{d_{\text{out}}}$.
A continuous-time state-space representation has the form
\[
\dot{\hb}(t)
= \Ab(t)\hb(t) + \Bb(t)x(t),
\qquad
y(t) = \Cb(t)^\top \hb(t),
\]
where $\Ab(t)\in\mathbb{R}^{n\times n}$, $\Bb(t)\in\mathbb{R}^{n\times d_{\text{in}}}$, and $\Cb(t)\in\mathbb{R}^{n\times d_{\text{out}}}$. Conditional on these coefficients, the dynamics are linear in the hidden state; the coefficients themselves may be fixed, time varying, or data dependent.

Modern selective SSMs, such as Mamba~\citep{gu2023mamba}, parametrize $(\Bb, \Cb, \Delta)$ as functions of the current input $x_t$, allowing for content-aware filtering. Mamba-2~\citep{dao2024transformers} further simplifies the transition matrix $\Ab$ to a scalar or diagonal structure, establishing a theoretical bridge known as Structured State Space Duality (SSD). This duality demonstrates that the recurrent SSM is equivalent to a specific form of causal linear attention, enabling efficient training via chunk-parallel matrix multiplications while maintaining constant-state inference.

\subsection{Linear Attention}
\label{sec:background_linattn}
\noindent
Linear Attention~\citep{katharopoulos2020transformers} circumvents the $\mathcal{O}(N^2)$ complexity of standard attention by replacing the softmax with a kernel feature map $\phi(\cdot): \mathbb{R}^d \to \mathbb{R}^{m}$ such that $\kappa(\qb_t,\mathbf{k}_j) = \phi(\qb_t)^\top \phi(\mathbf{k}_j)$. Exploiting the associativity of matrix multiplication, the output $\ob_t \in \mathbb{R}^{d_v}$ for the $t$-th token is:
\begin{equation}
\ob_t = \frac{\sum_{j=1}^t \phi(\qb_t)^\top \phi(\mathbf{k}_j)\vb_j}{\sum_{j=1}^t \phi(\qb_t)^\top \phi(\mathbf{k}_j)}
= \frac{\left(\sum_{j=1}^t \phi(\mathbf{k}_j)\vb_j^\top\right)^\top\phi(\qb_t)}{\phi(\qb_t)^\top \sum_{j=1}^t \phi(\mathbf{k}_j)},
\label{eq:lin_attn}
\end{equation}
This formulation allows the context to be compressed into a recurrent matrix state $\Sbb_t \in \mathbb{R}^{m \times d_v}$ and a normalizer $\zb_t \in \mathbb{R}^m$:
\begin{equation}
\ob_t = \frac{\Sbb_t^\top \phi(\qb_t)}{\zb_t^\top \phi(\qb_t)+\varepsilon_{\rm attn}},
\qquad
\Sbb_t = \Sbb_{t-1} + \phi(\mathbf{k}_t)\vb_t^\top, \quad \zb_t = \zb_{t-1} + \phi(\mathbf{k}_t).
\label{eq:lin_attn_rec}
\end{equation}
\noindent
For a fresh sequence, Eq.~\eqref{eq:lin_attn_rec} is exactly equivalent to Eq.~\eqref{eq:lin_attn} when $\Sbb_0=\mathbf{0}$, $\zb_0=\mathbf{0}$, $\varepsilon_{\rm attn}=0$, and the denominator is nonzero. With $\varepsilon_{\rm attn}>0$, it is the usual stabilized positive-feature variant. The normalized form is attention-like only when the read normalizer is nonnegative; signed-feature caveats are deferred to Appendix~\ref{app:normalized_linattn_caveat}.

\paragraph{Causality and indexing.}
Eq.~\eqref{eq:lin_attn_rec} follows the standard Transformer convention: at position $t$ we read from the updated state $(\Sbb_t,\zb_t)$, and the resulting representation is used to predict token $t{+}1$.
Our next-latent alignment shifts the write stream by one: after observing $\vb_t$, we write it under the previous write feature $\phi(\mathbf{k}_{t-1})$, or equivalently, $(\phi(\mathbf{k}_i),\vb_{i+1})$ under standard indexing with $i=t-1$.
This yields the read-after-write recurrence
\begin{equation}
\ob_t = \frac{\Sbb_{t}^\top \phi(\qb_t)}{\zb_{t}^\top \phi(\qb_t)+\varepsilon_{\rm attn}},
\qquad
\Sbb_t = \Sbb_{t-1} + \phi(\mathbf{k}_{t-1})\vb_t^\top,\quad
\zb_t = \zb_{t-1} + \phi(\mathbf{k}_{t-1}),
\label{eq:lin_attn_rec_shifted}
\end{equation}
where $\varepsilon_{\rm attn}\ge 0$ is a small stabilizer.
Defining the shifted write-feature stream
\[
\tilde{\mathbf{x}}_1:=\mathbf{0},
\qquad
\tilde{\mathbf{x}}_t:=\phi(\mathbf{k}_{t-1}) \quad \text{for } t\ge 2,
\]
Eq.~\eqref{eq:lin_attn_rec_shifted} is exactly Eq.~\eqref{eq:lin_attn_rec} with the standard write-feature stream $\{\phi(\mathbf{k}_t)\}_{t=1}^{T}$ replaced by $\{\tilde{\mathbf{x}}_t\}_{t=1}^{T}$ when $\varepsilon_{\rm attn}=0$; for $\varepsilon_{\rm attn}>0$, it is the corresponding stabilized shifted variant. The boundary condition is therefore imposed in feature space rather than raw-key space.

\paragraph{Normalized vs.\ unnormalized linear attention.}
The denominator in Eq.~\eqref{eq:lin_attn} uses the normalizer state $\zb_t$ to rescale the readout.
Many SSM/SSD-style architectures instead drop the denominator (and $\zb_t$) and use an unnormalized inner-product read:
\begin{equation}
\label{eq:lin_attn_unnorm}
\ob_t = \Sbb_t^\top \phi(\qb_t),
\qquad
\Sbb_t = (1-\eta_t\lambda_t)\Sbb_{t-1} + \eta_t\,\phi(\mathbf{k}_{t-1})\vb_t^\top,
\end{equation}
with the unshifted convention recovered by replacing $\mathbf{k}_{t-1}$ with $\mathbf{k}_t$.
In this denominator-free form, state magnitude and the effective memory timescale are controlled by explicit decay (e.g., $\lambda_t>0$) and/or gain control. Section~\ref{sec:inner_product_loss} shows that the numerator-state update in Eq.~\eqref{eq:lin_attn_unnorm} is exactly gradient descent on an inner-product objective. The auxiliary normalizer $\zb_t$ in the normalized variants is a separate bookkeeping state for the read denominator; it is not itself obtained from that objective.

Accordingly, Eq.~\eqref{eq:lin_attn_unnorm} is the gradient-descent update for the inner-product objective in Section~\ref{sec:inner_product_loss}. When $\lambda_t=0$, the write is purely additive rank-one Hebbian learning; when $\lambda_t>0$, it is additive plus scalar shrinkage. As noted in Section~\ref{sec:background_ssm}, Mamba-2 proves that the no-normalizer recurrence is functionally equivalent to a specific class of SSMs under the SSD framework.

% =========================================================
% FIGURE: Equivalence between recurrent and parallel forms
% of Linear Attention (Apple-style, matches Fig. 1 palette)
% Requires the same preamble as the main paper:
%   - colors: appleBlue, appleLightBlue, appleRed, appleLightRed,
%     appleGreen, appleLightGreen, appleGray, appleLightGray,
%     appleDarkGray, appleDark
%   - tikzlibrary: positioning, calc, shapes.geometric,
%     arrows.meta, backgrounds, fit, matrix, shadows.blur
% =========================================================
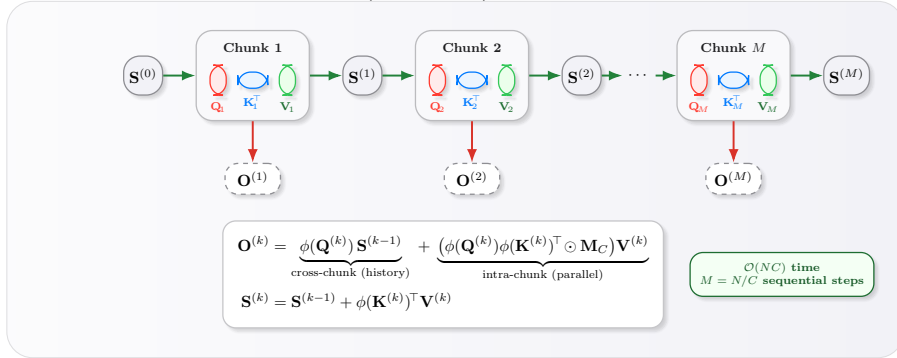
\begin{figure}[ht!]
\centering
\resizebox{0.72\linewidth}{!}{%
\begin{tikzpicture}[
font=,
>=Latex,
% --- Local styles (mirroring Figure 1) ---
base_node/.style={
  thick,
  rounded corners=8pt,
  blur shadow={shadow blur steps=5, shadow opacity=15}
},
tensor/.style={
  base_node,
  draw=appleBlue,
  fill=appleLightBlue,
  minimum height=2.4em,
  minimum width=2.4em,
  font=\bfseries
},
value_tensor/.style={
  base_node,
  draw=appleGreen,
  fill=appleLightGreen,
  minimum height=2.4em,
  minimum width=2.4em,
  font=\bfseries
},
query_tensor/.style={
  base_node,
  draw=appleRed,
  fill=appleLightRed,
  minimum height=2.4em,
  minimum width=2.4em,
  font=\bfseries
},
state/.style={
  base_node,
  draw=appleGray,
  fill=appleLightGray,
  rounded corners=10pt,
  minimum height=3.4em,
  minimum width=3.4em,
  align=center,
  font=\bfseries
},
op/.style={
  circle,
  fill=appleLightGray,
  draw=appleGray,
  inner sep=2pt,
  thick,
  font=\bfseries\small
},
arrow/.style={
  ->,
  thick,
  color=appleDarkGray,
  line width=1.2pt,
  rounded corners=4pt
},
update_arrow/.style={
  arrow,
  color=appleGreen!65!black
},
read_arrow/.style={
  arrow,
  color=appleRed
},
section_label/.style={
  font=\bfseries\Large,
  anchor=north west,
  color=appleDark
},
equiv_badge/.style={
  base_node,
  draw=appleDark,
  fill=white,
  rounded corners=14pt,
  inner sep=8pt,
  minimum width=1.6cm,
  minimum height=1.6cm,
  font=\bfseries\Huge,
  text=appleDark
},
eqbox/.style={
  base_node,
  fill=white,
  draw=appleGray!50,
  inner sep=9pt
}
]

% =========================================================
% PANEL A: Recurrent Form (Top)
% =========================================================

\begin{scope}[local bounding box=panelA]

% Token-time arrow (timeline) at the top
\node[font=\bfseries\small, text=appleGray] (timeLabelA) at (-1.6,2.7) {time $\rightarrow$};

% Chain of state nodes
\node[state] (S0) at (0,0) {$\mathbf{S}_0$};
\node[state, right=1.5cm of S0] (S1) {$\mathbf{S}_1$};
\node[state, right=1.5cm of S1] (S2) {$\mathbf{S}_2$};
\node[font=\bfseries, right=0.6cm of S2] (sdots) {$\cdots$};
\node[state, right=0.6cm of sdots] (St) {$\mathbf{S}_t$};

% Token inputs above each state transition
\node[tensor, scale=0.85, above=1.0cm of S1, xshift=-7pt] (k1) {$\mathbf{k}_1$};
\node[value_tensor, scale=0.85, right=0.12cm of k1] (v1) {$\mathbf{v}_1$};

\node[tensor, scale=0.85, above=1.0cm of S2, xshift=-7pt] (k2) {$\mathbf{k}_2$};
\node[value_tensor, scale=0.85, right=0.12cm of k2] (v2) {$\mathbf{v}_2$};

\node[tensor, scale=0.85, above=1.0cm of St, xshift=-7pt] (kt) {$\mathbf{k}_t$};
\node[value_tensor, scale=0.85, right=0.12cm of kt] (vt) {$\mathbf{v}_t$};

% State-to-state arrows
\draw[arrow] (S0) -- (S1);
\draw[arrow] (S1) -- (S2);
\draw[arrow] (S2) -- (sdots);
\draw[arrow] (sdots) -- (St);

% Write-in arrows (outer product injection)
\draw[update_arrow] ($(k1.south)!0.5!(v1.south)$) -- (S1.north);
\draw[update_arrow] ($(k2.south)!0.5!(v2.south)$) -- (S2.north);
\draw[update_arrow] ($(kt.south)!0.5!(vt.south)$) -- (St.north);

% Write label (small legend, placed below the chain to avoid overlap)
\node[font=\bfseries\scriptsize, text=appleGreen!50!black]
  at ($(S0)+(0.1,-0.95)$) {write $\phi(\mathbf{k}_s)\mathbf{v}_s^{\!\top}$};

% Readout at step t
\node[query_tensor, scale=0.9, below=1.0cm of St] (qt) {$\mathbf{q}_t$};
\node[base_node, dashed, draw=appleGray, fill=white,
      minimum width=1.7cm, minimum height=2.2em,
      right=1.3cm of St] (ytA) {$\mathbf{o}_t$};

\draw[read_arrow] (qt.north) -- (St.south);
\draw[arrow, color=appleRed!85!black] (St.east) -- (ytA.west);
\node[font=\bfseries\scriptsize, text=appleRed,
      above=0.04cm of ytA, anchor=south]
  {read $\mathbf{S}_t^{\!\top}\phi(\mathbf{q}_t)$};

% Equation for the recurrent form (left of complexity tag)
\node[eqbox, below=2.0cm of S2, align=center, xshift=-0.6cm] (eqA) {%
$\displaystyle
\begin{aligned}
\mathbf{S}_t &= \mathbf{S}_{t-1} + \phi(\mathbf{k}_t)\mathbf{v}_t^{\!\top}\\
\mathbf{o}_t &= \mathbf{S}_t^{\!\top} \phi(\mathbf{q}_t)
\end{aligned}$%
};

\node[base_node, fill=appleLightBlue!60, draw=appleBlue!50,
      inner sep=6pt, font=\bfseries\scriptsize, text=appleBlue,
      right=0.55cm of eqA, align=center] (complA)
  {$\mathcal{O}(N)$ time\\$\mathcal{O}(1)$ state at inference};

\end{scope}

% =========================================================
% CENTER: Equivalence strip
% =========================================================

% Common x for vertical alignment of both equivalence badges.
% Both equiv and equiv2 use this same x; only their y differs.
\coordinate (xRefEquiv) at ($(panelA.south)+(-3.5,0)$);

\coordinate (yPosAB) at ($(panelA.south)+(0,-1.7)$);
\node[equiv_badge] (equiv) at (xRefEquiv |- yPosAB) {$\equiv$};

\node[eqbox, draw=appleDark!40, fill=appleLightGray!50,
      right=0.45cm of equiv, inner sep=8pt,
      font=\small, text=appleDark, align=center] (derivBox)
{\textbf{exact equality for }$\mathbf{S}_0=\mathbf{0}$:\;
$\displaystyle
\mathbf{S}_t = \!\!\sum_{j\le t}\!\phi(\mathbf{k}_j)\mathbf{v}_j^{\!\top}
\;\Longrightarrow\;
\mathbf{o}_t = \!\!\sum_{j\le t}\!\bigl\langle\phi(\mathbf{q}_t),\phi(\mathbf{k}_j)\bigr\rangle\mathbf{v}_j$%
};

% =========================================================
% PANEL B: Parallel Form (Bottom) - horizontal pipeline
% =========================================================

\begin{scope}[shift={(0,-10.2)}, local bounding box=panelB]

% Q matrix (tall)
\node[base_node, draw=appleRed, fill=appleLightRed,
      minimum width=1.0cm, minimum height=2.4cm,
      align=center, font=\bfseries] (Qmat) at (0,0)
  {$\phi(\mathbf{Q})$};
\node[font=\scriptsize, text=appleRed, above=0.05cm of Qmat]
  {$N\!\times\!d$};

\node[op, right=0.55cm of Qmat] (mul1) {$\times$};

% K^T matrix (wide)
\node[base_node, draw=appleBlue, fill=appleLightBlue,
      minimum width=2.4cm, minimum height=1.0cm,
      align=center, font=\bfseries, right=0.55cm of mul1] (KTmat)
  {$\phi(\mathbf{K})^{\!\top}$};
\node[font=\scriptsize, text=appleBlue, above=0.05cm of KTmat]
  {$d\!\times\!N$};

% = operator
\node[font=\bfseries\Large, text=appleDark, right=0.65cm of KTmat] (eq1) {$=$};

% Score matrix QK^T with causal mask visualization
\node[base_node, draw=appleGray, fill=appleLightGray,
      minimum width=2.4cm, minimum height=2.4cm,
      align=center, font=\bfseries, right=0.65cm of eq1] (Scores)
  {};
% Lower-triangular causal pattern (kept content)
\begin{scope}
  \clip[rounded corners=8pt] ([xshift=2pt,yshift=-2pt]Scores.north west)
        rectangle ([xshift=-2pt,yshift=2pt]Scores.south east);
  \fill[appleBlue!18]
    ([xshift=2pt,yshift=-2pt]Scores.north west) --
    ([xshift=-2pt,yshift=2pt]Scores.south east) --
    ([xshift=2pt,yshift=2pt]Scores.south west) -- cycle;
\end{scope}
% Diagonal hatch lines for the masked-out upper triangle
\foreach \i in {0.2,0.5,0.8,1.1,1.4,1.7,2.0} {
  \draw[appleRed!22, line width=0.6pt]
    ([xshift=\i cm,yshift=-2pt]Scores.north west) --
    ([xshift=-2pt,yshift=-\i cm]Scores.north east);
}
\node[font=\bfseries\footnotesize, text=appleDark, align=center]
  at (Scores.center) {$\phi(\mathbf{Q})\phi(\mathbf{K})^{\!\top}$\\$\odot\,\mathbf{M}$};
\node[font=\scriptsize, text=appleGray, above=0.05cm of Scores]
  {$N\!\times\!N$ (causal)};

\node[op, right=0.55cm of Scores] (mul2) {$\times$};

% V matrix
\node[base_node, draw=appleGreen, fill=appleLightGreen,
      minimum width=1.2cm, minimum height=2.4cm,
      align=center, font=\bfseries, right=0.55cm of mul2] (Vmat)
  {$\mathbf{V}$};
\node[font=\scriptsize, text=appleGreen!55!black, above=0.05cm of Vmat]
  {$N\!\times\!d_v$};

\node[font=\bfseries\Large, text=appleDark, right=0.65cm of Vmat] (eq2) {$=$};

% Output O
\node[base_node, draw=appleDark, fill=white,
      minimum width=1.2cm, minimum height=2.4cm,
      align=center, font=\bfseries, right=0.65cm of eq2] (Out)
  {$\mathbf{O}$};
\node[font=\scriptsize, text=appleDark, above=0.05cm of Out]
  {$N\!\times\!d_v$};

% Pipeline arrows
\draw[arrow, color=appleRed!75!black] (Qmat.east) -- (mul1.west);
\draw[arrow, color=appleBlue] (KTmat.west) -- (mul1.east);
\draw[arrow] (KTmat.east) -- (eq1.west);
\draw[arrow] (eq1.east) -- (Scores.west);
\draw[arrow] (Scores.east) -- (mul2.west);
\draw[arrow, color=appleGreen!55!black] (Vmat.west) -- (mul2.east);
\draw[arrow] (Vmat.east) -- (eq2.west);
\draw[arrow] (eq2.east) -- (Out.west);

% Equation under the pipeline
\node[eqbox, below=1.6cm of Scores, align=center, xshift=-1.8cm] (eqB) {%
$\displaystyle
\mathbf{O} \;=\; \bigl(\phi(\mathbf{Q})\phi(\mathbf{K})^{\!\top}\odot \mathbf{M}\bigr)\,\mathbf{V}
\qquad\text{equivalently}\qquad
\mathbf{o}_t \;=\; \sum_{j\le t}\bigl\langle \phi(\mathbf{q}_t),\phi(\mathbf{k}_j)\bigr\rangle\,\mathbf{v}_j$%
};

\node[base_node, fill=appleLightRed!55, draw=appleRed!50,
      inner sep=6pt, font=\bfseries\scriptsize, text=appleRed,
      right=0.55cm of eqB, align=center] (complB)
  {$\mathcal{O}(N^2)$ time\\fully parallel over $t$};

\end{scope}

% =========================================================
% CENTER: Second equivalence strip (between B and C)
% =========================================================

\coordinate (yPosBC) at ($(panelB.south)+(0,-1.7)$);
\node[equiv_badge] (equiv2) at (xRefEquiv |- yPosBC) {$\equiv$};

\node[eqbox, draw=appleDark!40, fill=appleLightGray!50,
      right=0.45cm of equiv2, inner sep=9pt,
      font=\small, text=appleDark, align=center] (derivBox2)
{\textbf{regroup tokens into }$M$\textbf{ chunks of size }$C$, $N=MC$:\\[3pt]
$\displaystyle
\mathbf{O}^{(k)}=\phi(\mathbf{Q}^{(k)})\,\mathbf{S}^{(k-1)}
\;+\;
\bigl(\phi(\mathbf{Q}^{(k)})\phi(\mathbf{K}^{(k)})^{\!\top}\!\odot\mathbf{M}_C\bigr)\mathbf{V}^{(k)}$%
};

% =========================================================
% PANEL C: Chunk-wise Parallel Form (Bottom)
% =========================================================

\begin{scope}[shift={(0,-19.5)}, local bounding box=panelC,
  chunk_block/.style={
    base_node,
    draw=appleGray!60,
    fill=appleLightGray!45,
    minimum width=2.45cm,
    minimum height=1.95cm,
    rounded corners=10pt
  }
]

% Chunk-state chain across the panel:
%   S^(0) → [Chunk 1] → S^(1) → [Chunk 2] → S^(2) → ⋯ → [Chunk M] → S^(M)

% Chunk state nodes (we use \scriptsize state-style)
\node[state, minimum width=2.4em, minimum height=2.4em] (Sc0) at (0,0) {$\mathbf{S}^{(0)}$};

% Chunk 1
\node[chunk_block, right=0.7cm of Sc0] (Ck1) {};
% Internal contents of Chunk 1
\node[font=\bfseries\footnotesize, text=appleDark]
  at ([yshift=0.6cm]Ck1.center) {Chunk 1};
% Mini Q, K^T, V tensors stacked horizontally inside
\node[base_node, draw=appleRed, fill=appleLightRed,
      minimum width=0.35cm, minimum height=0.65cm,
      font=\bfseries\scriptsize, text=appleRed]
  at ([xshift=-0.75cm, yshift=-0.1cm]Ck1.center) (Q1m) {};
\node[font=\bfseries\scriptsize, text=appleRed,
      below=0.02cm of Q1m] {$\mathbf{Q}_{\!1}$};
\node[base_node, draw=appleBlue, fill=appleLightBlue,
      minimum width=0.65cm, minimum height=0.35cm]
  at ([xshift=0.00cm, yshift=-0.1cm]Ck1.center) (K1m) {};
\node[font=\bfseries\scriptsize, text=appleBlue,
      below=0.02cm of K1m] {$\mathbf{K}_{\!1}^{\!\top}$};
\node[base_node, draw=appleGreen, fill=appleLightGreen,
      minimum width=0.35cm, minimum height=0.65cm]
  at ([xshift=0.75cm, yshift=-0.1cm]Ck1.center) (V1m) {};
\node[font=\bfseries\scriptsize, text=appleGreen!55!black,
      below=0.02cm of V1m] {$\mathbf{V}_{\!1}$};

\node[state, minimum width=2.4em, minimum height=2.4em, right=0.7cm of Ck1] (Sc1) {$\mathbf{S}^{(1)}$};

% Chunk 2
\node[chunk_block, right=0.7cm of Sc1] (Ck2) {};
\node[font=\bfseries\footnotesize, text=appleDark]
  at ([yshift=0.6cm]Ck2.center) {Chunk 2};
\node[base_node, draw=appleRed, fill=appleLightRed,
      minimum width=0.35cm, minimum height=0.65cm]
  at ([xshift=-0.75cm, yshift=-0.1cm]Ck2.center) (Q2m) {};
\node[font=\bfseries\scriptsize, text=appleRed,
      below=0.02cm of Q2m] {$\mathbf{Q}_{\!2}$};
\node[base_node, draw=appleBlue, fill=appleLightBlue,
      minimum width=0.65cm, minimum height=0.35cm]
  at ([xshift=0.00cm, yshift=-0.1cm]Ck2.center) (K2m) {};
\node[font=\bfseries\scriptsize, text=appleBlue,
      below=0.02cm of K2m] {$\mathbf{K}_{\!2}^{\!\top}$};
\node[base_node, draw=appleGreen, fill=appleLightGreen,
      minimum width=0.35cm, minimum height=0.65cm]
  at ([xshift=0.75cm, yshift=-0.1cm]Ck2.center) (V2m) {};
\node[font=\bfseries\scriptsize, text=appleGreen!55!black,
      below=0.02cm of V2m] {$\mathbf{V}_{\!2}$};

\node[state, minimum width=2.4em, minimum height=2.4em, right=0.7cm of Ck2] (Sc2) {$\mathbf{S}^{(2)}$};

% Ellipsis
\node[font=\bfseries, right=0.45cm of Sc2] (cdotsC) {$\cdots$};

% Chunk M
\node[chunk_block, right=0.45cm of cdotsC] (CkM) {};
\node[font=\bfseries\footnotesize, text=appleDark]
  at ([yshift=0.6cm]CkM.center) {Chunk $M$};
\node[base_node, draw=appleRed, fill=appleLightRed,
      minimum width=0.35cm, minimum height=0.65cm]
  at ([xshift=-0.75cm, yshift=-0.1cm]CkM.center) (QMm) {};
\node[font=\bfseries\scriptsize, text=appleRed,
      below=0.02cm of QMm] {$\mathbf{Q}_{\!M}$};
\node[base_node, draw=appleBlue, fill=appleLightBlue,
      minimum width=0.65cm, minimum height=0.35cm]
  at ([xshift=0.00cm, yshift=-0.1cm]CkM.center) (KMm) {};
\node[font=\bfseries\scriptsize, text=appleBlue,
      below=0.02cm of KMm] {$\mathbf{K}_{\!M}^{\!\top}$};
\node[base_node, draw=appleGreen, fill=appleLightGreen,
      minimum width=0.35cm, minimum height=0.65cm]
  at ([xshift=0.75cm, yshift=-0.1cm]CkM.center) (VMm) {};
\node[font=\bfseries\scriptsize, text=appleGreen!55!black,
      below=0.02cm of VMm] {$\mathbf{V}_{\!M}$};

\node[state, minimum width=2.4em, minimum height=2.4em, right=0.7cm of CkM] (ScM) {$\mathbf{S}^{(M)}$};

% State propagation arrows (Mamba-2/SSD-style chunk recurrence)
\draw[arrow, update_arrow] (Sc0) -- (Ck1);
\draw[arrow, update_arrow] (Ck1) -- (Sc1);
\draw[arrow, update_arrow] (Sc1) -- (Ck2);
\draw[arrow, update_arrow] (Ck2) -- (Sc2);
\draw[arrow, update_arrow] (Sc2) -- (cdotsC);
\draw[arrow, update_arrow] (cdotsC) -- (CkM);
\draw[arrow, update_arrow] (CkM) -- (ScM);

% Output drops below each chunk
\node[base_node, dashed, draw=appleGray, fill=white,
      minimum width=1.2cm, minimum height=1.9em,
      below=0.9cm of Ck1, font=\bfseries] (O1c) {$\mathbf{O}^{(1)}$};
\node[base_node, dashed, draw=appleGray, fill=white,
      minimum width=1.2cm, minimum height=1.9em,
      below=0.9cm of Ck2, font=\bfseries] (O2c) {$\mathbf{O}^{(2)}$};
\node[base_node, dashed, draw=appleGray, fill=white,
      minimum width=1.2cm, minimum height=1.9em,
      below=0.9cm of CkM, font=\bfseries] (OMc) {$\mathbf{O}^{(M)}$};

\draw[arrow, color=appleRed!85!black] (Ck1.south) -- (O1c.north);
\draw[arrow, color=appleRed!85!black] (Ck2.south) -- (O2c.north);
\draw[arrow, color=appleRed!85!black] (CkM.south) -- (OMc.north);

% Equation under the panel
\node[eqbox, below=2.7cm of Sc2, align=center, xshift=-3cm] (eqC) {%
$\displaystyle
\begin{aligned}
\mathbf{O}^{(k)} &= \underbrace{\phi(\mathbf{Q}^{(k)})\,\mathbf{S}^{(k-1)}}_{\text{cross-chunk (history)}}
\;+\;
\underbrace{\bigl(\phi(\mathbf{Q}^{(k)})\phi(\mathbf{K}^{(k)})^{\!\top}\!\odot\mathbf{M}_C\bigr)\mathbf{V}^{(k)}}_{\text{intra-chunk (parallel)}}\\[2pt]
\mathbf{S}^{(k)} &= \mathbf{S}^{(k-1)} + \phi(\mathbf{K}^{(k)})^{\!\top}\mathbf{V}^{(k)}
\end{aligned}$%
};

\node[base_node, fill=appleLightGreen!55, draw=appleGreen!55!black,
      inner sep=6pt, font=\bfseries\scriptsize, text=appleGreen!45!black,
      right=0.55cm of eqC, align=center] (complC)
  {$\mathcal{O}(NC)$ time\\$M=N/C$ sequential steps};

\end{scope}

% =========================================================
% Backgrounds & titles
% =========================================================

% Compute a common horizontal extent so all three panels align.
\path
  let \p1=(panelA.west), \p2=(panelB.west), \p5=(panelC.west),
      \p3=(panelA.east), \p4=(panelB.east), \p6=(panelC.east)
  in
    coordinate (commonW) at ({min(\x1,min(\x2,\x5))-0.1},0)
    coordinate (commonE) at ({max(\x3,max(\x4,\x6))+0.1},0);

% Invisible anchors at the same x in all panels, so fit() spans the same range.
\node[inner sep=0pt, minimum size=0pt] (anchorA_L) at (commonW |- panelA.center) {};
\node[inner sep=0pt, minimum size=0pt] (anchorA_R) at (commonE |- panelA.center) {};
\node[inner sep=0pt, minimum size=0pt] (anchorB_L) at (commonW |- panelB.center) {};
\node[inner sep=0pt, minimum size=0pt] (anchorB_R) at (commonE |- panelB.center) {};
\node[inner sep=0pt, minimum size=0pt] (anchorC_L) at (commonW |- panelC.center) {};
\node[inner sep=0pt, minimum size=0pt] (anchorC_R) at (commonE |- panelC.center) {};

\begin{scope}[on background layer]
  \node[draw=appleGray!30, left color=appleLightGray!30, right color=appleLightGray,
        rounded corners=16pt,
        fit=(panelA)(eqA)(complA)(anchorA_L)(anchorA_R),
        inner sep=18pt] (boxA) {};
  \node[draw=appleGray!30, left color=appleLightGray!30, right color=appleLightGray,
        rounded corners=16pt,
        fit=(panelB)(eqB)(complB)(anchorB_L)(anchorB_R),
        inner sep=18pt] (boxB) {};
  \node[draw=appleGray!30, left color=appleLightGray!30, right color=appleLightGray,
        rounded corners=16pt,
        fit=(panelC)(eqC)(complC)(anchorC_L)(anchorC_R),
        inner sep=18pt] (boxC) {};
\end{scope}

% Section labels
\node[section_label] at ([yshift=18pt, xshift=10pt]boxA.north west)
  {A. Recurrent Form (causal scan)};
\node[section_label] at ([yshift=18pt, xshift=10pt]boxB.north west)
  {B. Parallel Form (masked attention)};
\node[section_label] at ([yshift=18pt, xshift=10pt]boxC.north west)
  {C. Chunk-wise Parallel Form (SSD-style)};

\end{tikzpicture}
}
\caption{\textbf{Three equivalent views of denominator-free linear attention.}
The diagram shows the standard same-step write stream $\phi(\mathbf{k}_t)\mathbf{v}_t^\top$; the next-latent version replaces it by the shifted stream $\mathbf{x}_t\mathbf{v}_t^\top$ with $\mathbf{x}_t:=\phi(\mathbf{k}_{t-1})$.
(A) The recurrent form maintains a fixed-size matrix state $\mathbf{S}_t\in\mathbb{R}^{d\times d_v}$.
(B) For a fresh sequence ($\mathbf{S}_0=\mathbf{0}$), unrolling the recurrence yields a masked-attention form with a causal mask $\mathbf{M}$: $\mathbf{O}=(\phi(\mathbf{Q})\phi(\mathbf{K})^{\!\top}\odot\mathbf{M})\mathbf{V}$. This costs $\mathcal{O}(N^2)$ time but is fully parallel over the sequence dimension.
(C) The chunk-wise parallel form interpolates between (A) and (B): the sequence is split into $M$ chunks of size $C$, intra-chunk computation is performed by parallel masked attention, and a fixed-size state $\mathbf{S}^{(k)}\in\mathbb{R}^{d\times d_v}$ is propagated across chunks.}
\label{fig:lin_attn_equivalence}
\end{figure}

\subsection{Delta Networks}
\label{sec:background_delta}
\noindent
Fast Weight Programmers and Delta Networks~\citep{schlag2021linear} formulate sequence modeling as the online learning of a value-retrieval function. Let $\Sbb_{t-1} \in \mathbb{R}^{d\times d_v}$ denote the fast-weight state matrix. Instead of purely additive accumulation, standard Delta Networks employ an error-driven update derived from the gradient of the instantaneous squared error between the state's reconstruction of the current key and value:
\begin{equation}
\ell_t(\Sbb) := \frac{1}{2}\left\|\Sbb^\top \mathbf{k}_t - \vb_t\right\|_2^2.
\label{eq:delta_loss}
\end{equation}
Here, $\Sbb^\top \mathbf{k}_t$ represents the model's prediction of value $\vb_t$ given key $\mathbf{k}_t$. The gradient with respect to the state is $\nabla_{\Sbb}\ell_t(\Sbb) = \mathbf{k}_t\big(\Sbb^\top \mathbf{k}_t - \vb_t\big)^\top$.

\paragraph{Delta rule.} We denote the gradient step size by $\eta_t$; later, $\beta_t$ denotes a dimensionless normalized gain and $\eta_t$ the induced step size. A single online gradient step gives
\begin{equation}
\Sbb_t = \Sbb_{t-1} - \eta_t \,\mathbf{k}_t(\Sbb_{t-1}^\top \mathbf{k}_t - \vb_t)^\top = (\Ib - \eta_t \mathbf{k}_t \mathbf{k}_t^\top)\Sbb_{t-1} + \eta_t \mathbf{k}_t \vb_t^\top.
\label{eq:delta_rule}
\end{equation}
The rank-one factor performs a targeted shrinkage/edit along the current key direction; it becomes an orthogonal projection only when $\eta_t=1/\|\mathbf{k}_t\|_2^2$. Gated Delta Networks add an explicit global decay gate; notation and comparison details are in Appendix~\ref{app:deltanet_step_and_gating}.

\section{Autoregressive Next-Latent Prediction}
\label{sec:autoregressive}

In this section, we cast the recurrent write as an explicit online optimization problem. Under the read-after-write convention, the causal example revealed at step $t$ pairs the newly observed target with the prefix write feature available when that target was predicted, namely $\mathbf{x}_t=\phi(\mathbf{k}_{t-1})$ and $\mathbf{y}_t=\vb_t$. Standard DeltaNet instead uses the same-step pair $(\phi(\mathbf{k}_t),\vb_t)$; that pairing remains causal, but it corresponds to a different local fast-memory objective. We therefore model $\Sbb$ as an online linear predictor from $\mathbf{x}_t$ to $\mathbf{y}_t$ and optimize an instantaneous ridge-regression loss~\citep{wang2025test, behrouz2024titans, behrouz2025atlas}:

\begin{equation}
\label{eq:inst-loss}
\ell_t(\Sbb) \triangleq \frac{1}{2}\left\|\Sbb^\top \mathbf{x}_t - \mathbf{y}_t\right\|_2^2 + \frac{\lambda_t}{2}\|\Sbb\|_F^2,
\end{equation}
where $\lambda_t \ge 0$ is a regularization coefficient. While full-batch minimization of the cumulative loss corresponds to the offline solution found in methods like MesaNet~\citep{von2025mesanet}, efficient autoregressive modeling requires an online approximation. We therefore employ Online Gradient Descent (OGD).

\paragraph{Notation.}
In kernelized linear attention, the key that writes to memory is typically a feature vector $\phi(\mathbf{k}) \in \mathbb{R}^{m}$ rather than the raw key $\mathbf{k}\in\mathbb{R}^{d}$.
All derivations in this section are interpreted
\[
\mathbf{x}_t \equiv \phi(\mathbf{k}_{t-1}) \in \mathbb{R}^{m},
\qquad
\Sbb_t \in \mathbb{R}^{m\times d_v},
\]
so we treat $\mathbf{x}_t$ as the generic write feature and keep notation uncluttered. Queries used for retrieval, e.g., $\phi(\qb_t)$ in linear attention, live in the same feature space but need not equal $\mathbf{x}_t$. Whenever a summary sentence informally refers to a pairing in raw-key space, the mathematically exact kernelized object is the corresponding write feature.

\paragraph{Implicit fast-memory objective.}
Eq.~\eqref{eq:inst-loss} is the instantaneous objective whose gradient step \emph{defines} the fast-memory write rule; it is not an additional supervised loss beyond the outer autoregressive likelihood. During training, we differentiate through the update so that the slow weights (which produce $(\qb,\mathbf{k},\vb)$ as well as $\beta_t,\lambda_t$) learn representations and gates that make these local updates useful.

\paragraph{Indexing and causality conventions.}
We use the read-after-write (RAW) convention throughout: after token $t$ is observed and written, the updated state $\Sbb_t$ is read to predict token $t{+}1$. Under next-latent alignment, the causal write pair is therefore $(\phi(\mathbf{k}_{t-1}),\vb_t)$, or equivalently $(\phi(\mathbf{k}_i),\vb_{i+1})$ under standard indexing. We set $\Sbb_0=\mathbf{0}$ and impose the feature-space boundary $\mathbf{x}_1:=\mathbf{0}$. For update rules with explicit shrinkage, the boundary sentinel is also assigned $\eta_1:=0$ (equivalently, $\alpha_1=0,\gamma_1=1$ in the log-space notation); otherwise $\lambda_1>0$ would decay a carried state even though no data pair is written. Detailed RAW/RBW and boundary conventions are deferred to Appendix~\ref{app:timing_boundary}.

With this convention, the internal fast-memory prediction is $\hat{\mathbf{y}}_t:=\Sbb_{t-1}^\top \mathbf{x}_t$, with residual $\mathbf{r}_t=\mathbf{y}_t-\hat{\mathbf{y}}_t$. This is distinct from the model readout at position $t$, which under RAW uses the updated state $\Sbb_t$.

\paragraph{Fast memory as continual learning.}
The recurrent state is the \emph{fast} memory updated within the forward pass; each token provides a local training pair $(\mathbf{x}_t,\mathbf{y}_t)$. Across the family, $\beta_t$ controls plasticity and $\lambda_t$ controls shrinkage/forgetting; the realized $\eta_t$ depends on the local normalization statistic, which is smoothness-matched for regression and energy-based for the inner-product implementations.

\subsection{Online Gradient Descent Update}
The gradient of the instantaneous loss in Eq.~\eqref{eq:inst-loss} with respect to the state $\Sbb$ is:
\begin{equation}
\nabla_{\Sbb}\ell_t(\Sbb) = \mathbf{x}_t (\Sbb^\top \mathbf{x}_t - \mathbf{y}_t)^\top + \lambda_t \Sbb.
\end{equation}
Applying a single gradient descent step with learning rate $\eta_t$ yields the update rule:
\begin{align}
\label{eq:ogd}
\Sbb_{t} &\leftarrow \Sbb_{t-1} - \eta_t \nabla_{\Sbb}\ell_t(\Sbb_{t-1}) \nonumber \\
&= \Sbb_{t-1} - \eta_t \left[ \mathbf{x}_t (\Sbb_{t-1}^\top \mathbf{x}_t - \mathbf{y}_t)^\top + \lambda_t \Sbb_{t-1} \right] \nonumber \\
&= (1 - \eta_t \lambda_t)\Sbb_{t-1} + \eta_t \mathbf{x}_t \mathbf{r}_t^\top,
\end{align}
where $\mathbf{r}_t \triangleq \mathbf{y}_t - \Sbb_{t-1}^\top \mathbf{x}_t$ is the residual (prediction error). Note that, unlike standard Delta Networks, $\mathbf{r}_t$ measures the discrepancy between the prediction from the previous write feature $\mathbf{x}_t=\phi(\mathbf{k}_{t-1})$ and the current value $\mathbf{v}_t$.

This is gradient descent on the instantaneous ridge objective. With the normalized step size below, it becomes a normalized update. In the special case $\lambda_t=0$ and $\varepsilon=0$, it reduces exactly to the classical NLMS recursion; for $\varepsilon>0$, it is the usual stabilized NLMS variant. Specifically, the loss $\ell_t$ is $L_t$-smooth with respect to the Frobenius norm, with smoothness constant $L_t = \|\mathbf{x}_t\|_2^2 + \lambda_t$. To ensure numerical stability and scale robustness, we adopt the normalized step size:
\begin{equation}
\label{eq:nlms-stepsize}
\eta_{t} = \frac{\beta_{t}}{\|\mathbf{x}_t\|_2^2 + \lambda_t + \varepsilon},
\qquad
\beta_t \in (0,2),\ \varepsilon\ge 0.
\end{equation}
\noindent
We adopt the convention $\eta_t:=0$ when $\|\mathbf{x}_t\|_2^2+\lambda_t+\varepsilon=0$, and also at the boundary sentinel $t=1$ when $\mathbf{x}_1=\mathbf{0}$ is used to denote ``no causal pair.'' In analysis, we may take $\varepsilon=0$ and assume $L_t>0$. In implementations, we take $\varepsilon>0$ for numerical robustness; any $\varepsilon>0$ only decreases $\eta_t$ and therefore preserves the descent guarantees below.

\medskip
\noindent Appendix~\ref{app:extra_impl_remarks} records secondary implementation details, including the $\beta_t>1$ sign-flip regime and the positive-decay interpretation of ridge shrinkage under log-space unrolling.

\subsection{Analysis}
We show that the normalized step size yields per-step descent in the instantaneous regularized objective $\ell_t$, a basic local stability property. This statement is pointwise in $t$: it does not imply monotone decrease of the cumulative online loss $\sum_s \ell_s$ or of the outer autoregressive training objective.

\begin{lemma}[Per-step Descent for Smooth Losses]
\label{lem:per-step-descent}
Let $f:\mathbb{R}^{d_x\times d_v}\to\mathbb{R}$ be $L$-smooth with respect to the Frobenius norm. For any step size $\eta\in(0,2/L)$, the gradient step $\Sbb^{+}=\Sbb-\eta\nabla f(\Sbb)$ satisfies
\begin{equation}
f(\Sbb^{+}) \le f(\Sbb) - \frac{\eta(2-\eta L)}{2}\,\|\nabla f(\Sbb)\|_F^2.
\end{equation}
\end{lemma}

\begin{proof}
By $L$-smoothness, for any $\Sbb$ and $\Sbb'$,
\[
f(\Sbb')\le f(\Sbb)+\langle\nabla f(\Sbb),\Sbb'-\Sbb\rangle+\frac{L}{2}\|\Sbb'-\Sbb\|_F^2.
\]
Set $\Sbb'=\Sbb^{+}=\Sbb-\eta\nabla f(\Sbb)$ and simplify.
\end{proof}

\noindent\textbf{Step-size parametrization.}
Choosing $\eta=\beta/L$ with $\beta\in(0,2)$ (hence requiring $L>0$) yields a decrease coefficient $\beta(2-\beta)/(2L)$.
For $L>0$, the stabilized choice $\eta=\beta/(L+\varepsilon)$ with $\varepsilon\ge 0$ also lies in $(0,2/L)$. When $L=0$, this interval is undefined; in the degenerate cases arising here ($\mathbf{x}_t=\mathbf{0}$ and $\lambda_t=0$, or the corresponding windowed analogue), the gradient is zero, so we define the update to be a no-op by setting $\eta:=0$.

\subsection{Delta Networks as Regression}
\label{sec:unified_perspective}

In this section, we interpret many previous \textit{fast weight} models under this optimization framework. We observe that Delta Networks and Linear Attention can be viewed as gradient updates induced by specific online objective functions.

By substituting the regression assignments $\mathbf{x}_t \leftarrow \phi(\mathbf{k}_{t-1})$ and $\mathbf{y}_t \leftarrow \vb_{t}$, Eq.~\eqref{eq:ogd} recovers the functional form of the Delta Network update rule~\citep{schlag2021linear}, but with the critical index shift:
\begin{equation}
\Sbb_t
= \underbrace{\big((1-\eta_{t} \lambda_{t})\Ib_{d_x} - \eta_{t}\mathbf{x}_{t}\mathbf{x}_{t}^{\top}\big)}_{\text{Decay \& Targeted Forget}} \Sbb_{t-1}
+ \underbrace{\eta_{t}\mathbf{x}_{t}\mathbf{y}_{t}^{\top}}_{\text{Write}}, \qquad \mathbf{x}_t=\phi(\mathbf{k}_{t-1}),\ \mathbf{y}_t=\vb_t.
\end{equation}
Here, the rank-one term $\mathbf{x}_{t}\mathbf{x}_{t}^\top \Sbb_{t-1}$ is the left Hessian action of the squared-error loss along the current write-feature direction. Intuitively, it reduces the component of the current predictor that acts on $\mathbf{x}_t=\phi(\mathbf{k}_{t-1})$ before adding the new target $\mathbf{y}_t=\vb_t$.

In contrast, replacing the regression (MSE) loss with the inner-product objective of Section~\ref{sec:inner_product_loss} removes the residual term and yields an additive write. With the standard unshifted assignment $(\mathbf{x}_t,\mathbf{y}_t)=(\phi(\mathbf{k}_t),\vb_t)$ (or $(\mathbf{k}_t,\vb_t)$ in the unkernelized case), this is the familiar Linear Attention / Mamba-2 accumulation. With the next-latent assignment $(\mathbf{x}_t,\mathbf{y}_t)=(\phi(\mathbf{k}_{t-1}),\vb_t)$, it becomes the one-step-shifted variant used by our methods (e.g., Falcon-3A in Section~\ref{sec:methods_sliding_ip}).

This regression perspective motivates a key algorithmic improvement that we analyze in Section~\ref{sec:methods}: with objective-matched normalization, the rank-one regression step uses $L_t=\|\mathbf{x}_t\|_2^2+\lambda_t$, while the sliding rule uses $L_t^{(B)}=\lambda_{\max}(\bar{\Cb}_t^{(B)})+\lambda_t$. This suggests that fixed learning rates are scale-mismatched for regression-style fast-weight updates; for inner-product writes, the same normalization is better viewed as a magnitude stabilizer than as a curvature requirement.

Appendix~\ref{app:ridge_reference} gives reference pseudocode for the sequential first-order online-ridge update.

\section{\Falcon: Fast Weight Attention}
\label{sec:methods}

Fast-weight memories and linear-attention architectures can be interpreted as online models that update a recurrent memory during the forward pass, tracing back to classical fast-weight mechanisms and their modern instantiations in linear Transformers and Delta-style rules~\citep{hinton1987using,schmidhuber1992learning, ba2016using, schlag2021linear}. In this section, we derive \Falcon ~from two local objectives under the shifted $\phi(\mathbf{k}_{t-1})\to\vb_t$ alignment: squared-error regression and a negative inner-product objective. This yields normalized step sizes, explicit forgetting controls, and sliding-window variants.

\paragraph{Naming and formula summary.}
All variants use the same causal pair
\[
\mathbf{x}_t:=\phi(\mathbf{k}_{t-1}),\qquad
\mathbf{y}_t:=\vb_t,\qquad
\mathbf{x}_1:=\mathbf{0},\qquad
\eta_1:=0,
\]
and read after writing, $\ob_t=\Sbb_t^\top\phi(\qb_t)$. Let
$\mathbf{r}_t:=\mathbf{y}_t-\Sbb_{t-1}^\top\mathbf{x}_t$ and
$\operatorname{Diag}(\boldsymbol{\eta}_t)$ denote the diagonal matrix of
per-column step sizes. The regression family is
\[
\begin{aligned}
\text{\Falcon-1:}\quad
\Sbb_t
&=(1-\eta_t\lambda_t)\Sbb_{t-1}
  +\eta_t\,\mathbf{x}_t\mathbf{r}_t^\top,
&\qquad
\eta_t&=\frac{\beta_t}{\|\mathbf{x}_t\|_2^2+\lambda_t+\varepsilon},
\\[2pt]
\text{\Falcon-2:}\quad
\Sbb_t
&=\Sbb_{t-1}\!\left(\Ib_{d_v}-\lambda_t\operatorname{Diag}(\boldsymbol{\eta}_t)\right)
  +\mathbf{x}_t(\boldsymbol{\eta}_t\odot\mathbf{r}_t)^\top,
&\qquad
\eta_{j,t}&=\frac{\beta_{j,t}}{\|\mathbf{x}_t\|_2^2+\lambda_t+\varepsilon},
\\[2pt]
\text{\Falcon-3:}\quad
\Sbb_t
&=(1-\eta_t\lambda_t)\Sbb_{t-1}
  +\frac{\eta_t}{B_t}\sum_{j\in\mathcal{I}_t}
  \mathbf{x}_j\big(\mathbf{y}_j-\Sbb_{t-1}^{\!\top}\mathbf{x}_j\big)^\top,
&\qquad
\eta_t&=\frac{\beta_t}{\mu_t^{(B)}+\lambda_t+\varepsilon}.
\end{aligned}
\]
Here $\mu_t^{(B)}=\lambda_{\max}\!\left(B_t^{-1}\sum_{j\in\mathcal{I}_t}\mathbf{x}_j\mathbf{x}_j^\top\right)$.
The inner-product family replaces the residual regression write by direct
target writes:
\[
\begin{aligned}
\text{\Falcon-1A:}\quad
\Sbb_t
&=(1-\eta_t\lambda_t)\Sbb_{t-1}+\eta_t\,\mathbf{x}_t\mathbf{y}_t^\top,
&\qquad
\eta_t&=\frac{\beta_t}{E_t+\lambda_t+\varepsilon},\\[2pt]
\text{\Falcon-2A:}\quad
\Sbb_t
&=\Sbb_{t-1}\!\left(\Ib_{d_v}-\lambda_t\operatorname{Diag}(\boldsymbol{\eta}_t)\right)
  +\mathbf{x}_t(\boldsymbol{\eta}_t\odot\mathbf{y}_t)^\top,
&\qquad
\eta_{j,t}&=\frac{\beta_{j,t}}{E_t+\lambda_t+\varepsilon},\\[2pt]
\text{\Falcon-3A:}\quad
\Sbb_t
&=(1-\eta_t\lambda_t)\Sbb_{t-1}+\eta_t\,\bar{\Nb}_t^{(B)},
&\qquad
\eta_t&=\frac{\beta_t}{\bar E_t^{(B)}+\lambda_t+\varepsilon},
\end{aligned}
\]
where $E_t=\|\mathbf{x}_t\|_2^2$,
$\bar{\Nb}_t^{(B)}=B_t^{-1}\sum_{j\in\mathcal{I}_t}\mathbf{x}_j\mathbf{y}_j^\top$,
and $\bar E_t^{(B)}=B_t^{-1}\sum_{j\in\mathcal{I}_t}\|\mathbf{x}_j\|_2^2$.
Thus, the index $1/2/3$ denotes scalar, per-column, and sliding-window dynamics, respectively. The suffix ``A'' denotes the inner-product objective.

% =========================================================
% Figure 1: Conceptual Overview
%
% Drop-in replacement for the original Figure 1.
% Now merges Falcon-1 (scalar NLMS) and Falcon-2 (per-channel) into a single
% panel B with a shared diagram and a dual-formula equation box.
%
% Requires the same preamble as the paper:
%   - colors: appleBlue, appleLightBlue, appleRed, appleLightRed,
%     appleGreen, appleLightGreen, appleGray, appleLightGray,
%     appleDarkGray, appleDark
%   - tikzlibrary: positioning, calc, shapes.geometric,
%     arrows.meta, backgrounds, fit, matrix, shadows.blur

% =========================================================
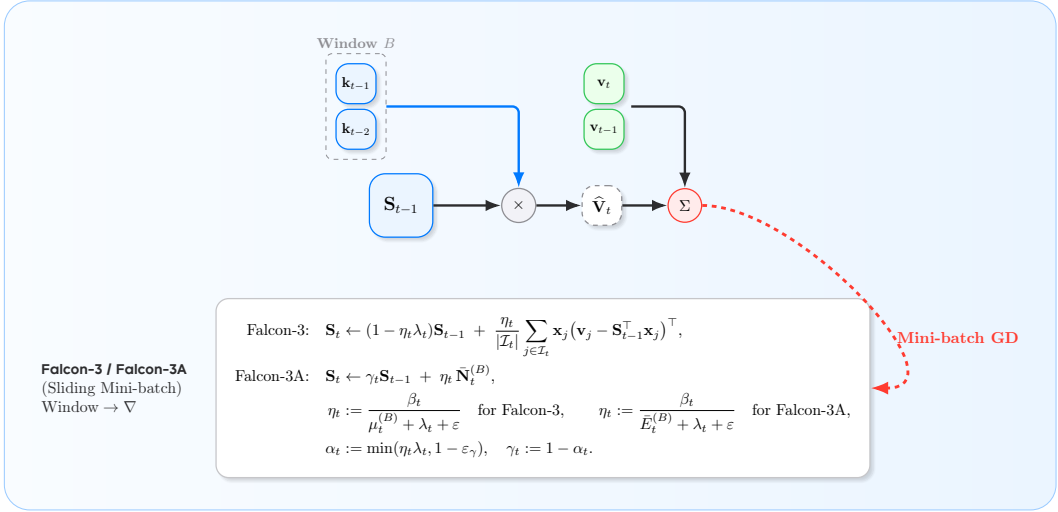
\begin{figure}[ht!]
\centering
\resizebox{0.95\linewidth}{!}{%
\begin{tikzpicture}[
font=,
>=Latex,
% Styles
base_node/.style={
thick,
rounded corners=8pt,
blur shadow={shadow blur steps=5, shadow opacity=15}
},
tensor/.style={
base_node,
draw=appleBlue,
fill=appleLightBlue,
minimum height=2.8em,
minimum width=2.8em,
font=\bfseries
},
value_tensor/.style={
base_node,
draw=appleGreen,
fill=appleLightGreen,
minimum height=2.8em,
minimum width=2.8em,
font=\bfseries
},
state/.style={
base_node,
draw=appleGray,
fill=appleLightGray,
rounded corners=10pt,
minimum height=4em,
minimum width=4em,
align=center,
font=\bfseries\large
},
op/.style={
circle,
fill=appleLightGray,
draw=appleGray,
inner sep=4pt,
thick,
font=\bfseries
},
arrow/.style={
->,
thick,
color=appleDarkGray,
line width=1.5pt,
rounded corners=4pt
},
signal_arrow/.style={
arrow,
color=appleBlue
},
error_arrow/.style={
->,
dashed,
color=appleRed,
line width=1.8pt
},
section_label/.style={
font=\bfseries\Large,
anchor=north west,
color=appleDark
},
subcaption/.style={
font=,
color=appleDarkGray,
align=center
}
]

% =========================================================
% PANEL A: Prior Work (Top-Left, small)
% =========================================================
\begin{scope}[local bounding box=panelA]
\node[tensor] (k_old) at (0,0) {$\mathbf{k}_t$};
\node[value_tensor, right=1.5cm of k_old] (v_old) {$\mathbf{v}_t$};
\node[state, below=2cm of $(k_old)!0.5!(v_old)$] (S_old) {$\mathbf{S}_t$};

\draw[arrow] (k_old) -- (S_old);
\draw[arrow] (v_old) -- (S_old);

\node[font=\bfseries\scriptsize, text=appleGray] at ($(k_old)!0.5!(S_old)$) {same-step};
\node[font=\bfseries, text=appleGray, left=0.2cm of S_old, anchor=east] {Write};

\node[subcaption, below=0.8cm of S_old] (capA) {
\textbf{Common Rule}\\
(Same-Step Association)\\
$\phi(\mathbf{k}_t) \leftrightarrow \mathbf{v}_t$
};
\end{scope}

% =========================================================
% PANEL B (MERGED): Falcon-1 / Falcon-2 (Top-Right)
% Shared diagram, dual-formula equation box
% =========================================================
\begin{scope}[shift={(12.5,0)}, local bounding box=panelB]
\node[tensor] (b_k_prev) at (0,0) {$\mathbf{k}_{t-1}$};
\node[state, draw=appleBlue, fill=appleLightBlue, below=2.5cm of b_k_prev] (b_S_prev) {$\mathbf{S}_{t-1}$};
\node[left=0.2cm of b_S_prev, align=right, font=\bfseries, color=appleBlue]
(b_lin_pred) {Linear\\Predictor};

\node[op, right=2.5cm of b_S_prev] (b_prod) {$\times$};
\node[base_node, dashed, draw=appleGray, fill=white, right=1.5cm of b_prod, minimum height=2.5em, minimum width=2.5em, font=\bfseries] (b_v_hat) {$\hat{\mathbf{v}}_t$};
\node[op, draw=appleRed, fill=appleLightRed, right=1.5cm of b_v_hat] (b_loss) {$\mathcal{L}$};

\node[value_tensor, above=2.5cm of b_loss] (b_v_curr) {$\mathbf{v}_t$};

\draw[signal_arrow] (b_k_prev) -- node[left, font=\bfseries, pos=0.5] {Feature $\mathbf{x}_t$} (b_S_prev);
\draw[signal_arrow] (b_k_prev.east) -- ++(0.8,0) -| (b_prod.north);
\draw[arrow] (b_S_prev) -- (b_prod);
\draw[arrow] (b_prod) -- (b_v_hat);
\draw[arrow] (b_v_hat) -- (b_loss);
\draw[arrow] (b_v_curr) -- node[right, font=\bfseries, pos=0.5] {Target $\mathbf{y}_t$} (b_loss);

% MERGED equation box: Falcon-1 (scalar) and Falcon-2 (per-channel), stacked
\node[base_node, fill=white, draw=appleGray!50, inner sep=12pt, below=1.5cm of $(b_S_prev.south)!0.5!(b_loss.south)$] (b_eq_box) {
$\displaystyle
\begin{aligned}
\text{\Falcon-1}:\;\;
\mathbf{S}_t &\leftarrow (1-\eta_t\lambda_t)\mathbf{S}_{t-1}
+ \eta_t\,\mathbf{x}_{t}\mathbf{r}_t^\top\\
&= \big((1-\eta_t\lambda_t)\mathbf{I}-\eta_t\,\mathbf{x}_{t}\mathbf{x}_{t}^\top\big)\mathbf{S}_{t-1}
+ \eta_t\,\mathbf{x}_{t}\mathbf{y}_t^\top,\\
\eta_t &=\frac{\beta_t}{\|\mathbf{x}_t\|_2^2+\lambda_t+\varepsilon},\\[3pt]
\text{\Falcon-2}:\;\;
\mathbf{S}_t &\leftarrow \mathbf{S}_{t-1}\big(\mathbf{I}-\lambda_t\operatorname{Diag}(\boldsymbol{\eta}_t)\big)
+ \mathbf{x}_{t}\big(\boldsymbol{\eta}_t \odot \mathbf{r}_t\big)^\top,\\
\eta_{j,t} &=\frac{\beta_{j,t}}{\|\mathbf{x}_t\|_2^2+\lambda_t+\varepsilon},\\[3pt]
\mathbf{r}_t &:= \mathbf{v}_t - \mathbf{S}_{t-1}^\top \mathbf{x}_{t}
\quad\text{(shared residual)}.
\end{aligned}
$
};

% Gradient Arrow (route to top-right corner of equation box so it doesn't pass through content)
\draw[error_arrow] (b_loss.south) to[out=-90, in=90] node[midway, right=4pt, font=\bfseries, text=appleRed] {Gradient Step} (b_eq_box.north east);

\node[subcaption, left=0.5cm of b_eq_box, anchor=east, align=left] (b_capB) {
\textbf{Ours: \Falcon-1 / \Falcon-2}\\
(Next-Latent Prediction)\\
$\phi(\mathbf{k}_{t-1}) \to \mathbf{v}_t$\\
\Falcon-1: scalar $\eta_t$\\
\Falcon-2: per-channel $\boldsymbol{\eta}_t$
};

\node[above=0.3cm of b_v_hat, font=\bfseries, color=appleGray] {Next-Latent Pred};
\end{scope}

% =========================================================
% PANEL C: Falcon-3 / Falcon-3A (Bottom)
% =========================================================
\begin{scope}[shift={(9,-15.25)}, local bounding box=panelC]
\node[state, draw=appleBlue, fill=appleLightBlue] (S_win) at (0, -2.5) {$\mathbf{S}_{t-1}$};

\node[tensor, scale=0.9] (k_w1) at (-1.0, 0.2) {$\mathbf{k}_{t-1}$};
\node[tensor, scale=0.9, below=0.1cm of k_w1] (k_w2) {$\mathbf{k}_{t-2}$};
\node[fit=(k_w1)(k_w2), draw=appleGray, dashed, rounded corners, inner sep=6pt, label={[font=\bfseries\small, color=appleGray]above:Window $B$}] (win_box) {};

\node[value_tensor, scale=0.9] (v_w1) at (4.5, 0.2) {$\mathbf{v}_{t}$};
\node[value_tensor, scale=0.9, below=0.1cm of v_w1] (v_w2) {$\mathbf{v}_{t-1}$};
\node[fit=(v_w1)(v_w2), inner sep=4pt] (target_box) {};

\node[op, right=1.5cm of S_win] (prod_win) {$\times$};
\node[base_node, dashed, draw=appleGray, fill=white, right=1.0cm of prod_win, minimum height=2.5em, minimum width=2.5em] (v_hat_win) {$\hat{\mathbf{V}}_t$};
\node[op, draw=appleRed, fill=appleLightRed, right=1.0cm of v_hat_win] (sum_loss) {$\Sigma$};

\draw[signal_arrow] (win_box.east) -| (prod_win.north);
\draw[arrow] (S_win) -- (prod_win);
\draw[arrow] (prod_win) -- (v_hat_win);
\draw[arrow] (v_hat_win) -- (sum_loss);
\draw[arrow] (target_box.east) -| (sum_loss.north);

\node[base_node, fill=white, draw=appleGray!50, inner sep=12pt, below=1.5cm of $(S_win.south)!0.5!(sum_loss.south)$] (eq_box_win) {%
$\displaystyle
\begin{aligned}
\text{\Falcon-3:}\quad
\mathbf{S}_t &\leftarrow (1-\eta_t\lambda_t)\mathbf{S}_{t-1}
\;+\; \frac{\eta_t}{|\mathcal{I}_t|} \sum_{j \in \mathcal{I}_t} \mathbf{x}_{j}\big(\mathbf{v}_j - \mathbf{S}_{t-1}^\top \mathbf{x}_{j}\big)^\top,\\
\text{\Falcon-3A:}\quad
\mathbf{S}_t &\leftarrow \gamma_t\mathbf{S}_{t-1}
\;+\; \eta_t\,\bar{\mathbf{N}}_t^{(B)},\\
\eta_t &:= \frac{\beta_t}{\mu_t^{(B)}+\lambda_t+\varepsilon}
\quad \text{for \Falcon-3},
\qquad
\eta_t := \frac{\beta_t}{\bar E_t^{(B)}+\lambda_t+\varepsilon}
\quad \text{for \Falcon-3A},\\
\alpha_t &:= \min(\eta_t\lambda_t,1-\varepsilon_\gamma),
\quad \gamma_t:=1-\alpha_t.
\end{aligned}
$%
};

\draw[error_arrow]
  (sum_loss.east) to[out=0, in=0]
  node[pos=0.6, below right, xshift=6pt, yshift=-1pt,
       font=\bfseries, text=appleRed] {Mini-batch GD}
  (eq_box_win.east);

\node[subcaption, left=0.5cm of eq_box_win, anchor=east, align=left] (capC) {
\textbf{\Falcon-3 / \Falcon-3A}\\
(Sliding Mini-batch)\\
$\text{Window} \to \nabla$
};
\end{scope}

% =========================================================
% BACKGROUNDS
% =========================================================
\begin{scope}[on background layer]
\node[draw=appleBlue!40, left color=appleLightBlue!30, right color=appleLightBlue, rounded corners=16pt, fit=(panelA), inner sep=20pt] (boxA) {};
\node[draw=appleBlue!40, left color=appleLightBlue!30, right color=appleLightBlue, rounded corners=16pt, fit=(panelB)(b_eq_box)(b_capB), inner sep=20pt] (boxB) {};
\node[draw=appleBlue!40, left color=appleLightBlue!30, right color=appleLightBlue, rounded corners=16pt, fit=(panelC)(eq_box_win)(capC), inner sep=20pt] (boxC) {};
\end{scope}

% Section Titles
\node[section_label] at ([yshift=20pt, xshift=0pt]boxA.north west) {A. Objective Alignment};
\node[section_label] at ([yshift=20pt, xshift=15pt]boxB.north west) {B. \Falcon-1 / \Falcon-2: Next-Latent Prediction};
\node[section_label] at ([yshift=20pt, xshift=15pt]boxC.north west) {C. \Falcon-3 / \Falcon-3A: Sliding Window};

\end{tikzpicture}
}
\caption{\textbf{Conceptual Overview.}
(A) Many fast-weight rules bind the same-step write-feature/target pair $(\phi(\mathbf{k}_t),\mathbf{v}_t)$ as a cache-style association. Under the prefix-prediction fast-memory objective studied here, the example revealed at step $t$ pairs a prefix write feature with the newly observed target, yielding $(\phi(\mathbf{k}_{t-1}),\mathbf{v}_t)$.
(B) \textbf{\Falcon-1 / \Falcon-2 (Next-Latent Prediction):} the state $\mathbf{S}_{t-1}$ predicts $\mathbf{v}_t$ from the prefix feature $\mathbf{x}_t:=\phi(\mathbf{k}_{t-1})$, then performs an online NLMS-stabilized ridge-regression update. 
(C) \textbf{\Falcon-3 / \Falcon-3A (Sliding Mini-batch):} applies a mini-batch regression (or inner-product) step over an active window $\mathcal{I}_t$ of nominal size $B$ (realized size $B_t:=|\mathcal{I}_t|\le B$) over causal pairs $(\mathbf{x}_j,\mathbf{v}_j)$ with $\mathbf{x}_j:=\phi(\mathbf{k}_{j-1})$.}
\label{fig:concept}
\end{figure}

\subsection{Scaled Linear Attention and Scaled DeltaNet}
\label{sec:scaled_fast_weight}

Unlike softmax attention, which uses the scaled dot product $\langle \qb,\mathbf{k}\rangle/\sqrt{d}$, fast-weight recurrences are directly sensitive to the norms of queries and keys:
(i) dot-product reads grow with $\|\qb_t\|_2\|\mathbf{k}\|_2$, and
(ii) additive (inner-product) writes grow with the write-feature norm.
To stabilize both the \emph{read} and the \emph{write} streams, especially under long decoding horizons and mixed precision, we use explicit feature scaling/normalization.

\paragraph{Scaled features.}
We define a generic RMS normalization operator for a vector $\ub\in\mathbb{R}^{d_u}$:
\begin{align*}
\operatorname{RMSNorm}(\ub)
:= \frac{\ub}{\sqrt{\|\ub\|_2^2/d_u+\varepsilon_{\rm rms}}},
\end{align*}
where $\varepsilon_{\rm rms}>0$ is a small stabilizer. Unless stated otherwise, we apply RMSNorm to the $(\qb_t,\mathbf{k}_t)$ projections used by fast-weight reads/writes, before forming dot products or outer products. This default differs from common $\ell_2$-normalized DeltaNet variants (e.g., in Gated DeltaNet implementations) and is substantially more stable in mixed precision because standard RMSNorm keeps coordinate magnitudes $\Theta(1)$. Moreover,
\[
\big\|\operatorname{RMSNorm}(\ub)\big\|_2^2
=
\frac{d_u\|\ub\|_2^2}{\|\ub\|_2^2+d_u\varepsilon_{\rm rms}}
\le d_u,
\]
so in the usual regime $\|\ub\|_2^2\gg d_u\varepsilon_{\rm rms}$ we indeed have $\|\operatorname{RMSNorm}(\ub)\|_2^2\approx d_u$. By default, we do \emph{not} normalize values $\vb_t$; value normalization (VNorm) is optional and disabled unless explicitly enabled.

\paragraph{Scaled Linear Attention.}
Under next-latent alignment, we use the RMS-normalized projections in the feature space. The scalar denominator-free inner-product recurrence, denoted \Falcon-1A below, is
\[
\ob_t=\Sbb_t^\top \phi(\qb_t),
\qquad
\Sbb_t=(1-\eta_t\lambda_t)\Sbb_{t-1}+\eta_t\,\mathbf{x}_t\vb_t^\top,
\qquad \mathbf{x}_t:=\phi(\mathbf{k}_{t-1}).
\]
Its per-column counterpart, \Falcon-2A, replaces the scalar write gain by a
vector $\boldsymbol{\eta}_t\in\mathbb{R}^{d_v}$:
\[
\Sbb_t
=\Sbb_{t-1}\!\left(\Ib_{d_v}-\lambda_t\operatorname{Diag}(\boldsymbol{\eta}_t)\right)
+\mathbf{x}_t(\boldsymbol{\eta}_t\odot\vb_t)^\top.
\]
The corresponding normalized numerator/denominator recurrence is recorded in Appendix~\ref{app:scaled_fast_weight_details}; in signed-feature settings, the denominator caveat in Appendix~\ref{app:normalized_linattn_caveat} applies.

\paragraph{Scaled DeltaNet and the common ridge parameterization.}
For regression-style fast weights, we use the same scaled write-feature $\mathbf{x}_t$ inside the NLMS update:
\[
\Sbb_t=(1-\eta_t\lambda_t)\Sbb_{t-1}+\eta_t\,\mathbf{x}_t\big(\vb_t-\Sbb_{t-1}^\top\mathbf{x}_t\big)^\top,
\qquad
\eta_t=\frac{\beta_t}{\|\mathbf{x}_t\|_2^2+\lambda_t+\varepsilon}.
\]

In scaled implementations, the network may emit a dimensionless base ridge $\bar\lambda_t$ that is converted to the actual coefficient used by the recurrence, $\lambda_t=\bar\lambda_t E_t$, where $E_t$ is the appropriate normalization statistic: $\|\mathbf{x}_t\|_2^2$ for \Falcon-2 and $\mu_t^{(B)}:=\lambda_{\max}(\bar{\Cb}_t^{(B)})$ for \Falcon-3. We then set
\[
\eta_t=\frac{\beta_t}{E_t+\lambda_t+\varepsilon},
\qquad
\alpha_t:=\eta_t\lambda_t,
\qquad
\gamma_t:=1-\alpha_t.
\]
Here $E_t=\|\mathbf{x}_t\|_2^2$ for the non-sliding scalar/per-column rules
(\Falcon-1/\Falcon-2 and \Falcon-1A/\Falcon-2A), while $E_t=\mu_t^{(B)}$ for
\Falcon-3 and $E_t=\bar E_t^{(B)}$ for \Falcon-3A. For regression, $E_t$ is the
local smoothness scale of the data term; for inner-product writes, the
corresponding energy statistics in Sections~\ref{sec:inner_product_loss}
and~\ref{sec:methods_sliding_ip} are practical write-magnitude controls. Exact
scale-robustness identities, detached-statistics details, the normalized
linear-attention recurrence, and log-space positive-decay handling are deferred
to Appendix~\ref{app:scaled_fast_weight_details} and Appendix~\ref{app:extra_impl_remarks}.

\subsection{Regression Loss (Delta Network)}
\label{sec:methods_regression}

We formulate the state update as an autoregressive linear regression problem. At time step $t$, let the state be $\Sbb_{t-1} \in \mathbb{R}^{d_x \times d_v}$, the write feature be $\mathbf{x}_t \triangleq \phi(\mathbf{k}_{t-1}) \in \mathbb{R}^{d_x}$, and the target be $\mathbf{y}_t \triangleq \vb_t \in \mathbb{R}^{d_v}$. (In the unkernelized case, $\phi$ is the identity and $d_x=d$.) The state $\Sbb_{t-1}$ acts as a linear predictor mapping the prefix feature $\mathbf{x}_t$ to the newly observed target $\mathbf{y}_t$.

The instantaneous squared-error loss is
\begin{align}
f_t(\Sbb) := \frac{1}{2}\big\|\Sbb^\top \mathbf{x}_t - \mathbf{y}_t\big\|_2^2.
\end{align}
Evaluated at the pre-update state, the residual is $\mathbf{r}_t \triangleq \mathbf{y}_t - \Sbb_{t-1}^\top \mathbf{x}_t$. The gradient with respect to the state is:
\begin{align}
\nabla_{\Sbb} f_t(\Sbb_{t-1}) = \mathbf{x}_t (\Sbb_{t-1}^\top \mathbf{x}_t - \mathbf{y}_t)^\top = -\,\mathbf{x}_t\mathbf{r}_t^\top.
\end{align}

\paragraph{Delta update (standard vs.\ next-latent).}
A single step of Online Gradient Descent (OGD) with learning rate $\eta_t$ yields the Delta update:
\begin{align}
\Sbb_t
&= \Sbb_{t-1} - \eta_t \nabla_{\Sbb} f_t(\Sbb_{t-1}) \nonumber \\
&= \Sbb_{t-1} + \eta_t \mathbf{x}_t \mathbf{r}_t^\top \nonumber \\
&= (\Ib_{d_x} - \eta_t \mathbf{x}_t\mathbf{x}_t^\top)\Sbb_{t-1} + \eta_t \mathbf{x}_t\mathbf{y}_t^\top.
\end{align}
Under next-latent alignment we have $\mathbf{x}_t=\phi(\mathbf{k}_{t-1})$ and $\mathbf{y}_t=\vb_t$. Replacing $\mathbf{k}_{t-1}$ with $\mathbf{k}_t$ (equivalently, $\mathbf{x}_t\leftarrow \phi(\mathbf{k}_t)$) recovers the unshifted DeltaNet update of~\citet{schlag2021linear}.

\paragraph{$L_2$ Regularization.}
Adding an $L_2$ penalty $\frac{\lambda_t}{2}\|\Sbb\|_F^2$ to the instantaneous loss yields the shrinkage term $-\eta_t \lambda_t \Sbb_{t-1}$ in the online update (cf.\ Eq.~\eqref{eq:ogd}). Under normalized step sizes, this results in the multiplicative factor $(1-\eta_t\lambda_t)$, which we treat as a simple and controllable forgetting mechanism.

% =========================================================
% FIGURE: Three equivalent views of FALCON-1
%   A. Recurrent Form (causal scan with rank-one edit + write)
%   B. Parallel Form  (WY representation + triangular solve)
%   C. Chunk-wise Parallel Form (per-chunk Gram + TriSolve)
%
% Requires the same preamble as Figure: Linear Attention Equivalence
%   - colors: appleBlue, appleLightBlue, appleRed, appleLightRed,
%     appleGreen, appleLightGreen, appleGray, appleLightGray,
%     appleDarkGray, appleDark
%   - tikzlibraries: positioning, calc, shapes.geometric,
%     arrows.meta, backgrounds, fit, matrix, shadows.blur
% =========================================================
\begin{figure}[ht!]
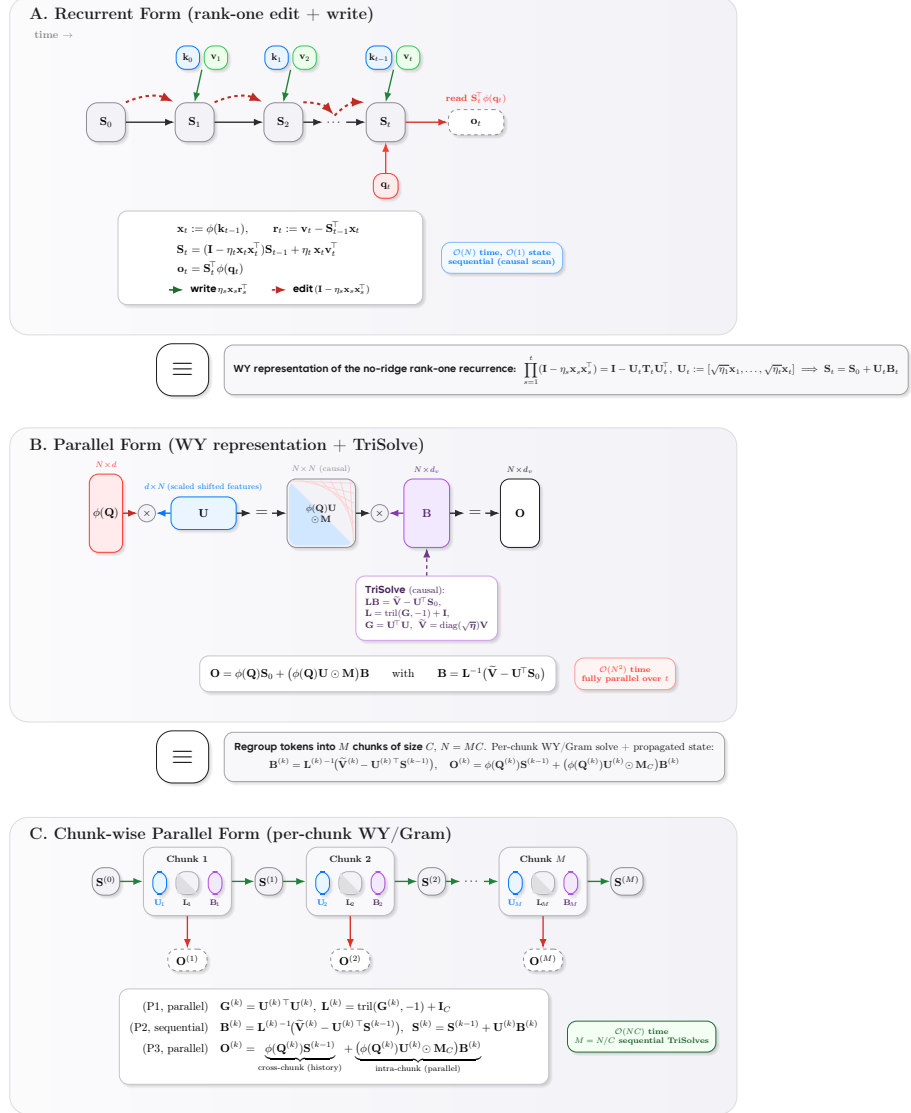

\centering
% Purple accent for the *new* DeltaNet object: the TriSolve result B.
% (Definitions are local to this figure; safe to define once if reused.)
\providecolor{applePurple}{RGB}{175, 82, 222}
\providecolor{appleLightPurple}{RGB}{243, 233, 255}
\resizebox{0.72\linewidth}{!}{%

}
\caption{\textbf{Three equivalent views of the \Falcon-1 rank-one kernel.}
For readability, the displayed WY algebra uses the no-ridge case $\lambda_t=0$; the $\lambda_t>0$ recurrence is reduced exactly to this form by the chunk-local positive-decay renormalization in Appendix~\ref{app:deltanet2_lite}, while the underlying no-ridge WY kernel is given in Appendix~\ref{sec:delta-parallel}.
(A) The recurrent form maintains a fixed-size matrix state $\mathbf{S}_t\in\mathbb{R}^{d_x\times d_v}$ and applies the rank-one edit $(\mathbf{I}-\eta_t\mathbf{x}_t\mathbf{x}_t^{\!\top})$ before writing the new target $\mathbf{v}_t$ under the shifted feature $\mathbf{x}_t:=\phi(\mathbf{k}_{t-1})$.
(B) With $\mathbf{U}:=\mathbf{X}^{\!\top}\operatorname{Diag}(\sqrt{\boldsymbol{\eta}})$, the WY form converts the recurrence into the masked pattern $\phi(\mathbf{Q})\mathbf{S}_0+(\phi(\mathbf{Q})\mathbf{U}\odot\mathbf{M})\mathbf{B}$.
(C) The chunk-wise form splits the sequence into $M$ chunks of size $C$ and performs one multi-right-hand-side triangular solve per chunk.}
\label{fig:deltanet_equivalence}
\end{figure}

% =========================================================
% Falcon-2: Three equivalent views (refined for tight spacing,
% no title/plot overlap, no equiv-strip/box-border collision)
% =========================================================
\begin{figure}[ht!]
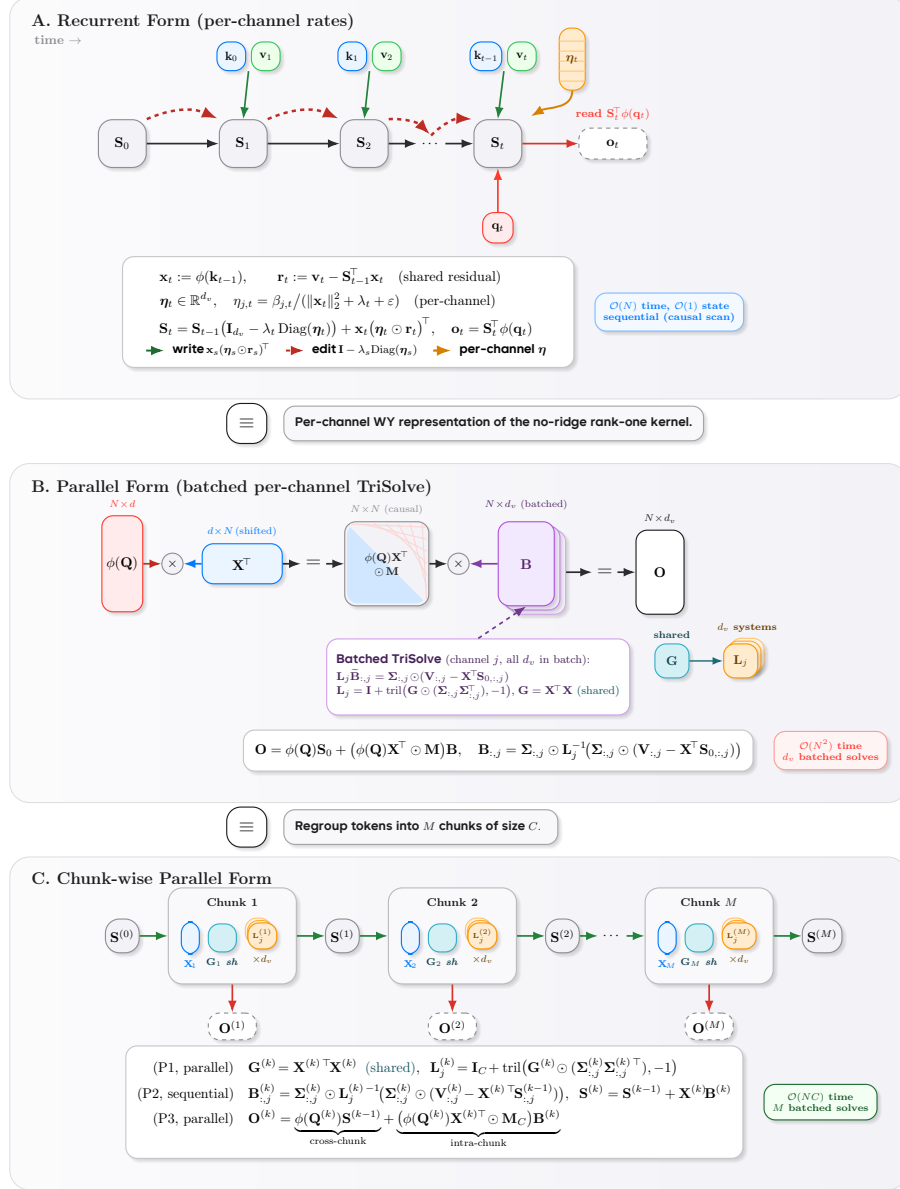

\centering
\resizebox{0.72\linewidth}{!}{%

}
\caption{\textbf{Three views of \Falcon-2 and its batched WY kernel.}
(A) The general recurrent form maintains a fixed-size matrix state $\mathbf{S}_t\in\mathbb{R}^{d_x\times d_v}$. Unlike \Falcon-1, the step size is a vector $\boldsymbol{\eta}_t\in\mathbb{R}^{d_v}$, so each value channel follows its own plasticity trajectory while sharing the write-feature direction $\mathbf{x}_t=\phi(\mathbf{k}_{t-1})$ and residual $\mathbf{r}_t=\mathbf{v}_t-\mathbf{S}_{t-1}^{\!\top}\mathbf{x}_t$.
(B) Panels (B)--(C) display the no-ridge rank-one kernel. For $\lambda_t>0$, the same equations apply to the exactly equivalent channel-wise renormalized variables of Appendix~\ref{app:deltanet2_parallel}. The state factors as $\mathbf{S}_t=\mathbf{S}_0+\mathbf{X}_t^{\!\top}\mathbf{B}_t$. Because $\boldsymbol{\eta}$ is per-channel, $\mathbf{B}$ is obtained from one unit-lower-triangular system $\mathbf{L}_j$ per value channel, with all channels sharing the write-feature Gram $\mathbf{G}=\mathbf{X}^{\!\top}\mathbf{X}$.
(C) Within each chunk, the Gram $\mathbf{G}^{(k)}$ is built once; the $d_v$ channel-specific systems are then solved in one batched TriSolve, and a fixed-size matrix state is propagated across chunks.}
\label{fig:falcon2_equivalence}
\end{figure}

\paragraph{Falcon-2: Adaptive Learning Rates.}
We propose \textbf{Falcon-2}, which combines NLMS-style normalization with per-channel adaptive learning rates.
In Appendix~\ref{app:deltanet2_parallel}, we derive a vectorized dual form that makes column-wise adaptivity computationally tractable on GPUs.
This allows the step size to be a vector $\boldsymbol{\eta}_t \in \mathbb{R}^{d_v}$, tailoring the update magnitude for each value channel (column of $\Sbb$) independently.
In the models studied here, we use the standard multi-head factorization: each head carries its own fast state, and the same per-column rule is applied independently within that head.

\paragraph{Per-column NLMS step sizes.}
Concretely, we use an NLMS-style normalizer shared across columns and learn per-column gains (equivalently, per output-feature / column gains):
\[
\eta_{j,t} = \frac{\beta_{j,t}}{\|\mathbf{x}_t\|_2^2+\lambda_t+\varepsilon},
\qquad \beta_{j,t}\in(0,2),\ \varepsilon>0.
\]
When $\lambda_t=0$, $\varepsilon=0$, and the gains are tied across channels ($\beta_{j,t}\equiv\beta_t$), this reduces to the classical scalar NLMS step size; otherwise, it is a column-wise NLMS generalization with a shared normalizer.
Since the squared-error (ridge) loss decomposes across value coordinates (columns of $\Sbb$), the per-step descent argument from Lemma~\ref{lem:per-step-descent} applies column-wise provided $0<\beta_{j,t}<2$ for all $j$.
Equivalently, Falcon-2 is a collection of $d_v$ independent scalar-step updates on a separable objective, rather than a single Frobenius-gradient step with a full matrix-valued learning rate.

\paragraph{Update rule.}
Let $\boldsymbol{\eta}_t=(\eta_{1,t},\ldots,\eta_{d_v,t})^\top$. The update can be written compactly as:
\begin{align}
\Sbb_t
&= \Sbb_{t-1}\Big(\Ib_{d_v}-\lambda_t\,\operatorname{Diag}(\boldsymbol{\eta}_t)\Big) + \mathbf{x}_t\big(\boldsymbol{\eta}_t \odot \mathbf{r}_t\big)^\top,
\end{align}
An expanded edit decomposition, separating the per-column shrinkage path from the feature-direction edit, is deferred to Appendix~\ref{app:deltanet2_parallel}.

\paragraph{Relation to RWKV-7 and Kimi Linear Attention.}
A brief comparison to RWKV-7 and Kimi Linear Attention is moved to Appendix~\ref{app:extra_impl_remarks}. 

The equivalent column-wise recursion and the log-space positive-decay renormalization used by the chunk-parallel kernels are deferred to Appendix~\ref{app:deltanet2_parallel}.

\paragraph{Chunk-parallel implementation.}
Falcon-2 can be trained sequence-parallel by chunking the length axis and using a Gram/WY representation within each chunk.
In the multi-head setting used here, this computation is applied independently within each head.
Within a chunk, the WY/Gram form builds a shared key Gram matrix and a channel-dependent unit-lower-triangular system. The injected-value and projected-history paths share this triangular factor, so the implementation solves one merged residual system rather than two. This removes one batched TriSolve per chunk in the forward pass without changing the recurrence or asymptotic complexity, and mirrors the single-inversion form emphasized in Comba~\citep{hu2025comba}.
Algorithm~\ref{alg:deltanet2-parallelized-decay} gives the single-head Falcon-2 chunk-wise forward pass with the same positive-decay convention used by the implementation. The no-ridge rank-one case is recovered by setting $\lambda_t=0$, in which case $\gamma_{t,j}=1$ and the chunk-local rescaling becomes the identity. Appendix~\ref{app:deltanet2_parallel} gives the exact WY/Gram algebra and complexity analysis.

\begin{breakablealgorithm}
\caption{\Falcon-2 (Chunk-parallel Forward)}
\label{alg:deltanet2-parallelized-decay}
\small
\begin{algorithmic}[1]
\Require Shifted write features $\mathbf{K}\in\mathbb{R}^{L\times d_x}$ with row $t$ equal to $\mathbf{x}_t^\top$, query features $\mathbf{Q}\in\mathbb{R}^{L\times d_x}$, values $\mathbf{V}\in\mathbb{R}^{L\times d_v}$, per-channel step sizes $\boldsymbol{\eta}\in\mathbb{R}^{L\times d_v}$ with $\eta_{t,j}\ge0$ (and $\eta_{1,:}=0$ for the boundary sentinel), ridge coefficients $\boldsymbol{\lambda}\in\mathbb{R}_{\ge0}^{L}$, decay floor $\varepsilon_\gamma>0$, chunk size $C$ (assume $C\mid L$), and initial state $\Sbb_{\rm init}\in\mathbb{R}^{d_x\times d_v}$.
\Ensure Outputs $\mathbf{O}\in\mathbb{R}^{L\times d_v}$ for the per-channel recurrence
$\mathbf{s}_{t,j}=((1-\alpha_{t,j})\Ib-\eta_{t,j}\mathbf{x}_t\mathbf{x}_t^\top)\mathbf{s}_{t-1,j}+\eta_{t,j}v_{t,j}\mathbf{x}_t$,
where $\alpha_{t,j}:=\min(\lambda_t\eta_{t,j},1-\varepsilon_\gamma)$; if the clamp is inactive, this is the ridge update with coefficient $\lambda_t$, and otherwise it uses the effective shrinkage coefficient $\alpha_{t,j}/\eta_{t,j}$ when $\eta_{t,j}>0$.
\State Partition the length-$L$ axis into $M=L/C$ contiguous chunks.
\State $\Sbb_{\rm in}\gets\Sbb_{\rm init}$ and initialize $\mathbf{O}\gets\mathbf{0}$.
\For{$m=1,\ldots,M$}
  \State Slice $\mathbf{K}_{(m)},\mathbf{Q}_{(m)}\in\mathbb{R}^{C\times d_x}$, $\mathbf{V}_{(m)}\in\mathbb{R}^{C\times d_v}$, $\boldsymbol{\eta}_{(m)}\in\mathbb{R}^{C\times d_v}$, and $\boldsymbol{\lambda}_{(m)}\in\mathbb{R}^{C}$.
  \State Set $\mathbf{K}\gets\mathbf{K}_{(m)}^\top\in\mathbb{R}^{d_x\times C}$ and $\mathbf{Q}\gets\mathbf{Q}_{(m)}^\top\in\mathbb{R}^{d_x\times C}$.
  \State $\boldsymbol{\alpha}\gets\min\big(\boldsymbol{\eta}_{(m)}\odot(\boldsymbol{\lambda}_{(m)}\mathbf{1}_{d_v}^\top),\,1-\varepsilon_\gamma\big)$.
  \State $\log\boldsymbol{\gamma}\gets\operatorname{log1p}(-\boldsymbol{\alpha})$ and $\hat{\boldsymbol{\eta}}\gets\boldsymbol{\eta}_{(m)}\odot\exp(-\log\boldsymbol{\gamma})$.
  \State Compute chunk-local log-prefixes $u_{0,:}\gets\mathbf{0}$ and, for $i=1{:}C$, $u_{i,:}\gets u_{i-1,:}+\log\boldsymbol{\gamma}_{i,:}$ and $\boldsymbol{\delta}_{i,:}\gets\exp(u_{i,:})$.
  \State Rescale values locally: $\hat{\mathbf{V}}_{i,:}\gets\mathbf{V}_{(m),i,:}\odot\exp(-u_{i-1,:})$ for $i=1{:}C$.
  \State $\boldsymbol{\Sigma}\gets\sqrt{\hat{\boldsymbol{\eta}}}$, $\mathbf{G}\gets\mathbf{K}^\top\mathbf{K}$, and $\mathbf{M}\gets\operatorname{tril}(\mathbf{Q}^\top\mathbf{K},0)$.
  \State $\mathbf{H}\gets\mathbf{Q}^\top\Sbb_{\rm in}$ and $\mathbf{P}\gets\mathbf{K}^\top\Sbb_{\rm in}$.
  \State For each value channel $j$, form $\mathbf{L}_j\gets\Ib_C+\operatorname{tril}\big(\mathbf{G}\odot(\boldsymbol{\Sigma}_{:,j}\boldsymbol{\Sigma}_{:,j}^\top),-1\big)$.
  \State $\widetilde{\mathbf{B}}\gets\mathrm{BatchedSolveTri}\big(\{\mathbf{L}_j\}_{j=1}^{d_v},\,\boldsymbol{\Sigma}\odot(\hat{\mathbf{V}}-\mathbf{P})\big)$.
  \State $\mathbf{B}\gets\boldsymbol{\Sigma}\odot\widetilde{\mathbf{B}}$.
  \State $\hat{\mathbf{O}}^{(m)}\gets\mathbf{H}+\mathbf{M}\mathbf{B}$ and $\hat{\Sbb}_{\rm out}\gets\Sbb_{\rm in}+\mathbf{K}\mathbf{B}$.
  \State Write $\mathbf{O}_{(m),i,:}\gets\hat{\mathbf{O}}^{(m)}_{i,:}\odot\boldsymbol{\delta}_{i,:}$ for $i=1{:}C$.
  \State $\Sbb_{\rm in}\gets\hat{\Sbb}_{\rm out}\operatorname{Diag}(\boldsymbol{\delta}_{C,:})$.
\EndFor
\State \Return $\mathbf{O}$.
\end{algorithmic}
\end{breakablealgorithm}

\begin{figure}[ht!]
\centering
\resizebox{0.75\linewidth}{!}{%
\begin{tikzpicture}[
    font=,
    >={Latex[length=2.4mm,width=1.8mm]},
    card/.style={rounded corners=8pt, fill=white, draw=appleGray!45, line width=0.6pt},
    pill/.style 2 args={rounded corners=3pt, fill=#1!10, draw=#1, line width=0.4pt,
                        minimum width=#2, minimum height=0.46cm,
                        font=\scriptsize, text=#1, inner sep=2pt},
    arr/.style={->, line width=0.7pt, draw=appleGray!70, rounded corners=2pt},
    arrW/.style={->, line width=1pt, draw=appleOrange, rounded corners=2pt},
    arrR/.style={->, line width=1pt, draw=appleBlue, rounded corners=2pt},
    arrEta/.style={->, line width=0.8pt, draw=appleOrange!75, rounded corners=2pt},
    panelbg/.style={rounded corners=12pt, fill=appleLightGray!22, draw=appleGray!35, line width=0.6pt},
]

% ===== Panel backgrounds =====
\begin{scope}[on background layer]
  \node[panelbg, minimum width=15cm, minimum height=5.8cm] at (7.5,  4.0) {}; % A
  \node[panelbg, minimum width=15cm, minimum height=5.8cm] at (7.5, -2.3) {}; % B
  \node[panelbg, minimum width=15cm, minimum height=5.8cm] at (7.5, -8.6) {}; % C
\end{scope}

% ============================================================
% PANEL A: FALCON-1A -- scalar inner-product write
% ============================================================
\node[font=\large\bfseries, text=appleDark] at (7.5, 6.55)
    {A. Falcon-1A: scalar inner-product write};
\node[font=\footnotesize, text=appleGray] at (7.5, 6.15)
    {one scalar $\eta_t$ shared across all $d_v$ value channels (additive write)};

\node[card, minimum width=2.3cm, minimum height=2.9cm] (inputA) at (1.5, 3.3) {};
\node[font=\footnotesize\bfseries, text=appleDark] at ($(inputA.north)+(0,-0.28)$) {Token $t$};
\node[pill={appleBlue}{1.85cm}] (xtA) at ($(inputA.center)+(0,0.55)$)
    {$\mathbf{x}_t\!=\!\phi(\mathbf{k}_{t-1})$};
\node[pill={appleGreen!55!black}{1.85cm}] (ytA) at ($(inputA.center)+(0,0.00)$)
    {$\mathbf{y}_t\!=\!\mathbf{v}_t$};
\node[pill={applePurple}{1.85cm}] (qtA) at ($(inputA.center)+(0,-0.55)$)
    {$\mathbf{q}_t$};

\def\cw{0.24}\def\rh{0.24}
\pgfmathsetmacro{\matAX}{4.7}\pgfmathsetmacro{\matAY}{3.25}
\node[card, minimum width=2.7cm, minimum height=3.3cm] (stateAprev) at (\matAX, \matAY) {};
\node[font=\footnotesize\bfseries, text=appleDark] at (\matAX, \matAY+1.35) {$\mathbf{S}_{t-1}$};
\node[font=\tiny, text=appleGray] at (\matAX, \matAY+1.10) {$d_x\!\times\!d_v$};
\pgfmathsetmacro{\gridAleft}{\matAX - 4*\cw}
\pgfmathsetmacro{\gridAbot}{\matAY - 0.05 - 3*\rh}
\foreach \r in {0,...,5}{\foreach \c in {0,...,7}{%
  \pgfmathsetmacro{\cx}{\gridAleft + \c*\cw}\pgfmathsetmacro{\cy}{\gridAbot + \r*\rh}%
  \fill[appleLightBlue, rounded corners=0.6pt]
    (\cx+0.025, \cy+0.025) rectangle (\cx+\cw-0.025, \cy+\rh-0.025);}}

\node[circle, draw=appleOrange, fill=appleLightOrange, line width=0.7pt,
      minimum size=0.95cm, font=\footnotesize\bfseries, text=appleOrange!40!black]
    (etaA) at (8.0, 5.30) {$\eta_t$};

\node[card, minimum width=2.5cm, minimum height=1.5cm] (formulaA) at (8.0, 3.3) {};
\node[font=\footnotesize\bfseries, text=appleDark] at ($(formulaA.north)+(0,-0.25)$) {Update};
\node[font=\scriptsize, text=appleDarkGray, align=center] at ($(formulaA.center)+(0,-0.07)$)
    {$\mathbf{S}_t\!=\!(1{-}\eta_t\lambda_t)\mathbf{S}_{t-1}$\\[-1pt]
     $+\,\eta_t\,\mathbf{x}_t\mathbf{y}_t^{\!\top}$};

\pgfmathsetmacro{\matBX}{11.3}
\node[card, minimum width=2.7cm, minimum height=3.3cm] (stateAcur) at (\matBX, \matAY) {};
\node[font=\footnotesize\bfseries, text=appleDark] at (\matBX, \matAY+1.35) {$\mathbf{S}_t$};
\node[font=\tiny, text=appleGray] at (\matBX, \matAY+1.10) {all cols updated equally};
\pgfmathsetmacro{\gridBleft}{\matBX - 4*\cw}
\foreach \r in {0,...,5}{\foreach \c in {0,...,7}{%
  \pgfmathsetmacro{\cx}{\gridBleft + \c*\cw}\pgfmathsetmacro{\cy}{\gridAbot + \r*\rh}%
  \fill[appleOrange!55, rounded corners=0.6pt]
    (\cx+0.025, \cy+0.025) rectangle (\cx+\cw-0.025, \cy+\rh-0.025);}}

\pgfmathsetmacro{\barAY}{\gridAbot - 0.32}
\foreach \c in {0,...,7}{\pgfmathsetmacro{\cx}{\gridBleft + \c*\cw}%
  \fill[appleOrange, rounded corners=0.6pt]
    (\cx+0.03, \barAY) rectangle (\cx+\cw-0.03, \barAY+0.20);}
\node[font=\tiny, text=appleOrange!40!black] at (\matBX, \barAY-0.17)
    {$\eta_t$ per col (uniform)};

\node[card, minimum width=1.9cm, minimum height=1.4cm] (outputA) at (13.95, \matAY) {};
\node[font=\footnotesize\bfseries, text=appleDark] at ($(outputA.north)+(0,-0.25)$) {Output};
\node[pill={applePurple}{1.6cm}] at ($(outputA.center)+(0,-0.05)$)
    {$\mathbf{S}_t^{\!\top}\mathbf{q}_t$};

\draw[arr]    (xtA.east) -- ++(0.20,0) |- (formulaA.west |- xtA);
\draw[arr]    (ytA.east) -- ++(0.20,0) |- (formulaA.west |- ytA);
\draw[arr]    (stateAprev.east) -- (formulaA.west);
\draw[arrW]   (formulaA.east)   -- (stateAcur.west);
\draw[arrR]   (stateAcur.east)  -- (outputA.west);
\draw[arrEta] (etaA.south)      -- (formulaA.north);
\draw[arrR]   (qtA.south) -- ++(0,-1.20) -| (outputA.south);

% ============================================================
% PANEL B: FALCON-2A -- per-channel inner-product write (yshift=-6.3cm)
% ============================================================
\begin{scope}[yshift=-6.3cm]

\node[font=\large\bfseries, text=appleDark] at (7.5, 6.55)
    {B. Falcon-2A: per-channel inner-product write};
\node[font=\footnotesize, text=appleGray] at (7.5, 6.15)
    {vector $\boldsymbol{\eta}_t\in\mathbb{R}^{d_v}$, one write gain per value channel};

\node[card, minimum width=2.3cm, minimum height=2.9cm] (inputB) at (1.5, 3.3) {};
\node[font=\footnotesize\bfseries, text=appleDark] at ($(inputB.north)+(0,-0.28)$) {Token $t$};
\node[pill={appleBlue}{1.85cm}] (xtB) at ($(inputB.center)+(0,0.55)$)
    {$\mathbf{x}_t\!=\!\phi(\mathbf{k}_{t-1})$};
\node[pill={appleGreen!55!black}{1.85cm}] (ytB) at ($(inputB.center)+(0,0.00)$)
    {$\mathbf{y}_t\!=\!\mathbf{v}_t$};
\node[pill={applePurple}{1.85cm}] (qtB) at ($(inputB.center)+(0,-0.55)$)
    {$\mathbf{q}_t$};

\node[card, minimum width=2.7cm, minimum height=3.3cm] (stateBprev) at (\matAX, \matAY) {};
\node[font=\footnotesize\bfseries, text=appleDark] at (\matAX, \matAY+1.35) {$\mathbf{S}_{t-1}$};
\node[font=\tiny, text=appleGray] at (\matAX, \matAY+1.10) {$d_x\!\times\!d_v$};
\foreach \r in {0,...,5}{\foreach \c in {0,...,7}{%
  \pgfmathsetmacro{\cx}{\gridAleft + \c*\cw}\pgfmathsetmacro{\cy}{\gridAbot + \r*\rh}%
  \fill[appleLightBlue, rounded corners=0.6pt]
    (\cx+0.025, \cy+0.025) rectangle (\cx+\cw-0.025, \cy+\rh-0.025);}}

% Vector \eta_t shown as a horizontal strip with per-channel intensities
\def\etaBY{5.30}
\node[font=\scriptsize\bfseries, text=appleOrange!40!black]
    at (8.0, \etaBY+0.42) {$\boldsymbol{\eta}_t\!\in\!\mathbb{R}^{d_v}$};
\foreach \i/\op in {0/22, 1/78, 2/38, 3/88, 4/18, 5/62, 6/48, 7/32}{%
  \pgfmathsetmacro{\cx}{8.0 - 4*\cw + \i*\cw}%
  \fill[appleOrange!\op, rounded corners=0.6pt]
    (\cx+0.03, \etaBY-0.18) rectangle (\cx+\cw-0.03, \etaBY+0.18);}
\pgfmathsetmacro{\etaBL}{8.0-4*\cw-0.02}
\pgfmathsetmacro{\etaBR}{8.0+4*\cw+0.02}
\pgfmathsetmacro{\etaBT}{\etaBY+0.21}
\pgfmathsetmacro{\etaBB}{\etaBY-0.21}
\draw[draw=appleOrange!50, line width=0.4pt, rounded corners=1.5pt]
    (\etaBL, \etaBB) rectangle (\etaBR, \etaBT);

\node[card, minimum width=2.9cm, minimum height=1.5cm] (formulaB) at (8.0, 3.3) {};
\node[font=\footnotesize\bfseries, text=appleDark] at ($(formulaB.north)+(0,-0.25)$) {Update};
\node[font=\scriptsize, text=appleDarkGray, align=center] at ($(formulaB.center)+(0,-0.07)$)
    {$\mathbf{S}_t\!=\!\mathbf{S}_{t-1}\!\bigl(\mathbf{I}{-}\lambda_t\!\operatorname{Diag}(\boldsymbol{\eta}_t)\bigr)$\\[-1pt]
     $+\,\mathbf{x}_t(\boldsymbol{\eta}_t\!\odot\!\mathbf{y}_t)^{\!\top}$};

\node[card, minimum width=2.7cm, minimum height=3.3cm] (stateBcur) at (\matBX, \matAY) {};
\node[font=\footnotesize\bfseries, text=appleDark] at (\matBX, \matAY+1.35) {$\mathbf{S}_t$};
\node[font=\tiny, text=appleGray] at (\matBX, \matAY+1.10) {cols updated by $\eta_{j,t}$};
\foreach \c/\op in {0/22, 1/78, 2/38, 3/88, 4/18, 5/62, 6/48, 7/32}{%
  \foreach \r in {0,...,5}{%
    \pgfmathsetmacro{\cx}{\gridBleft + \c*\cw}\pgfmathsetmacro{\cy}{\gridAbot + \r*\rh}%
    \fill[appleOrange!\op, rounded corners=0.6pt]
      (\cx+0.025, \cy+0.025) rectangle (\cx+\cw-0.025, \cy+\rh-0.025);}}

\foreach \c/\h in {0/0.06, 1/0.22, 2/0.11, 3/0.25, 4/0.05, 5/0.18, 6/0.14, 7/0.09}{%
  \pgfmathsetmacro{\cx}{\gridBleft + \c*\cw}%
  \fill[appleOrange, rounded corners=0.6pt]
    (\cx+0.03, \barAY) rectangle (\cx+\cw-0.03, \barAY+\h);}
\node[font=\tiny, text=appleOrange!40!black] at (\matBX, \barAY-0.17)
    {$\eta_{j,t}$ per col (varies)};

\node[card, minimum width=1.9cm, minimum height=1.4cm] (outputB) at (13.95, \matAY) {};
\node[font=\footnotesize\bfseries, text=appleDark] at ($(outputB.north)+(0,-0.25)$) {Output};
\node[pill={applePurple}{1.6cm}] at ($(outputB.center)+(0,-0.05)$)
    {$\mathbf{S}_t^{\!\top}\mathbf{q}_t$};

\draw[arr]    (xtB.east) -- ++(0.20,0) |- (formulaB.west |- xtB);
\draw[arr]    (ytB.east) -- ++(0.20,0) |- (formulaB.west |- ytB);
\draw[arr]    (stateBprev.east) -- (formulaB.west);
\draw[arrW]   (formulaB.east)   -- (stateBcur.west);
\draw[arrR]   (stateBcur.east)  -- (outputB.west);
\draw[arrEta] (8.0, \etaBB)     -- (formulaB.north);
\draw[arrR]   (qtB.south) -- ++(0,-1.20) -| (outputB.south);

\end{scope}

% ============================================================
% PANEL C: FALCON-3A -- sliding-window inner-product write (yshift=-12.6cm)
% ============================================================
\begin{scope}[yshift=-12.6cm]

\node[font=\large\bfseries, text=appleDark] at (7.5, 6.55)
    {C. Falcon-3A: sliding-window inner-product write};
\node[font=\footnotesize, text=appleGray] at (7.5, 6.15)
    {scalar $\eta_t$, write $\bar{\mathbf{N}}_t^{(B)}$ averaged over window $\mathcal{I}_t$ of size $B$};

% Input card -- Window with explicit B mini-cells
\node[card, minimum width=2.3cm, minimum height=2.9cm] (inputC) at (1.5, 3.3) {};
\node[font=\footnotesize\bfseries, text=appleDark] at ($(inputC.north)+(0,-0.28)$)
    {Window $\mathcal{I}_t$ ($B\!=\!4$)};

% Window-cell geometry (4 cells, 0.36cm wide, 0.05cm gap)
\pgfmathsetmacro{\miniW}{0.36}
\pgfmathsetmacro{\miniG}{0.05}
\pgfmathsetmacro{\stripL}{1.5 - 2*\miniW - 1.5*\miniG}
\pgfmathsetmacro{\xLastC}{\stripL + 3*(\miniW+\miniG)}

% --- X-window: 4 blue mini-cells, gradient (older -> newer) ---
\foreach \i/\op in {0/22, 1/42, 2/65, 3/95}{%
  \pgfmathsetmacro{\xL}{\stripL + \i*(\miniW+\miniG)}%
  \fill[appleBlue!\op, rounded corners=0.6pt]
    (\xL, 3.70) rectangle (\xL+\miniW, 4.00);}
% Latest cell (j=t): thicker border
\draw[appleBlue, line width=0.7pt, rounded corners=0.6pt]
  (\xLastC, 3.70) rectangle (\xLastC+\miniW, 4.00);
% Phantom anchor matching strip right edge
\node[draw=none, fill=none, inner sep=0pt,
      minimum width=1.59cm, minimum height=0.30cm]
  (xtC) at (1.5, 3.85) {};
\node[font=\scriptsize\bfseries, text=appleBlue] at (0.50, 3.85) {$\mathbf{x}_j$};

% --- V-window: 4 green mini-cells, same gradient ---
\foreach \i/\op in {0/22, 1/42, 2/65, 3/95}{%
  \pgfmathsetmacro{\xL}{\stripL + \i*(\miniW+\miniG)}%
  \fill[appleGreen!\op, rounded corners=0.6pt]
    (\xL, 3.15) rectangle (\xL+\miniW, 3.45);}
\draw[appleGreen!55!black, line width=0.7pt, rounded corners=0.6pt]
  (\xLastC, 3.15) rectangle (\xLastC+\miniW, 3.45);
\node[draw=none, fill=none, inner sep=0pt,
      minimum width=1.59cm, minimum height=0.30cm]
  (ytC) at (1.5, 3.30) {};
\node[font=\scriptsize\bfseries, text=appleGreen!55!black] at (0.50, 3.30) {$\mathbf{v}_j$};

% q_t pill (single, regular)
\node[pill={applePurple}{1.85cm}] (qtC) at (1.5, 2.75) {$\mathbf{q}_t$};

% S_{t-1} grid
\node[card, minimum width=2.7cm, minimum height=3.3cm] (stateCprev) at (\matAX, \matAY) {};
\node[font=\footnotesize\bfseries, text=appleDark] at (\matAX, \matAY+1.35) {$\mathbf{S}_{t-1}$};
\node[font=\tiny, text=appleGray] at (\matAX, \matAY+1.10) {$d_x\!\times\!d_v$};
\foreach \r in {0,...,5}{\foreach \c in {0,...,7}{%
  \pgfmathsetmacro{\cx}{\gridAleft + \c*\cw}\pgfmathsetmacro{\cy}{\gridAbot + \r*\rh}%
  \fill[appleLightBlue, rounded corners=0.6pt]
    (\cx+0.025, \cy+0.025) rectangle (\cx+\cw-0.025, \cy+\rh-0.025);}}

% scalar \eta_t
\node[circle, draw=appleOrange, fill=appleLightOrange, line width=0.7pt,
      minimum size=0.95cm, font=\footnotesize\bfseries, text=appleOrange!40!black]
    (etaC) at (8.0, 5.30) {$\eta_t$};

% Update formula C (windowed inner-product write)
\node[card, minimum width=3.3cm, minimum height=1.7cm] (formulaC) at (8.0, 3.3) {};
\node[font=\footnotesize\bfseries, text=appleDark] at ($(formulaC.north)+(0,-0.25)$) {Update};
\node[font=\scriptsize, text=appleDarkGray, align=center] at ($(formulaC.center)+(0,-0.10)$)
    {$\mathbf{S}_t\!=\!(1{-}\eta_t\lambda_t)\mathbf{S}_{t-1}\!+\!\eta_t\bar{\mathbf{N}}_t^{(B)}$\\[2pt]
     $\bar{\mathbf{N}}_t^{(B)}\!=\!\tfrac{1}{B_t}\!\sum\limits_{j\in\mathcal{I}_t}\!\mathbf{x}_j\mathbf{v}_{j}^{\!\top}$};

% S_t (uniform orange columns, sliding-window inner-product)
\node[card, minimum width=2.7cm, minimum height=3.3cm] (stateCcur) at (\matBX, \matAY) {};
\node[font=\footnotesize\bfseries, text=appleDark] at (\matBX, \matAY+1.35) {$\mathbf{S}_t$};
\node[font=\tiny, text=appleGray] at (\matBX, \matAY+1.10) {windowed write step};
\foreach \r in {0,...,5}{\foreach \c in {0,...,7}{%
  \pgfmathsetmacro{\cx}{\gridBleft + \c*\cw}\pgfmathsetmacro{\cy}{\gridAbot + \r*\rh}%
  \fill[appleOrange!55, rounded corners=0.6pt]
    (\cx+0.025, \cy+0.025) rectangle (\cx+\cw-0.025, \cy+\rh-0.025);}}

% Per-column \eta bar (uniform, scalar)
\foreach \c in {0,...,7}{\pgfmathsetmacro{\cx}{\gridBleft + \c*\cw}%
  \fill[appleOrange, rounded corners=0.6pt]
    (\cx+0.03, \barAY) rectangle (\cx+\cw-0.03, \barAY+0.20);}
\node[font=\tiny, text=appleOrange!40!black] at (\matBX, \barAY-0.17)
    {$\eta_t$ per col (uniform)};

% Output C
\node[card, minimum width=1.9cm, minimum height=1.4cm] (outputC) at (13.95, \matAY) {};
\node[font=\footnotesize\bfseries, text=appleDark] at ($(outputC.north)+(0,-0.25)$) {Output};
\node[pill={applePurple}{1.6cm}] at ($(outputC.center)+(0,-0.05)$)
    {$\mathbf{S}_t^{\!\top}\mathbf{q}_t$};

% Arrows C
\draw[arr]    (xtC.east) -- ++(0.20,0) |- (formulaC.west |- xtC);
\draw[arr]    (ytC.east) -- ++(0.20,0) |- (formulaC.west |- ytC);
\draw[arr]    (stateCprev.east) -- (formulaC.west);
\draw[arrW]   (formulaC.east)   -- (stateCcur.west);
\draw[arrR]   (stateCcur.east)  -- (outputC.west);
\draw[arrEta] (etaC.south)      -- (formulaC.north);
\draw[arrR]   (qtC.south) -- ++(0,-1.20) -| (outputC.south);

\end{scope}

\end{tikzpicture}}
\caption{\textbf{Falcon-1A vs.\ Falcon-2A vs.\ Falcon-3A.}
Each panel shows one inner-product fast-weight update step on the fixed-size
state $\mathbf{S}_{t-1}\!\in\!\mathbb{R}^{d_x\times d_v}$. We write
$\mathbf{q}_t:=\phi(\qb_t)$ for the query feature. Unlike the regression family
(Falcon-1/2/3), these rules write the target directly rather than a residual:
there is no $\mathbf{S}_{t-1}^{\!\top}\mathbf{x}_t$ subtraction inside the
write. Forgetting acts only on the carry---through a scalar factor in
Falcon-1A/Falcon-3A and through per-column factors in Falcon-2A.
\textbf{(A)~Falcon-1A} applies one scalar $\eta_t$ to all $d_v$ value channels;
every column of $\mathbf{S}_t$ receives the same energy-normalized write gain
$\eta_t=\beta_t/(\|\mathbf{x}_t\|_2^2+\lambda_t+\varepsilon)$ (uniform orange shading and flat per-column bar). \textbf{(B)~Falcon-2A} promotes the write gain to a vector
$\boldsymbol{\eta}_t\!\in\!\mathbb{R}^{d_v}$ shown as the orange strip above the
update card: each column of $\mathbf{S}_t$ receives its own gain $\eta_{j,t}$, while the write feature $\mathbf{x}_t$
and target $\mathbf{v}_t$ remain shared.
\textbf{(C)~Falcon-3A} returns to a scalar $\eta_t$ but writes the windowed
average cross-covariance
$\bar{\mathbf{N}}_t^{(B)}=B_t^{-1}\!\sum_{j\in\mathcal{I}_t}\mathbf{x}_j\mathbf{v}_j^{\!\top}$
over a sliding window $\mathcal{I}_t$ of $B$ recent causal pairs; the blue/green
strips show four $\{\mathbf{x}_j,\mathbf{v}_j\}_{j\in\mathcal{I}_t}$ entries. The denominator
$\eta_t=\beta_t/(\bar E_t^{(B)}+\lambda_t+\varepsilon)$ uses the window
energy $\bar E_t^{(B)}$ as a write-magnitude normalizer. The read
$\mathbf{o}_t=\mathbf{S}_t^{\!\top}\mathbf{q}_t$ uses the updated state in all three cases.}
\label{fig:falcon123A-overview}
\end{figure}
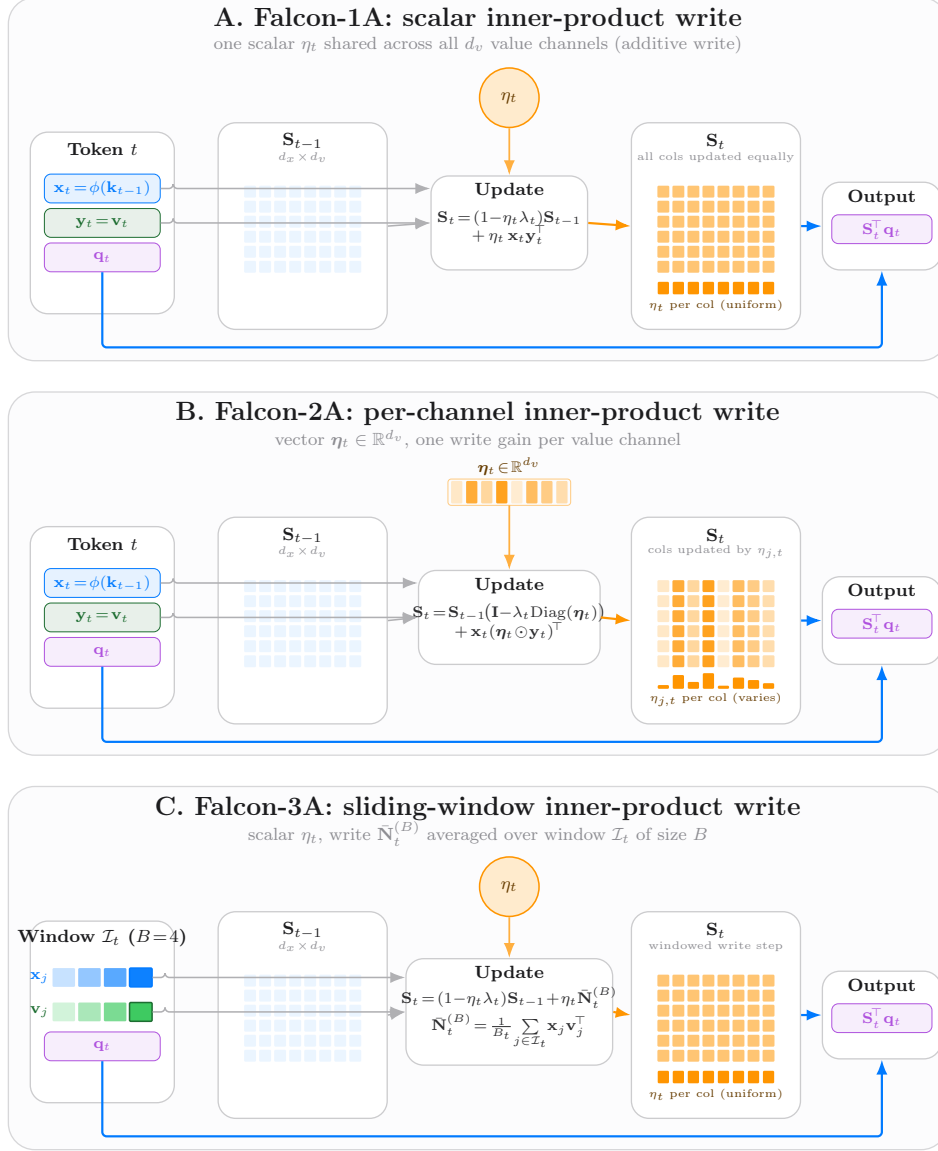

\subsection{Inner Product Loss (Linear Attention and Mamba-2)}
\label{sec:inner_product_loss}

We now consider an Inner Product objective that encourages alignment between the state prediction and the target. We write it in minimization form and optionally add an $L_2$ penalty:
\begin{equation}
\label{eq:ip_loss}
\ell_t^{\rm ip}(\Sbb)
\triangleq -\langle \Sbb^\top \mathbf{x}_t, \mathbf{y}_t\rangle + \frac{\lambda_t}{2}\|\Sbb\|_F^2,
\qquad \lambda_t\ge 0.
\end{equation}
When $\lambda_t=0$, the objective is linear in $\Sbb$ (no finite minimizer) and gradient descent reduces to purely additive Hebbian writes.

\paragraph{Standard vs.\ next-latent alignment.}
If we choose the unshifted features $(\mathbf{x}_t,\mathbf{y}_t)=(\phi(\mathbf{k}_t),\vb_t)$ (or $(\mathbf{k}_t,\vb_t)$ in the unkernelized case), the additive update below matches the usual Linear Attention write $\phi(\mathbf{k}_t)\vb_t^\top$ (Eq.~\eqref{eq:lin_attn_rec}). Our next-latent framework instead uses $(\mathbf{x}_t,\mathbf{y}_t)=(\phi(\mathbf{k}_{t-1}),\vb_t)$, yielding a one-step shifted write stream.

\paragraph{\Falcon-A variants and notation.}
We use the suffix ``A'' for the inner-product objective and keep the numerical
index aligned with the regression family. Thus \Falcon-1A is the scalar
non-sliding inner-product rule, \Falcon-2A is the per-column non-sliding
inner-product rule, and \Falcon-3A is the sliding-window inner-product rule in
Section~\ref{sec:methods_sliding_ip}. As in \Falcon-1/\Falcon-2/\Falcon-3,
$\beta_t$ denotes the dimensionless gain, $\lambda_t$ the actual shrinkage
coefficient used by the recurrence (obtained directly or via the same
scale-coupled construction described above), and $\eta_t$ the resulting step
size. For the inner-product family, this step size should be read as an
energy-normalized write gain rather than as a curvature-matched denominator.
Any decay fraction $\alpha_t:=\eta_t\lambda_t$ is derived rather than
independently parameterized.

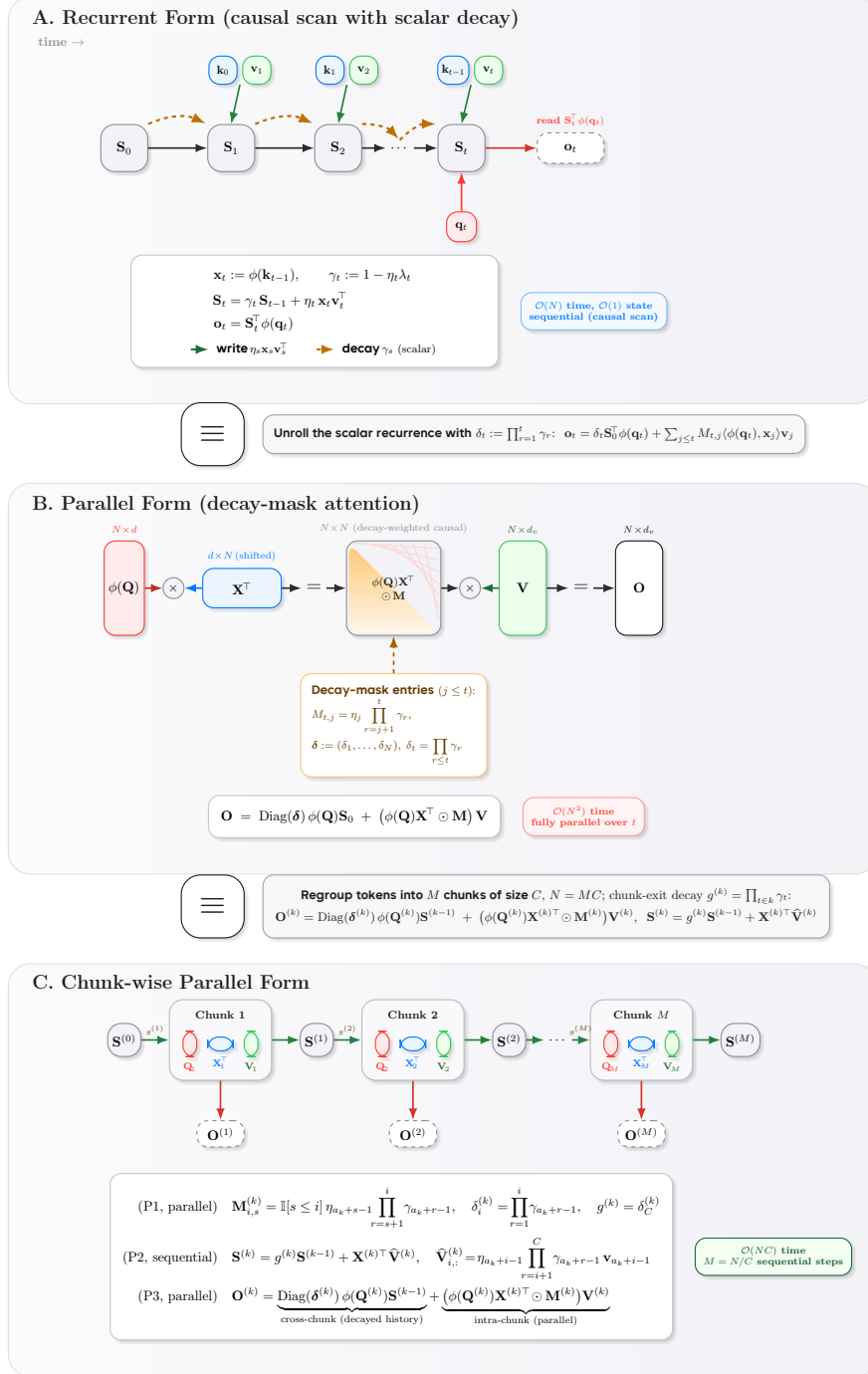
\begin{figure}[ht!]
\centering
\resizebox{0.7\linewidth}{!}{%
\begin{tikzpicture}[
font=,
>=Latex,
base_node/.style={
  thick,
  rounded corners=8pt,
  blur shadow={shadow blur steps=5, shadow opacity=15}
},
tensor/.style={
  base_node,
  draw=appleBlue,
  fill=appleLightBlue,
  minimum height=2.4em,
  minimum width=2.4em,
  font=\bfseries
},
value_tensor/.style={
  base_node,
  draw=appleGreen,
  fill=appleLightGreen,
  minimum height=2.4em,
  minimum width=2.4em,
  font=\bfseries
},
query_tensor/.style={
  base_node,
  draw=appleRed,
  fill=appleLightRed,
  minimum height=2.4em,
  minimum width=2.4em,
  font=\bfseries
},
state/.style={
  base_node,
  draw=appleGray,
  fill=appleLightGray,
  rounded corners=10pt,
  minimum height=3.4em,
  minimum width=3.4em,
  align=center,
  font=\bfseries
},
op/.style={
  circle,
  fill=appleLightGray,
  draw=appleGray,
  inner sep=2pt,
  thick,
  font=\bfseries\small
},
arrow/.style={
  ->,
  thick,
  color=appleDarkGray,
  line width=1.2pt,
  rounded corners=4pt
},
update_arrow/.style={
  arrow,
  color=appleGreen!65!black
},
read_arrow/.style={
  arrow,
  color=appleRed
},
decay_arrow/.style={
  arrow,
  dashed,
  color=appleOrange!75!black,
  line width=1.6pt,
  shorten >=1pt
},
section_label/.style={
  font=\bfseries\Large,
  anchor=north west,
  color=appleDark
},
equiv_badge/.style={
  base_node,
  draw=appleDark,
  fill=white,
  rounded corners=14pt,
  inner sep=8pt,
  minimum width=1.6cm,
  minimum height=1.6cm,
  font=\bfseries\Huge,
  text=appleDark
},
eqbox/.style={
  base_node,
  fill=white,
  draw=appleGray!50,
  inner sep=9pt
}
]

% =========================================================
% PANEL A: Recurrent Form (Top)
% =========================================================
\begin{scope}[local bounding box=panelA]

\node[font=\bfseries\small, text=appleGray] (timeLabelA) at (-1.6,2.7) {time $\rightarrow$};

\node[state] (S0) at (0,0) {$\mathbf{S}_0$};
\node[state, right=1.5cm of S0] (S1) {$\mathbf{S}_1$};
\node[state, right=1.5cm of S1] (S2) {$\mathbf{S}_2$};
\node[font=\bfseries, right=0.6cm of S2] (sdots) {$\cdots$};
\node[state, right=0.6cm of sdots] (St) {$\mathbf{S}_t$};

\node[tensor, scale=0.85, above=1.0cm of S1, xshift=-7pt] (k1) {$\mathbf{k}_{0}$};
\node[value_tensor, scale=0.85, right=0.12cm of k1] (v1) {$\mathbf{v}_1$};

\node[tensor, scale=0.85, above=1.0cm of S2, xshift=-7pt] (k2) {$\mathbf{k}_{1}$};
\node[value_tensor, scale=0.85, right=0.12cm of k2] (v2) {$\mathbf{v}_2$};

\node[tensor, scale=0.85, above=1.0cm of St, xshift=-7pt] (kt) {$\mathbf{k}_{t-1}$};
\node[value_tensor, scale=0.85, right=0.12cm of kt] (vt) {$\mathbf{v}_t$};

\draw[arrow] (S0) -- (S1);
\draw[arrow] (S1) -- (S2);
\draw[arrow] (S2) -- (sdots);
\draw[arrow] (sdots) -- (St);

\draw[update_arrow] ($(k1.south)!0.5!(v1.south)$) -- (S1.north);
\draw[update_arrow] ($(k2.south)!0.5!(v2.south)$) -- (S2.north);
\draw[update_arrow] ($(kt.south)!0.5!(vt.south)$) -- (St.north);

\draw[decay_arrow] (S0.north east) to[bend left=24] (S1.north west);
\draw[decay_arrow] (S1.north east) to[bend left=24] (S2.north west);
\draw[decay_arrow] (S2.north east) to[bend left=24] (sdots.north);
\draw[decay_arrow] (sdots.north) to[bend left=24] (St.north west);

\node[query_tensor, scale=0.9, below=1.0cm of St] (qt) {$\mathbf{q}_t$};
\node[base_node, dashed, draw=appleGray, fill=white,
      minimum width=1.7cm, minimum height=2.2em,
      right=1.3cm of St] (ytA) {$\mathbf{o}_t$};

\draw[read_arrow] (qt.north) -- (St.south);
\draw[arrow, color=appleRed!85!black] (St.east) -- (ytA.west);
\node[font=\bfseries\scriptsize, text=appleRed,
      above=0.04cm of ytA, anchor=south]
  {read $\mathbf{S}_t^{\!\top}\phi(\mathbf{q}_t)$};

\node[eqbox, below=2.1cm of S2, align=left, xshift=-0.6cm] (eqA) {%
\begin{minipage}{8.6cm}
\centering
$\displaystyle
\begin{aligned}
\mathbf{x}_t &:= \phi(\mathbf{k}_{t-1}),\qquad
\gamma_t := 1-\eta_t\lambda_t\\[2pt]
\mathbf{S}_t &= \gamma_t\,\mathbf{S}_{t-1}
              + \eta_t\,\mathbf{x}_t\mathbf{v}_t^{\!\top}\\
\mathbf{o}_t &= \mathbf{S}_t^{\!\top}\phi(\mathbf{q}_t)
\end{aligned}$\\[6pt]
{\footnotesize
\tikz[baseline=-0.55ex]\draw[->,color=appleGreen!55!black,line width=1.4pt](0,0)--(0.45,0);\;
\textbf{write}\,$\eta_s\mathbf{x}_s\mathbf{v}_s^{\!\top}$
\quad\quad
\tikz[baseline=-0.55ex]\draw[->,color=appleOrange!75!black,dashed,line width=1.4pt](0,0)--(0.45,0);\;
\textbf{decay}\,$\gamma_s$ (scalar)}
\end{minipage}
};

\node[base_node, fill=appleLightBlue!60, draw=appleBlue!50,
      inner sep=6pt, font=\bfseries\scriptsize, text=appleBlue,
      right=0.55cm of eqA, align=center] (complA)
  {$\mathcal{O}(N)$ time, $\mathcal{O}(1)$ state\\sequential (causal scan)};

\end{scope}

% =========================================================
% CENTER: Equivalence strip A -> B
% =========================================================
\coordinate (xRefEquiv) at ($(panelA.south)+(-3.5,0)$);
\coordinate (yPosAB) at ($(panelA.south)+(0,-1.7)$);
\node[equiv_badge] (equiv) at (xRefEquiv |- yPosAB) {$\equiv$};

\node[eqbox, draw=appleDark!40, fill=appleLightGray!50,
      right=0.45cm of equiv, inner sep=8pt,
      font=\small, text=appleDark, align=center] (derivBox)
{\textbf{Unroll the scalar recurrence with}
$\delta_t:=\prod_{r=1}^{t}\gamma_r$:\;
$\mathbf{o}_t = \delta_t\mathbf{S}_0^{\!\top}\phi(\mathbf{q}_t)
+\sum_{j\le t} M_{t,j}\langle\phi(\mathbf{q}_t),\mathbf{x}_j\rangle\mathbf{v}_j$%
};

% =========================================================
% PANEL B: Parallel Form
% =========================================================
\begin{scope}[shift={(0,-11.2)}, local bounding box=panelB]

\node[base_node, draw=appleRed, fill=appleLightRed,
      minimum width=1.0cm, minimum height=2.4cm,
      align=center, font=\bfseries] (Qmat) at (0,0)
  {$\phi(\mathbf{Q})$};
\node[font=\scriptsize, text=appleRed, above=0.05cm of Qmat]
  {$N\!\times\!d$};

\node[op, right=0.45cm of Qmat] (mul1) {$\times$};

\node[base_node, draw=appleBlue, fill=appleLightBlue,
      minimum width=2.0cm, minimum height=1.0cm,
      align=center, font=\bfseries, right=0.45cm of mul1] (XTmat)
  {$\mathbf{X}^{\!\top}$};
\node[font=\scriptsize, text=appleBlue, above=0.05cm of XTmat]
  {$d\!\times\!N$\,(shifted)};

\node[font=\bfseries\Large, text=appleDark, right=0.5cm of XTmat] (eq1) {$=$};

\node[base_node, draw=appleGray, fill=appleLightGray,
      minimum width=2.4cm, minimum height=2.4cm,
      align=center, font=\bfseries, right=0.5cm of eq1] (Scores)
  {};
\begin{scope}
  \clip[rounded corners=8pt] ([xshift=2pt,yshift=-2pt]Scores.north west)
        rectangle ([xshift=-2pt,yshift=2pt]Scores.south east);
  \shade[bottom color=appleOrange!8, top color=appleOrange!55,
         rounded corners=8pt]
    ([xshift=2pt,yshift=-2pt]Scores.north west) --
    ([xshift=-2pt,yshift=2pt]Scores.south east) --
    ([xshift=2pt,yshift=2pt]Scores.south west) -- cycle;
\end{scope}
\foreach \i in {0.2,0.5,0.8,1.1,1.4,1.7,2.0} {
  \draw[appleRed!22, line width=0.6pt]
    ([xshift=\i cm,yshift=-2pt]Scores.north west) --
    ([xshift=-2pt,yshift=-\i cm]Scores.north east);
}
\node[font=\bfseries\footnotesize, text=appleDark, align=center]
  at (Scores.center) {$\phi(\mathbf{Q})\mathbf{X}^{\!\top}$\\$\odot\,\mathbf{M}$};
\node[font=\scriptsize, text=appleGray, above=0.05cm of Scores]
  {$N\!\times\!N$ (decay-weighted causal)};

\node[op, right=0.45cm of Scores] (mul2) {$\times$};

\node[base_node, draw=appleGreen, fill=appleLightGreen,
      minimum width=1.2cm, minimum height=2.4cm,
      align=center, font=\bfseries, right=0.45cm of mul2] (Vmat)
  {$\mathbf{V}$};
\node[font=\scriptsize, text=appleGreen!55!black, above=0.05cm of Vmat]
  {$N\!\times\!d_v$};

\node[font=\bfseries\Large, text=appleDark, right=0.55cm of Vmat] (eq2) {$=$};

\node[base_node, draw=appleDark, fill=white,
      minimum width=1.2cm, minimum height=2.4cm,
      align=center, font=\bfseries, right=0.55cm of eq2] (Out)
  {$\mathbf{O}$};
\node[font=\scriptsize, text=appleDark, above=0.05cm of Out]
  {$N\!\times\!d_v$};

\draw[arrow, color=appleRed!75!black] (Qmat.east) -- (mul1.west);
\draw[arrow, color=appleBlue] (XTmat.west) -- (mul1.east);
\draw[arrow] (XTmat.east) -- (eq1.west);
\draw[arrow] (eq1.east) -- (Scores.west);
\draw[arrow] (Scores.east) -- (mul2.west);
\draw[arrow, color=appleGreen!55!black] (Vmat.west) -- (mul2.east);
\draw[arrow] (Vmat.east) -- (eq2.west);
\draw[arrow] (eq2.east) -- (Out.west);

\node[base_node, fill=white, draw=appleOrange!55,
      minimum width=4.6cm, minimum height=1.4cm, align=left,
      font=\footnotesize, text=appleOrange!40!black,
      below=0.95cm of Scores, inner sep=8pt] (maskBox) {%
\textbf{Decay-mask entries} ($j\le t$):\\
$\displaystyle M_{t,j} = \eta_j \prod_{r=j+1}^{t}\gamma_r,$\\[1pt]
$\displaystyle\boldsymbol{\delta} := (\delta_1,\ldots,\delta_N),\;\delta_t=\prod_{r\le t}\gamma_r$};

\draw[arrow, dashed, color=appleOrange!65!black] (maskBox.north) -- (Scores.south);

\node[eqbox, below=0.45cm of maskBox, align=center, xshift=-1.0cm] (eqB) {%
$\displaystyle
\mathbf{O} \;=\; \operatorname{Diag}(\boldsymbol{\delta})\,\phi(\mathbf{Q})\mathbf{S}_0
\;+\;
\bigl(\phi(\mathbf{Q})\mathbf{X}^{\!\top}\odot \mathbf{M}\bigr)\,\mathbf{V}$%
};

\node[base_node, fill=appleLightRed!55, draw=appleRed!50,
      inner sep=6pt, font=\bfseries\scriptsize, text=appleRed,
      right=0.55cm of eqB, align=center] (complB)
  {$\mathcal{O}(N^2)$ time\\fully parallel over $t$};

\end{scope}

% =========================================================
% CENTER: Equivalence strip B -> C
% =========================================================
\coordinate (yPosBC) at ($(panelB.south)+(-2,-1.75)$);
\node[equiv_badge] (equiv2) at (xRefEquiv |- yPosBC) {$\equiv$};

\node[eqbox, draw=appleDark!40, fill=appleLightGray!50,
      right=0.45cm of equiv2, inner sep=9pt,
      font=\small, text=appleDark, align=center] (derivBox2)
{\textbf{Regroup tokens into }$M$\textbf{ chunks of size }$C$, $N=MC$; chunk-exit decay $g^{(k)}=\prod_{t\in k}\gamma_t$:\\[3pt]
$\displaystyle
\mathbf{O}^{(k)}=\operatorname{Diag}(\boldsymbol{\delta}^{(k)})\,\phi(\mathbf{Q}^{(k)})\mathbf{S}^{(k-1)}
\;+\;
\bigl(\phi(\mathbf{Q}^{(k)})\mathbf{X}^{(k)\,\!\top}\!\odot\mathbf{M}^{(k)}\bigr)\mathbf{V}^{(k)},\;\;
\mathbf{S}^{(k)} = g^{(k)}\mathbf{S}^{(k-1)} + \mathbf{X}^{(k)\,\!\top}\widehat{\mathbf{V}}^{(k)}$%
};

% =========================================================
% PANEL C: Chunk-wise Parallel Form
% =========================================================
\begin{scope}[shift={(0,-22.7)}, local bounding box=panelC,
  chunk_block/.style={
    base_node,
    draw=appleGray!60,
    fill=appleLightGray!45,
    minimum width=2.6cm,
    minimum height=2.05cm,
    rounded corners=10pt
  }
]

\node[state, minimum width=2.4em, minimum height=2.4em] (Sc0) at (0,0) {$\mathbf{S}^{(0)}$};

% Chunk 1
\node[chunk_block, right=0.7cm of Sc0] (Ck1) {};
\node[font=\bfseries\footnotesize, text=appleDark]
  at ([yshift=0.65cm]Ck1.center) {Chunk 1};
\node[base_node, draw=appleRed, fill=appleLightRed,
      minimum width=0.35cm, minimum height=0.65cm]
  at ([xshift=-0.78cm, yshift=-0.1cm]Ck1.center) (Q1m) {};
\node[font=\bfseries\scriptsize, text=appleRed,
      below=0.02cm of Q1m] {$\mathbf{Q}_{\!1}$};
\node[base_node, draw=appleBlue, fill=appleLightBlue,
      minimum width=0.65cm, minimum height=0.35cm]
  at ([xshift=0.00cm, yshift=-0.1cm]Ck1.center) (X1m) {};
\node[font=\bfseries\scriptsize, text=appleBlue,
      below=0.02cm of X1m] {$\mathbf{X}_{\!1}^{\!\top}$};
\node[base_node, draw=appleGreen, fill=appleLightGreen,
      minimum width=0.35cm, minimum height=0.65cm]
  at ([xshift=0.78cm, yshift=-0.1cm]Ck1.center) (V1m) {};
\node[font=\bfseries\scriptsize, text=appleGreen!55!black,
      below=0.02cm of V1m] {$\mathbf{V}_{\!1}$};

\node[state, minimum width=2.4em, minimum height=2.4em, right=0.7cm of Ck1] (Sc1) {$\mathbf{S}^{(1)}$};

% Chunk 2
\node[chunk_block, right=0.7cm of Sc1] (Ck2) {};
\node[font=\bfseries\footnotesize, text=appleDark]
  at ([yshift=0.65cm]Ck2.center) {Chunk 2};
\node[base_node, draw=appleRed, fill=appleLightRed,
      minimum width=0.35cm, minimum height=0.65cm]
  at ([xshift=-0.78cm, yshift=-0.1cm]Ck2.center) (Q2m) {};
\node[font=\bfseries\scriptsize, text=appleRed,
      below=0.02cm of Q2m] {$\mathbf{Q}_{\!2}$};
\node[base_node, draw=appleBlue, fill=appleLightBlue,
      minimum width=0.65cm, minimum height=0.35cm]
  at ([xshift=0.00cm, yshift=-0.1cm]Ck2.center) (X2m) {};
\node[font=\bfseries\scriptsize, text=appleBlue,
      below=0.02cm of X2m] {$\mathbf{X}_{\!2}^{\!\top}$};
\node[base_node, draw=appleGreen, fill=appleLightGreen,
      minimum width=0.35cm, minimum height=0.65cm]
  at ([xshift=0.78cm, yshift=-0.1cm]Ck2.center) (V2m) {};
\node[font=\bfseries\scriptsize, text=appleGreen!55!black,
      below=0.02cm of V2m] {$\mathbf{V}_{\!2}$};

\node[state, minimum width=2.4em, minimum height=2.4em, right=0.7cm of Ck2] (Sc2) {$\mathbf{S}^{(2)}$};

\node[font=\bfseries, right=0.45cm of Sc2] (cdotsC) {$\cdots$};

% Chunk M
\node[chunk_block, right=0.45cm of cdotsC] (CkM) {};
\node[font=\bfseries\footnotesize, text=appleDark]
  at ([yshift=0.65cm]CkM.center) {Chunk $M$};
\node[base_node, draw=appleRed, fill=appleLightRed,
      minimum width=0.35cm, minimum height=0.65cm]
  at ([xshift=-0.78cm, yshift=-0.1cm]CkM.center) (QMm) {};
\node[font=\bfseries\scriptsize, text=appleRed,
      below=0.02cm of QMm] {$\mathbf{Q}_{\!M}$};
\node[base_node, draw=appleBlue, fill=appleLightBlue,
      minimum width=0.65cm, minimum height=0.35cm]
  at ([xshift=0.00cm, yshift=-0.1cm]CkM.center) (XMm) {};
\node[font=\bfseries\scriptsize, text=appleBlue,
      below=0.02cm of XMm] {$\mathbf{X}_{\!M}^{\!\top}$};
\node[base_node, draw=appleGreen, fill=appleLightGreen,
      minimum width=0.35cm, minimum height=0.65cm]
  at ([xshift=0.78cm, yshift=-0.1cm]CkM.center) (VMm) {};
\node[font=\bfseries\scriptsize, text=appleGreen!55!black,
      below=0.02cm of VMm] {$\mathbf{V}_{\!M}$};

\node[state, minimum width=2.4em, minimum height=2.4em, right=0.7cm of CkM] (ScM) {$\mathbf{S}^{(M)}$};

% State propagation arrows with decay labels g^(k)
\draw[arrow, update_arrow] (Sc0) -- node[above, font=\bfseries\tiny, text=appleOrange!40!black]{$g^{(1)}$} (Ck1);
\draw[arrow, update_arrow] (Ck1) -- (Sc1);
\draw[arrow, update_arrow] (Sc1) -- node[above, font=\bfseries\tiny, text=appleOrange!40!black]{$g^{(2)}$} (Ck2);
\draw[arrow, update_arrow] (Ck2) -- (Sc2);
\draw[arrow, update_arrow] (Sc2) -- (cdotsC);
\draw[arrow, update_arrow] (cdotsC) -- node[above, font=\bfseries\tiny, text=appleOrange!40!black]{$g^{(M)}$} (CkM);
\draw[arrow, update_arrow] (CkM) -- (ScM);

% Output drops
\node[base_node, dashed, draw=appleGray, fill=white,
      minimum width=1.2cm, minimum height=1.9em,
      below=1.0cm of Ck1, font=\bfseries] (O1c) {$\mathbf{O}^{(1)}$};
\node[base_node, dashed, draw=appleGray, fill=white,
      minimum width=1.2cm, minimum height=1.9em,
      below=1.0cm of Ck2, font=\bfseries] (O2c) {$\mathbf{O}^{(2)}$};
\node[base_node, dashed, draw=appleGray, fill=white,
      minimum width=1.2cm, minimum height=1.9em,
      below=1.0cm of CkM, font=\bfseries] (OMc) {$\mathbf{O}^{(M)}$};

\draw[arrow, color=appleRed!85!black] (Ck1.south) -- (O1c.north);
\draw[arrow, color=appleRed!85!black] (Ck2.south) -- (O2c.north);
\draw[arrow, color=appleRed!85!black] (CkM.south) -- (OMc.north);

\node[eqbox, below=2.95cm of Sc2, align=center, xshift=-3.0cm] (eqC) {%
$\displaystyle
\begin{aligned}
\text{(P1, parallel)}\quad
&\mathbf{M}^{(k)}_{i,s}=\mathbb{I}[s\le i]\,\eta_{a_k+s-1}\!\!\prod_{r=s+1}^{i}\!\gamma_{a_k+r-1},\quad
\delta^{(k)}_i=\!\prod_{r=1}^{i}\!\gamma_{a_k+r-1},\quad
g^{(k)}=\delta^{(k)}_C\\[1pt]
\text{(P2, sequential)}\quad
&\mathbf{S}^{(k)} = g^{(k)}\mathbf{S}^{(k-1)} + \mathbf{X}^{(k)\,\!\top}\widehat{\mathbf{V}}^{(k)},\quad
\widehat{\mathbf{V}}^{(k)}_{i,:}\!=\!\eta_{a_k+i-1}\!\!\prod_{r=i+1}^{C}\!\gamma_{a_k+r-1}\,\mathbf{v}_{a_k+i-1}\\[1pt]
\text{(P3, parallel)}\quad
&\mathbf{O}^{(k)} = \underbrace{\operatorname{Diag}(\boldsymbol{\delta}^{(k)})\,\phi(\mathbf{Q}^{(k)})\mathbf{S}^{(k-1)}}_{\text{cross-chunk (decayed history)}}
+\underbrace{\bigl(\phi(\mathbf{Q}^{(k)})\mathbf{X}^{(k)\,\!\top}\!\odot\mathbf{M}^{(k)}\bigr)\mathbf{V}^{(k)}}_{\text{intra-chunk (parallel)}}
\end{aligned}$%
};

\node[base_node, fill=appleLightGreen!55, draw=appleGreen!55!black,
      inner sep=6pt, font=\bfseries\scriptsize, text=appleGreen!45!black,
      right=0.55cm of eqC, align=center] (complC)
  {$\mathcal{O}(NC)$ time\\$M=N/C$ sequential steps};

\end{scope}

% =========================================================
% Backgrounds & titles
% =========================================================
\path
  let \p1=(panelA.west), \p2=(panelB.west), \p5=(panelC.west),
      \p3=(panelA.east), \p4=(panelB.east), \p6=(panelC.east)
  in
    coordinate (commonW) at ({min(\x1,min(\x2,\x5))-0.1},0)
    coordinate (commonE) at ({max(\x3,max(\x4,\x6))+0.1},0);

\node[inner sep=0pt, minimum size=0pt] (anchorA_L) at (commonW |- panelA.center) {};
\node[inner sep=0pt, minimum size=0pt] (anchorA_R) at (commonE |- panelA.center) {};
\node[inner sep=0pt, minimum size=0pt] (anchorB_L) at (commonW |- panelB.center) {};
\node[inner sep=0pt, minimum size=0pt] (anchorB_R) at (commonE |- panelB.center) {};
\node[inner sep=0pt, minimum size=0pt] (anchorC_L) at (commonW |- panelC.center) {};
\node[inner sep=0pt, minimum size=0pt] (anchorC_R) at (commonE |- panelC.center) {};

\begin{scope}[on background layer]
  \node[draw=appleGray!30, left color=appleLightGray!30, right color=appleLightGray,
        rounded corners=16pt,
        fit=(panelA)(eqA)(complA)(anchorA_L)(anchorA_R),
        inner sep=18pt, inner ysep=26pt] (boxA) {};
  \node[draw=appleGray!30, left color=appleLightGray!30, right color=appleLightGray,
        rounded corners=16pt,
        fit=(panelB)(eqB)(complB)(maskBox)(anchorB_L)(anchorB_R),
        inner sep=18pt, inner ysep=26pt] (boxB) {};
  \node[draw=appleGray!30, left color=appleLightGray!30, right color=appleLightGray,
        rounded corners=16pt,
        fit=(panelC)(eqC)(complC)(anchorC_L)(anchorC_R),
        inner sep=18pt, inner ysep=26pt] (boxC) {};
\end{scope}

\node[section_label] at ([yshift=-5pt, xshift=14pt]boxA.north west |- boxA.north)
  {A. Recurrent Form (causal scan with scalar decay)};
\node[section_label] at ([yshift=-5pt, xshift=14pt]boxB.north west |- boxB.north)
  {B. Parallel Form (decay-mask attention)};
\node[section_label] at ([yshift=-5pt, xshift=14pt]boxC.north west |- boxC.north)
  {C. Chunk-wise Parallel Form};

\end{tikzpicture}
}
\caption{\textbf{Three equivalent views of \Falcon-1A.}
(A) The recurrent form maintains a fixed-size matrix state $\mathbf{S}_t\in\mathbb{R}^{d_x\times d_v}$. Compared with linear attention, each carry is multiplied by a scalar decay $\gamma_t=1-\eta_t\lambda_t$, and the write feature is the prefix feature $\mathbf{x}_t:=\phi(\mathbf{k}_{t-1})$ rather than the same-step feature $\phi(\mathbf{k}_t)$.
(B) Unrolling the scalar recurrence yields a masked-attention form with the same shifted features but a \emph{decay-weighted} causal mask: $M_{t,j}=\eta_j\prod_{r=j+1}^{t}\gamma_r$ for $j\le t$ and $0$ otherwise.
(C) The chunk-wise form splits the sequence into $M$ chunks of size $C$.}
\label{fig:falcon1a_equivalence}
\end{figure}

\paragraph{Gradient and update.}
The gradient of Eq.~\eqref{eq:ip_loss} is
\[
\nabla_{\Sbb}\ell_t^{\rm ip}(\Sbb)= -\mathbf{x}_t\mathbf{y}_t^\top + \lambda_t \Sbb.
\]
A scalar gradient step gives the Linear-Attention/Mamba-2 style update, which
we denote \Falcon-1A:
\begin{equation}
\label{eq:Falcon1a_update}
\Sbb_t
= (1-\eta_t\lambda_t)\Sbb_{t-1} + \eta_t\,\mathbf{x}_t\mathbf{y}_t^\top,
\end{equation}
where $\mathbf{x}_t=\phi(\mathbf{k}_{t-1})$ and $\mathbf{y}_t=\vb_t$ under next-latent alignment. Setting $\lambda_t=0$ recovers the usual additive write $\Sbb_t=\Sbb_{t-1}+\eta_t\mathbf{x}_t\mathbf{y}_t^\top$.

\paragraph{\Falcon-2A: per-column inner-product write.}
Because Eq.~\eqref{eq:ip_loss} decomposes over value coordinates, each column
can use its own energy-normalized learning rate. Define
\[
\boldsymbol{\eta}_t:=(\eta_{1,t},\ldots,\eta_{d_v,t})^\top.
\]
The per-column inner-product update is
\begin{equation}
\label{eq:Falcon2a_update}
\Sbb_t
=
\Sbb_{t-1}\!\left(\Ib_{d_v}-\lambda_t\operatorname{Diag}(\boldsymbol{\eta}_t)\right)
+\mathbf{x}_t(\boldsymbol{\eta}_t\odot\mathbf{y}_t)^\top .
\end{equation}
Equivalently, the $j$-th column evolves as
\[
\mathbf{s}_{t,j}
=(1-\eta_{j,t}\lambda_t)\mathbf{s}_{t-1,j}
+\eta_{j,t}y_{t,j}\mathbf{x}_t .
\]
Thus \Falcon-2A is not the scalar inner-product rule; it is the per-column
inner-product analogue of \Falcon-2.

\begin{figure}[ht!]
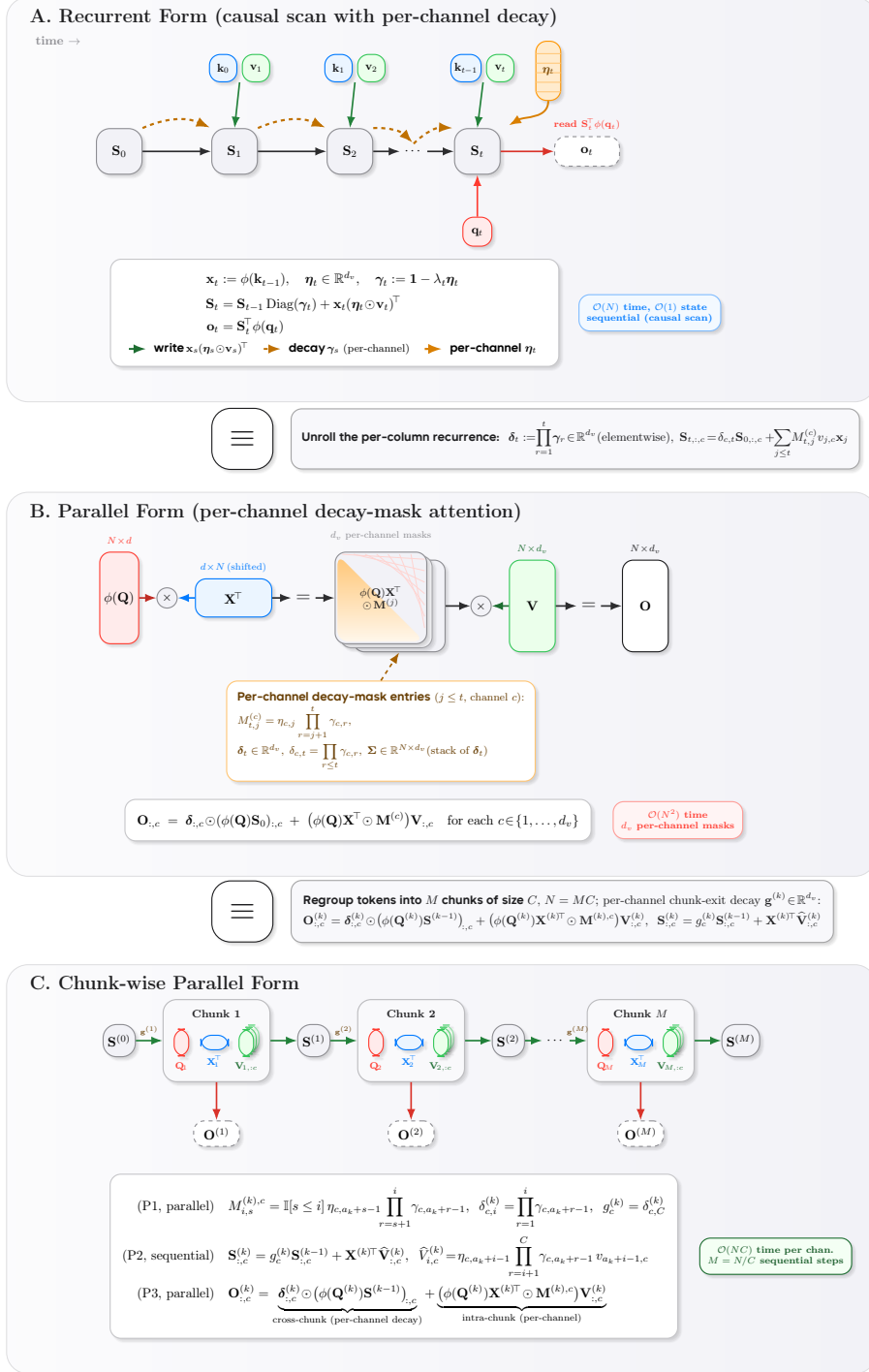

\centering
\resizebox{0.72\linewidth}{!}{%

}
\caption{\textbf{Three equivalent views of \Falcon-2A.}
(A) The recurrent form maintains a fixed-size matrix state $\mathbf{S}_t\in\mathbb{R}^{d_x\times d_v}$.
(B) Unrolling the per-column recurrence yields a masked-attention form with the same shifted features but $d_v$ \emph{channel-specific} decay-weighted causal masks: $M^{(c)}_{t,j}=\eta_{c,j}\prod_{r=j+1}^{t}\gamma_{c,r}$ for $j\le t$ and $0$ otherwise.
(C) The chunk-wise form splits the sequence into $M$ chunks of size $C$.}
\label{fig:falcon2a_equivalence}
\end{figure}

\paragraph{\Falcon-1A/\Falcon-2A step sizes.}
Unlike regression, the inner-product objective is $\lambda_t$-smooth
independently of the write-feature energy. Accordingly, the objective-matched
denominator would depend only on $\lambda_t$; in the inner-product
implementations studied here, we instead retain an energy-normalized write gain
to control the magnitude of the additive write. Let $E_t:=\|\mathbf{x}_t\|_2^2$,
and if the scale-coupled parameterization is active set
$\lambda_t:=\bar\lambda_t E_t$ before applying the update. The scalar
\Falcon-1A step size is

\begin{equation}
\label{eq:ip_stepsize}
\eta_t
= \frac{\beta_t}{E_t+\lambda_t+\varepsilon},
\qquad \beta_t\in(0,2),\ \varepsilon\ge 0.
\end{equation}
The per-column \Falcon-2A step sizes are
\begin{equation}
\label{eq:ip_stepsize_vector}
\eta_{j,t}
= \frac{\beta_{j,t}}{E_t+\lambda_t+\varepsilon},
\qquad \beta_{j,t}\in(0,2),\ \varepsilon\ge 0.
\end{equation}
As in Eq.~\eqref{eq:nlms-stepsize}, we set $\eta_t:=0$ (or
$\eta_{j,t}:=0$ for all $j$) when the denominator vanishes, and also at the
boundary sentinel $t=1$ when $\mathbf{x}_1=\mathbf{0}$ is imposed.
When $\lambda_t>0$, these choices satisfy
$\eta_t<2/\lambda_t$ and $\eta_{j,t}<2/\lambda_t$ for any admissible
$\beta_t,\beta_{j,t}$. Since $\ell_t^{\rm ip}$ is $\lambda_t$-smooth
(its Hessian is $\lambda_t\Ib$), Lemma~\ref{lem:per-step-descent} yields
per-step descent for the unclamped scalar update, and the same argument applies
column-wise to Eq.~\eqref{eq:Falcon2a_update}. If the later positive-decay clamp is activated for log-space unrolling, the implemented shrinkage should be 
interpreted as using the effective ridge coefficient $\tilde\lambda_t:=\alpha_t/\eta_t$ in the scalar case, or
$\tilde\lambda_{j,t}:=\alpha_{j,t}/\eta_{j,t}$ in the per-column case, whenever the corresponding step size is positive. The $E_t$ term is not required by
curvature, but stabilizes the write magnitude and yields a sensible $\lambda_t\to 0$ limit.

\paragraph{Decay positivity.}
Some parallel/unrolled forms (Section~\ref{sec:mb_ip_parallel}) use $\gamma_t:=1-\eta_t\lambda_t$ in log space and therefore require $\gamma_t>0$.
In implementations, compute $\alpha_t:=\eta_t\lambda_t$ and, if necessary, clamp $\alpha_t\leftarrow \min(\alpha_t,1-\varepsilon_\gamma)$ before computing $\log\gamma_t=\operatorname{log1p}(-\alpha_t)$ (fp32). The clamp is inactive whenever $\eta_t\lambda_t<1-\varepsilon_\gamma$; in that regime, the dynamics match $\gamma_t=1-\eta_t\lambda_t$ exactly.
As above, the descent statement pertains to the unclamped recurrence. When the clamp activates, the implemented recurrence should be viewed as a numerically safe surrogate whose shrinkage path uses the effective ridge coefficient
\[
\tilde\lambda_t := \alpha_t/\eta_t
\]
whenever $\eta_t>0$ (and $\tilde\lambda_t:=0$ when $\eta_t=0$), while the additive write gain remains $\eta_t$.

\subsection{Mini-batch Update Rule for Regression (Falcon-3)}
\label{sec:methods_sliding_regression}

To better capture local dependencies and reduce noise accumulation, we introduce a \emph{sliding-state} mechanism: instead of updating the state from only the instantaneous residual, we take a single mini-batch gradient step on a finite history window of nominal size $B$.

Falcon-3 can be viewed as a sliding-window specialization of the internal-objective view exemplified by ATLAS~\citep{behrouz2025atlas}, but here instantiated with a linear matrix memory, a squared-error objective, and strict next-latent alignment. Exact continuation across segment boundaries additionally requires a fixed-width tail of the last $B{-}1$ causal pairs; Appendix~\ref{app:sliding_window_remarks} records the details.

\begin{algorithm}[t]
\caption{Falcon-3 (Recurrent Form)}
\label{alg:Falcon3_reference}
\small
\begin{algorithmic}[1]
\Require Query features $\{\mathbf{q}_t\}_{t=1}^{T}$ with $\mathbf{q}_t=\phi(\qb_t)$, shifted write-features $\{\mathbf{x}_t\}_{t=1}^{T}$ with $\mathbf{x}_1=\mathbf{0}$ and $\mathbf{x}_t=\phi(\mathbf{k}_{t-1})$ for $t\ge 2$, values $\{\vb_t\}_{t=1}^{T}$, window size $B$, gains $\beta_t\in(0,2)$, ridge $\lambda_t\ge 0$, stabilizer $\varepsilon>0$, init $\Sbb_0=\mathbf{0}$.
\Ensure Outputs $\mathbf{O}\in\mathbb{R}^{T\times d_v}$ with row $t$ equal to $\ob_t^\top=\mathbf{q}_t^\top\Sbb_t$ (read-after-write), final matrix state $\Sbb_T$, and final FIFO buffer $\mathcal{W}_T$; exact continuation across a later segment requires both $\Sbb_T$ and $\mathcal{W}_T$.
\State Initialize empty FIFO buffer $\mathcal{W}\gets [\,]$ and output buffer $\mathbf{O}\gets \mathbf{0}$.
\For{$t=1,\ldots,T$}
\If{$t=1$}
\State $\eta_t\gets 0$, \quad $\Sbb_t\gets \Sbb_{t-1}$, \quad $\ob_t\gets \Sbb_t^\top\mathbf{q}_t$
\State Write $\ob_t^\top$ into row $t$ of $\mathbf{O}$ and \textbf{continue}
\EndIf
\State Push $(\mathbf{x}_t,\vb_t)$ into $\mathcal{W}$; if $|\mathcal{W}|>B$, pop the oldest pair.
\State $B_t\gets |\mathcal{W}|$, \quad $\mathbf{U}_t\gets \mathbf{0}_{d_x\times d_v}$
\ForAll{$(\mathbf{x},\mathbf{v})\in\mathcal{W}$}
\State $\mathbf{U}_t\gets \mathbf{U}_t + \mathbf{x}\big(\mathbf{v}-\Sbb_{t-1}^\top\mathbf{x}\big)^\top$
\EndFor
\State Let $\mathbf{X}_t\in\mathbb{R}^{d_x\times B_t}$ stack the write-features in $\mathcal{W}$.
\State $\mu_t^{(B)}\gets \lambda_{\max}(\mathbf{X}_t^\top\mathbf{X}_t)/B_t$ \Comment{equiv.\ $\lambda_{\max}(\bar{\Cb}_t^{(B)})$}
\State $\eta_t\gets \beta_t/(\mu_t^{(B)}+\lambda_t+\varepsilon)$
\State $\Sbb_t\gets (1-\eta_t\lambda_t)\Sbb_{t-1} + (\eta_t/B_t)\,\mathbf{U}_t$
\State $\ob_t\gets \Sbb_t^\top\mathbf{q}_t$
\State Write $\ob_t^\top$ into row $t$ of $\mathbf{O}$
\EndFor
\State \Return $(\mathbf{O},\Sbb_T,\mathcal{W})$
\end{algorithmic}
\end{algorithm}

% =========================================================
% FIGURE: Three equivalent views of Falcon-3
%   A. Recurrent Form (sliding mini-batch causal scan)
%   B. Parallel Form  (ParallelFlow tensorInv on time x rank)
%   C. Chunk-wise Parallel Form (three-phase ParallelFlow)
%
% Requires the same preamble as Figure: Linear Attention Equivalence
% and Figure: Three equivalent views of Falcon-1
%   - colors: appleBlue, appleLightBlue, appleRed, appleLightRed,
%     appleGreen, appleLightGreen, appleGray, appleLightGray,
%     appleDarkGray, appleDark, applePurple, appleLightPurple,
%     appleTeal, appleLightTeal
%   - tikzlibraries: positioning, calc, shapes.geometric,
%     arrows.meta, backgrounds, fit, matrix, shadows.blur
% =========================================================
\begin{figure}[ht!]
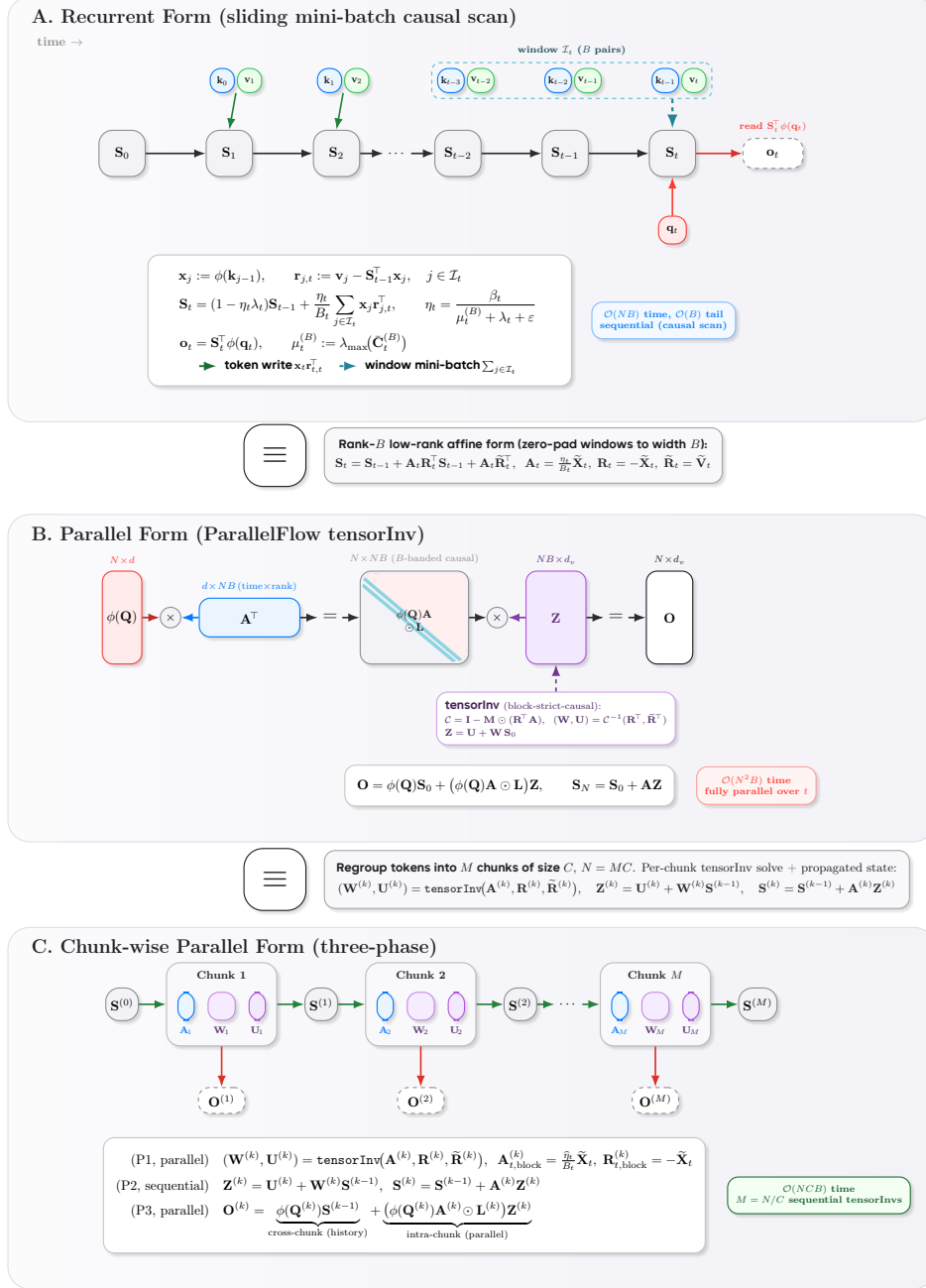

\centering
\providecolor{applePurple}{RGB}{175, 82, 222}
\providecolor{appleLightPurple}{RGB}{243, 233, 255}
\providecolor{appleTeal}{RGB}{48, 176, 199}
\providecolor{appleLightTeal}{RGB}{213, 240, 245}
\resizebox{0.75\linewidth}{!}{%

}
\caption{\textbf{Three equivalent views of \Falcon-3.}
(A) The recurrent form maintains a fixed-size matrix state $\mathbf{S}_t\in\mathbb{R}^{d_x\times d_v}$. Each step writes a mini-batch gradient on the window $\mathcal{I}_t$ of the last $B$ causal pairs $\{(\phi(\mathbf{k}_{j-1}),\mathbf{v}_j)\}_{j\in\mathcal{I}_t}$, with all residuals $\mathbf{r}_{j,t}=\mathbf{v}_j-\mathbf{S}_{t-1}^{\!\top}\mathbf{x}_j$ evaluated at the pre-update state.
(B) After the scalar positive-decay renormalization (the identity when $\lambda_t=0$), zero-padding each active window to width $B$ recasts the update as a rank-$B$ affine recurrence $\mathbf{S}_t=\mathbf{S}_{t-1}+\mathbf{A}_t\mathbf{R}_t^{\!\top}\mathbf{S}_{t-1}+\mathbf{A}_t\widetilde{\mathbf{R}}_t^{\!\top}$.
(C) Splitting the sequence into $M$ chunks of size $C$ yields the exact three-phase ParallelFlow algorithm.}
\label{fig:falcon3_equivalence}
\end{figure}

\paragraph{Sequence-parallel training.}
After zero-padding each active window to width $B$, Falcon-3 becomes a fixed-rank-$B$ low-rank recurrence. Algorithm~\ref{alg:Falcon3_reference} gives the reference sequential update, and Algorithm~\ref{alg:Falcon3_forward} gives the chunk-parallel ParallelFlow implementation based on the positive-decay reduction in Eq.~\eqref{eq:deltanet3_residual_renorm}. Appendix~\ref{app:deltanet3_parallel} records the driver construction and mask conventions. The inner-product counterpart (Falcon-3A) admits an explicit masked linear-attention form (Section~\ref{sec:mb_ip_parallel}), enabling fully vectorized training over the sequence dimension.

For $t\ge 2$, let the active window indices be $\mathcal{I}_t=\{j \mid \max(2,t-B+1)\le j\le t\}$ and denote the realized window size by $B_t:=|\mathcal{I}_t|\le B$ (so $B_t\ge 1$). Define the write feature $\mathbf{x}_j:=\phi(\mathbf{k}_{j-1})$ for $j\ge 2$ (so $\mathbf{x}_j=\mathbf{k}_{j-1}$ when $\phi$ is the identity), and impose the boundary convention $\mathbf{x}_1:=\mathbf{0}$. We set $\eta_1:=0$, so the $t=1$ write is a no-op; all windowed objectives/updates below are defined for $t\ge 2$.
To make the update magnitude (and hence the effective decay) invariant to the nominal window size $B$, we optimize the window-average squared loss:
\begin{align}
\ell_t^{\rm reg,(B)}(\Sbb)
:=
\frac{1}{2B_t} \sum_{j \in \mathcal{I}_t} \|\Sbb^\top \mathbf{x}_{j} - \vb_j\|_2^2
+ \frac{\lambda_t}{2}\|\Sbb\|_F^2,\qquad (t\ge 2).
\end{align}

\noindent\textbf{Sufficient Statistics.}
We define the sliding covariance $\Cb_t^{(B)}$ and cross-covariance $\Nb_t^{(B)}$ matrices:
\begin{align}
\Cb_t^{(B)} \triangleq \sum_{j \in \mathcal{I}_t} \mathbf{x}_{j}\mathbf{x}_{j}^\top, \qquad
\Nb_t^{(B)} \triangleq \sum_{j \in \mathcal{I}_t} \mathbf{x}_{j}\vb_j^\top.
\end{align}
Define the window-averaged statistics
\[
\bar{\Cb}_t^{(B)} := \frac{1}{B_t}\Cb_t^{(B)}, \qquad
\bar{\Nb}_t^{(B)} := \frac{1}{B_t}\Nb_t^{(B)}.
\]
Then the gradient evaluated at the pre-update state is
\[
\nabla_{\Sbb} \ell_t^{\rm reg,(B)}(\Sbb_{t-1})
= \bar{\Cb}_t^{(B)}\Sbb_{t-1} - \bar{\Nb}_t^{(B)} + \lambda_t \Sbb_{t-1}.
\]

\noindent\textbf{Update Rule.}
We apply a block-normalized gradient step:
\begin{align}
\Sbb_t = \Sbb_{t-1} - \eta_t \nabla_{\Sbb}\ell_t^{\rm reg,(B)}(\Sbb_{t-1}).
\end{align}
Substituting the gradient yields the affine update
\begin{equation}
\label{eq:deltanet3_affine}
\Sbb_t = \big(\Ib_{d_x} - \eta_t(\bar{\Cb}_t^{(B)} + \lambda_t \Ib_{d_x})\big)\Sbb_{t-1} + \eta_t \bar{\Nb}_t^{(B)}.
\end{equation}
Equivalently, collecting residuals at the pre-update state yields
\begin{equation}
\label{eq:deltanet3_residual}
\Sbb_t
=
(1-\eta_t\lambda_t)\Sbb_{t-1}
\;+\;
\frac{\eta_t}{B_t}\sum_{j\in\mathcal{I}_t}\mathbf{x}_j\big(\vb_j-\Sbb_{t-1}^\top\mathbf{x}_j\big)^\top.
\end{equation}
This is the direct mini-batch analogue of Eq.~\eqref{eq:ogd}: all window residuals are evaluated at the pre-update state $\Sbb_{t-1}$, and the update averages their rank-one gradients.

For $t\ge 2$, let $\mathbf{X}_t\in\mathbb{R}^{d_x\times B_t}$ stack the active-window write-features, so that $\bar{\Cb}_t^{(B)}=\mathbf{X}_t\mathbf{X}_t^\top/B_t$. We normalize by the exact local smoothness scale of the windowed ridge objective rather than by its trace upper bound:
\[
\mu_t^{(B)}:=\lambda_{\max}(\bar{\Cb}_t^{(B)})=\frac{\|\mathbf{X}_t\|_2^2}{B_t}=\frac{\lambda_{\max}(\mathbf{X}_t^\top\mathbf{X}_t)}{B_t},
\qquad
\eta_t = \frac{\beta_t}{\mu_t^{(B)} + \lambda_t + \varepsilon},
\qquad \beta_t \in (0,2),\ \varepsilon>0.
\]
Then $L_t^{(B)}=\mu_t^{(B)}+\lambda_t$, so whenever $L_t^{(B)}>0$ this normalization ensures $\eta_t \in (0,2/L_t^{(B)})$ for any $\beta_t\in(0,2)$, and Lemma~\ref{lem:per-step-descent} yields per-step descent for Eq.~\eqref{eq:deltanet3_residual} before any positive-decay clamp. If $L_t^{(B)}=0$ (equivalently, $\bar{\Cb}_t^{(B)}=0$ and $\lambda_t=0$), then $\bar{\Nb}_t^{(B)}=0$ as well and the update is a no-op.
Crucially, because we optimize the window average, $\bar{\Nb}_t^{(B)}$ is an average and $\mu_t^{(B)}$ is the spectral norm of an average covariance, so neither quantity grows linearly with the nominal window size $B$. If the write feature itself is RMS-normalized, then
\[
\mu_t^{(B)}\le \bar E_t^{(B)}:=\operatorname{tr}(\bar{\Cb}_t^{(B)})\approx d_x,
\]
so the denominator remains $O(d_x)$ rather than $O(Bd_x)$; for a generic kernel map $\phi$, the correct statement is simply that $\mu_t^{(B)}$ tracks the realized windowed smoothness scale. Consequently, neither the injection $\eta_t\bar{\Nb}_t^{(B)}$ nor the decay fraction $\alpha_t:=\eta_t\lambda_t$ is systematically amplified by increasing $B$. If the scale-coupled ridge parameterization of Section~\ref{sec:scaled_fast_weight} is enabled, replace $\lambda_t$ throughout this subsection by
$\lambda_t^{\mathrm{eff}}:=\bar\lambda_t\,\mu_t^{(B)}$. In the current implementation, this smoothness statistic can be treated as a statistics-only multiplier when constructing $\lambda_t^{\mathrm{eff}}$ (detached / stop-gradient through the multiplier), while the step-size denominator still uses the live $\mu_t^{(B)}$.
Importantly, one need not materialize the $d_x\times d_x$ matrix $\Cb_t^{(B)}$ to evaluate either the gradient or the step size: if we stack the window write-features into $\mathbf{X}_t \in \mathbb{R}^{d_x\times B_t}$, then
\begin{equation}
\bar{\Cb}_t^{(B)}\Sbb_{t-1} = \frac{1}{B_t}\,\mathbf{X}_t \big(\mathbf{X}_t^\top \Sbb_{t-1}\big),
\qquad
\mu_t^{(B)}=\frac{\lambda_{\max}(\mathbf{X}_t^\top\mathbf{X}_t)}{B_t}.
\end{equation}
Since $B_t\le B$ is small, $\mu_t^{(B)}$ can be computed exactly from the $B_t\times B_t$ Gram matrix or approximated with a few power iterations. For the implemented positive-decay recurrence, define
\[
\begin{aligned}
\alpha_t^{\rm raw}&:=\eta_t\lambda_t,
&\alpha_t&:=\min(\alpha_t^{\rm raw},\,1-\varepsilon_\gamma),
&\gamma_t&:=1-\alpha_t,\\
c_0&:=1,
&c_t&:=\prod_{r=1}^{t}\gamma_r,
&\tilde{\Sbb}_t&:=\Sbb_t/c_t.
\end{aligned}
\]
When $\alpha_t=\alpha_t^{\rm raw}$, this is exactly the ridge-gradient
recurrence in Eq.~\eqref{eq:deltanet3_residual}. If the clamp activates, the
implemented shrinkage path should instead be interpreted as using the effective
coefficient
\[
\tilde\lambda_t:=
\begin{cases}
\alpha_t/\eta_t, & \eta_t>0,\\
0, & \eta_t=0,
\end{cases}
\]
while keeping the same residual injection gain $\eta_t$. Thus the descent claim
above applies to the unclamped update; the clamped update is a positive-decay
surrogate. With this convention, the implemented recurrence is equivalent to
\begin{equation}
\label{eq:deltanet3_residual_renorm}
\tilde{\Sbb}_t
=
\tilde{\Sbb}_{t-1}
+
\frac{\hat\eta_t}{B_t}\sum_{j\in\mathcal{I}_t}
\mathbf{x}_j\Big(\frac{\vb_j}{c_{t-1}}-\tilde{\Sbb}_{t-1}^\top\mathbf{x}_j\Big)^\top,
\qquad
\hat\eta_t:=\eta_t/\gamma_t.
\end{equation}
Algorithm~\ref{alg:Falcon3_forward} realizes Eq.~\eqref{eq:deltanet3_residual_renorm} chunk-locally via log-prefix decays; Appendix~\ref{app:sliding_window_remarks} records the corresponding continuation and chunk-boundary details.

\begin{breakablealgorithm}
\caption{Falcon-3 via ParallelFlow (Chunk-parallel Forward)}
\label{alg:Falcon3_forward}
\small
\begin{algorithmic}[1]
\Require Queries $\{\qb_t\}_{t=1}^{T}$, keys $\{\mathbf{k}_t\}_{t=1}^{T}$, values $\{\vb_t\}_{t=1}^{T}$, feature map $\phi$ (default identity), window size $B$, gains $\beta_t\in(0,2)$, ridge $\lambda_t\ge 0$, stabilizer $\varepsilon>0$, decay clamp $\varepsilon_\gamma>0$, chunk size $C$ (assume $C\mid T$), init $\Sbb_0$.
\Ensure Outputs $\mathbf{O}\in\mathbb{R}^{T\times d_v}$ with row $t$ equal to $\ob_t^\top=\phi(\qb_t)^\top\Sbb_t$ (read-after-write), final matrix state $\Sbb_T$, and boundary tail $\mathcal{T}_T:=\{(\mathbf{x}_j,\vb_j)\}_{j=\max(2,T-B+2)}^{T}$ for exact continuation across a later segment.

\State Write/query features: set $\mathbf{x}_1\gets \mathbf{0}$, for $t\ge 2$ set $\mathbf{x}_t\gets \phi(\mathbf{k}_{t-1})$, and for all $t$ set $\mathbf{q}_t\gets \phi(\qb_t)$.
\State Window indices $\mathcal{I}_t\gets\{j\mid \max(2,t-B+1)\le j\le t\}$,\; window sizes $B_t\gets|\mathcal{I}_t|$,\; and $B_t^+\gets \max(1,B_t)$.
\State Window smoothness statistics: $\mu_1^{(B)}\gets 0$ and for $t\ge 2$ set $\mu_t^{(B)}\gets \lambda_{\max}(\mathbf{G}_t)/B_t$, where $\mathbf{G}_t:=\mathbf{X}_t^\top\mathbf{X}_t$ and $\mathbf{X}_t:=[\mathbf{x}_j]_{j\in\mathcal{I}_t}\in\mathbb{R}^{d_x\times B_t}$.
\State Step sizes: set $\eta_1\gets 0$ and for $t\ge 2$ set $\eta_t\gets \beta_t/(\mu_t^{(B)}+\lambda_t+\varepsilon)$.
\State Clamp decays: $\alpha_t\gets \min(\eta_t\lambda_t,\,1-\varepsilon_\gamma)$,\; $\log\gamma_t\gets \operatorname{log1p}(-\alpha_t)$,\; $\gamma_t\gets \exp(\log\gamma_t)$,\; and $\hat\eta_t\gets \eta_t/\gamma_t$.
\State Partition into $M=T/C$ chunks $[a_k,b_k]=[(k{-}1)C{+}1,kC]$.

\State \Comment{\textbf{Phase 1: Intra-chunk ParallelFlow solves (parallel over chunks)}}
\For{$k=1$ \textbf{to} $M$ \textbf{in parallel}}
  \State $a\gets a_k$, $b\gets b_k$, $p\gets \max(2,a-B+1)$,\; $J\gets\{p,\ldots,b\}$.
    \State Slice $\mathbf{Q}^{(k)}\gets (\mathbf{q}_t^\top)_{t=a}^{b}\in\mathbb{R}^{C\times d_x}$ and gather the write pairs $\{(\mathbf{x}_j,\vb_j)\}_{j\in J}$.
    \State Build zero-padded window stacks $\mathbf{X}_{\rm win}^{(k)}\in\mathbb{R}^{d_x\times (CB)}$ and $\mathbf{V}_{\rm win}^{(k)}\in\mathbb{R}^{d_v\times (CB)}$:
    the $i$-th block of $B$ columns contains the pairs from $\mathcal{I}_{a+i-1}$ (in increasing $j$), padded with $\mathbf{0}$ to width $B$.
    \State Local log-prefix decays (length-$C$ scan): $\delta^{(k)}_0\gets 1$, $u_0\gets 0$, and for $i=1{:}C$ set
    $u_i\gets u_{i-1}+\log\gamma_{a+i-1}$,\; $\delta^{(k)}_i\gets \exp(u_i)$.
    \State Rescale values blockwise: divide every column in block $i$ of $\mathbf{V}_{\rm win}^{(k)}$ by the scalar $\delta^{(k)}_{i-1}$, yielding $\widehat{\mathbf{V}}_{\rm win}^{(k)}$.
    \State \Comment{Flatten time$\times$rank. Same-time rank components are uncoupled; only earlier time-blocks interact inside \texttt{tensorInv}.}
    \State Driver matrices:
    scale block $i$ of $\mathbf{X}_{\rm win}^{(k)}$ by $s_i:=\hat\eta_{a+i-1}/B^+_{a+i-1}$ to obtain $\mathbf{A}^{(k)}\in\mathbb{R}^{d_x\times(CB)}$;
    set $\mathbf{R}^{(k)}\gets -\mathbf{X}_{\rm win}^{(k)}$ and $\widetilde{\mathbf{R}}^{(k)}\gets \widehat{\mathbf{V}}_{\rm win}^{(k)}$.
    \State $(\mathbf{W}^{(k)},\mathbf{U}^{(k)})\gets \texttt{tensorInv}\big(\mathbf{A}^{(k)},\mathbf{R}^{(k)},\widetilde{\mathbf{R}}^{(k)}\big)$ \Comment{block-strict-causal solve on $(\text{time},\text{rank})$ pairs}
    \State Optional scan form: $\mathbf{M}^{(k)}\gets \delta^{(k)}_{C}\big(\Ib_{d_x}+\mathbf{A}^{(k)}\mathbf{W}^{(k)}\big)$ and $\mathbf{b}^{(k)}\gets \delta^{(k)}_{C}\mathbf{A}^{(k)}\mathbf{U}^{(k)}$.
    \State Cache $\big(\mathbf{Q}^{(k)},\mathbf{A}^{(k)},\mathbf{W}^{(k)},\mathbf{U}^{(k)},\boldsymbol{\delta}^{(k)}\big)$ where $\boldsymbol{\delta}^{(k)}:=(\delta^{(k)}_1,\ldots,\delta^{(k)}_C)$.
\EndFor

\State \Comment{\textbf{Phase 2: Inter-chunk boundary propagation (sequential; equivalently an associative scan over }$(\mathbf{M}^{(k)},\mathbf{b}^{(k)})$\textbf{)}}
\State $\Sbb_{\rm in}^{(1)}\gets \Sbb_0$.
\For{$k=1$ \textbf{to} $M$}
  \State Store $\Sbb_{\rm in}^{(k)}$.
    \State $\mathbf{Z}^{(k)}\gets \mathbf{U}^{(k)}+\mathbf{W}^{(k)}\Sbb_{\rm in}^{(k)}$.
    \State $\widehat{\Sbb}_{\rm out}^{(k)}\gets \Sbb_{\rm in}^{(k)}+\mathbf{A}^{(k)}\mathbf{Z}^{(k)}$. \Comment{chunk exit in the local renormalized domain}
    \State $\Sbb_{\rm in}^{(k+1)}\gets \delta^{(k)}_{C}\,\widehat{\Sbb}_{\rm out}^{(k)}$. \Comment{restore original scale at the chunk boundary; equiv. $\Sbb_{\rm in}^{(k+1)}=\mathbf{M}^{(k)}\Sbb_{\rm in}^{(k)}+\mathbf{b}^{(k)}$}
\EndFor
\State $\Sbb_T\gets \Sbb_{\rm in}^{(M+1)}$.

\State \Comment{\textbf{Phase 3: Materialize token outputs (parallel over chunks)}}
\For{$k=1$ \textbf{to} $M$ \textbf{in parallel}}
  \State $\mathbf{Z}^{(k)}\gets \mathbf{U}^{(k)}+\mathbf{W}^{(k)}\Sbb_{\rm in}^{(k)}$.
  \State $\mathbf{H}^{(k)}\gets \mathbf{Q}^{(k)}\Sbb_{\rm in}^{(k)}$ \Comment{$C\times d_v$}
  \State $\mathbf{P}^{(k)}\gets \mathbf{Q}^{(k)}\mathbf{A}^{(k)}$ \Comment{$C\times(CB)$}
  \State Define the block-causal mask $L^{(k)}\in\{0,1\}^{C\times(CB)}$ by
  $L^{(k)}_{i,(m-1)B+1:mB}=\mathbb{I}[m\le i]\mathbf{1}_{1\times B}$ (read-after-write).
    \State $\widehat{\mathbf{O}}^{(k)}\gets \mathbf{H}^{(k)}+\big(\mathbf{P}^{(k)}\odot L^{(k)}\big)\mathbf{Z}^{(k)}$. \Comment{local-renormalized outputs}
    \State Row-rescale: $\mathbf{O}^{(k)}_{i,:}\gets \delta^{(k)}_{i}\,\widehat{\mathbf{O}}^{(k)}_{i,:}$ for $i=1{:}C$, and write into rows $a_k{:}b_k$ of $\mathbf{O}$.
\EndFor
\State $\mathcal{T}_T\gets\{(\mathbf{x}_j,\vb_j)\}_{j=\max(2,T-B+2)}^{T}$.
\State \Return $(\mathbf{O},\Sbb_T,\mathcal{T}_T)$.
\end{algorithmic}
\end{breakablealgorithm}

\noindent\textbf{Boundary/scan details.}
The offline overlap interpretation, exact continuation requirement, and explicit associative chunk map are given in Appendix~\ref{app:sliding_window_remarks} and Appendix~\ref{app:Falcon3_chunk_map}.

\subsection{Mini-batch Update Rule for Inner Product Loss (Falcon-3A)}
\label{sec:methods_sliding_ip}

% =========================================================
% Figure: Three equivalent views of Falcon-3A
%   A. Recurrent Form (sliding-window mini-batch causal scan)
%   B. Parallel Form  (window-induced decay-mask attention)
%   C. Chunk-wise Parallel Form
%
% Requires the same preamble as Figure \ref{fig:falcon1a_equivalence}:
%   - colors: appleBlue, appleLightBlue, appleRed, appleLightRed,
%     appleGreen, appleLightGreen, appleGray, appleLightGray,
%     appleDarkGray, appleDark, appleOrange, appleLightOrange,
%     appleTeal, appleLightTeal, applePurple, appleLightPurple
%   - tikzlibraries: positioning, calc, shapes.geometric,
%     arrows.meta, backgrounds, fit, matrix, shadows.blur
% =========================================================
\begin{figure}[ht!]
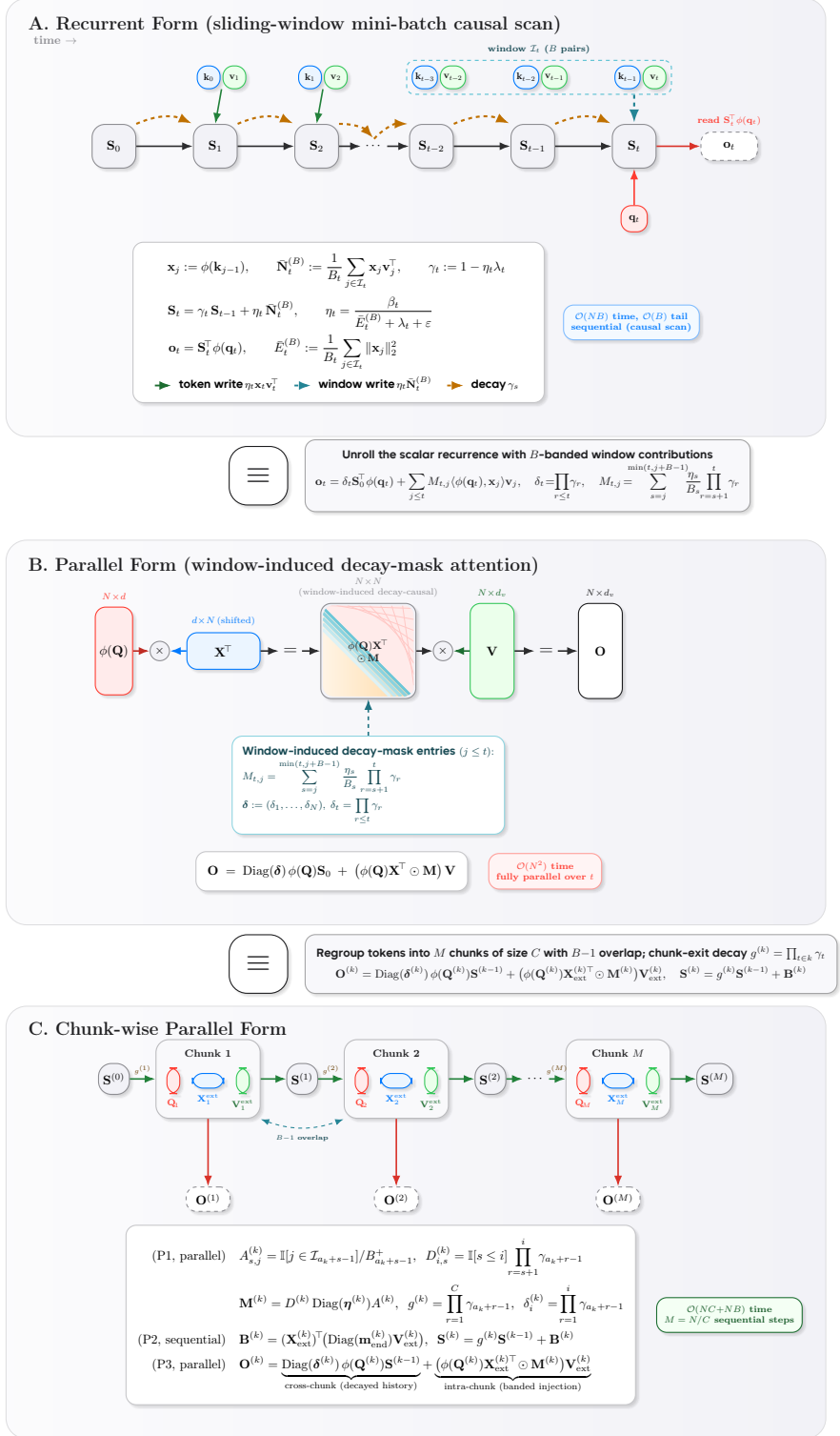

\centering
\resizebox{0.7\linewidth}{!}{%

}
\caption{\textbf{Three equivalent views of \Falcon-3A.}
(A) The recurrent form maintains a fixed-size matrix state $\mathbf{S}_t\in\mathbb{R}^{d_x\times d_v}$.
(B) Unrolling the scalar recurrence with the windowed write yields a masked-attention form with the same shifted features and a \emph{window-induced decay-weighted} causal mask. The window operator is $B$-banded, but the resulting mask generally has a dense lower-triangular decay tail.
(C) The chunk-wise form splits the sequence into $M$ chunks of size $C$.}
\label{fig:falcon3a_equivalence}
\end{figure}

We apply the same sliding-window principle to the inner-product objective, while keeping the same parameter semantics as in the rest of the family: $\beta_t$ is the dimensionless gain, $\lambda_t$ is the actual shrinkage coefficient used by the recurrence (after any optional scale coupling), and $\eta_t$ is the resulting normalized step size.

For $t\ge 2$, let $B_t:=|\mathcal{I}_t|\le B$ (and skip the boundary update at $t=1$). The windowed loss is
\begin{align}
\ell_t^{\rm ip,(B)}(\Sbb)
:= - \frac{1}{B_t}\sum_{j \in \mathcal{I}_t} \big\langle \Sbb^\top \mathbf{x}_{j}, \vb_j \big\rangle
+ \frac{\lambda_t}{2}\|\Sbb\|_F^2.
\end{align}
Define the window-averaged cross-covariance and write energy
\[
\bar{\Nb}_t^{(B)} := \frac{1}{B_t}\sum_{j \in \mathcal{I}_t} \mathbf{x}_{j}\vb_j^\top,
\qquad
\bar E_t^{(B)} := \frac{1}{B_t}\sum_{j\in\mathcal{I}_t}\|\mathbf{x}_j\|_2^2.
\]
Then the gradient evaluated at the pre-update state is
\[
\nabla_{\Sbb} \ell_t^{\rm ip,(B)}(\Sbb_{t-1})
= - \bar{\Nb}_t^{(B)} + \lambda_t \Sbb_{t-1}.
\]

Unlike \Falcon-3, whose regression step size uses the local smoothness $\mu_t^{(B)}$, the inner-product windowed rule uses the window-energy statistic as a practical write-gain normalizer. Let $E_t^{(B)}:=\bar E_t^{(B)}$, and if the scale-coupled parameterization is active set $\lambda_t:=\bar\lambda_t E_t^{(B)}$ before computing the step size:
\begin{equation}
\label{eq:ip_mb_stepsize}
\eta_t
= \frac{\beta_t}{E_t^{(B)}+\lambda_t+\varepsilon},
\qquad
\beta_t\in(0,2),\ \varepsilon>0,
\end{equation}
with the boundary convention $\eta_1:=0$. Applying one gradient step gives
\begin{equation}
\label{eq:Falcon3a_update}
\Sbb_t
=
(1-\eta_t\lambda_t)\Sbb_{t-1}
+
\eta_t\,\bar{\Nb}_t^{(B)}.
\end{equation}

When $\lambda_t>0$, the objective is $\lambda_t$-smooth, so Eq.~\eqref{eq:ip_mb_stepsize} implies $\eta_t<2/\lambda_t$ for any $\beta_t\in(0,2)$, and Lemma~\ref{lem:per-step-descent} yields per-step descent in $\ell_t^{\rm ip,(B)}$ for the unclamped update. If the positive-decay clamp below activates, the implemented shrinkage should again be interpreted as using the effective ridge coefficient $\tilde\lambda_t:=\alpha_t/\eta_t$ when $\eta_t>0$ (and $0$ when $\eta_t=0$), rather than as an exact gradient step for the original $\lambda_t$. When $\lambda_t=0$, the objective is linear and unbounded below, so the same normalization should be interpreted as a magnitude stabilizer for the additive write rather than as a bounded-objective guarantee.

For log-space unrolling, we introduce only the derived quantities
\begin{equation}
\label{eq:Falcon3a_decayfrac}
\alpha_t^{\rm raw}:=\eta_t\lambda_t,\qquad
\alpha_t:=\min(\alpha_t^{\rm raw},\,1-\varepsilon_\gamma),\qquad
\gamma_t:=1-\alpha_t\in[\varepsilon_\gamma,1],
\end{equation}
with $\alpha_1:=0$ and $\gamma_1:=1$. Thus, $\alpha_t$ and $\gamma_t$ are implementation variables derived from the same $(\beta_t,\lambda_t,\eta_t)$ parameterization used elsewhere; they are not separate learned controls.

Because we optimize the window average, neither the typical scale of $\bar{\Nb}_t^{(B)}$ nor the decay fraction $\alpha_t=\eta_t\lambda_t$ grows systematically with the nominal window size $B$. We again set $\eta_t:=0$ when the denominator in Eq.~\eqref{eq:ip_mb_stepsize} vanishes (in practice, we take $\varepsilon>0$). Appendix~\ref{app:sliding_window_remarks} records the stationary calculation and the exact boundary-state requirements. For $B=1$ and $t\ge 2$, Eq.~\eqref{eq:Falcon3a_update} reduces to the scalar non-sliding inner-product update, i.e.\ \Falcon-1A; in the additive ablation $\lambda_t\equiv 0$, the write is purely additive. The per-column non-sliding analogue is \Falcon-2A in Eq.~\eqref{eq:Falcon2a_update}.

\subsection{Parallel (Attention) Form of Mini-batch Inner Product Update}
\label{sec:mb_ip_parallel}

\begin{breakablealgorithm}
\caption{Chunk-parallel masked attention for Falcon-3A (Forward)}
\label{alg:Falcon3a_chunked_attn}
\small
\begin{algorithmic}[1]
\Require Query features $\Qb\in\mathbb{R}^{L\times d_x}$ (row $t$ is $\phi(\qb_t)^\top$), write features $\mathbf{X}\in\mathbb{R}^{L\times d_x}$ (row $t$ is $\mathbf{x}_t^\top=\phi(\mathbf{k}_{t-1})^\top$), values $\Vb\in\mathbb{R}^{L\times d_v}$ (row $t$ is $\vb_t^\top$), gains $\boldsymbol{\beta}\in(0,2)^L$, actual ridge coefficients $\boldsymbol{\lambda}\in\mathbb{R}_{\ge 0}^{L}$, stabilizer $\varepsilon>0$, window size $B$, chunk size $C$ (assume $C\mid L$), decay floor $\varepsilon_\gamma>0$, initial state $\Sbb_{\rm init}\in\mathbb{R}^{d_x\times d_v}$.
\Ensure Outputs $\Ob\in\mathbb{R}^{L\times d_v}$ where row $t$ is $\ob_t^\top=\phi(\qb_t)^\top\Sbb_t$ under
$\eta_1=0$, $\gamma_1=1$, $\bar{\Nb}^{(B)}_1=\mathbf{0}$, $\bar E_1^{(B)}=0$, and for $t\ge2$,
$\Sbb_t=\gamma_t\Sbb_{t-1}+\eta_t\bar{\Nb}^{(B)}_t$ with
$\bar{\Nb}^{(B)}_t=\frac{1}{B_t}\sum_{j\in\mathcal{I}_t}\mathbf{x}_j\vb_j^\top$ and
$\bar E_t^{(B)}=\frac{1}{B_t}\sum_{j\in\mathcal{I}_t}\|\mathbf{x}_j\|_2^2$,
$\eta_t=\beta_t/(\bar E_t^{(B)}+\lambda_t+\varepsilon)$,
and $\gamma_t:=1-\alpha_t$ with $\alpha_t:=\min(\eta_t\lambda_t,1-\varepsilon_\gamma)$.
\State Set $\mathcal{I}_1\gets\emptyset$, $B_1\gets0$. For $t\ge2$, define $\mathcal{I}_t=\{j\mid \max(2,t-B+1)\le j\le t\}$ and set $B_t\gets|\mathcal{I}_t|$.
\State When forming window weights, use $B_t^+\gets\max(1,B_t)$ so that $\mathbb{I}[j\in\mathcal{I}_t]/B_t^+$ is always well-defined.
\State Set $\bar E_1^{(B)}\gets 0$, and for $t\ge 2$ compute $\bar E_t^{(B)}\gets \frac{1}{B_t}\sum_{j\in\mathcal{I}_t}\|\mathbf{x}_j\|_2^2$.
\State If the scale-coupled ridge parameterization is active, set $\lambda_1\gets 0$ and for $t\ge 2$ set $\lambda_t\gets \bar\lambda_t\,\bar E_t^{(B)}$; otherwise treat the provided $\lambda_t$ as the actual coefficients.
\State Set $\eta_1\gets 0$. For $t\ge 2$, compute $\eta_t\gets \beta_t/(\bar E_t^{(B)}+\lambda_t+\varepsilon)$.
\State Set $\alpha_1\gets 0$, $\gamma_1\gets 1$, and $\log\gamma_1\gets 0$. For $t\ge 2$, compute $\alpha_t\gets \min(\eta_t\lambda_t,1-\varepsilon_\gamma)$, $\log\gamma_t\gets \operatorname{log1p}(-\alpha_t)$, and $\gamma_t\gets \exp(\log\gamma_t)$.
\State Partition into $M=L/C$ chunks with indices $[a_k,b_k]=[(k{-}1)C{+}1,kC]$.
\State Initialize output buffer $\Ob\gets \mathbf{0}_{L\times d_v}$.

\State \Comment{\textbf{Phase 1: Intra-chunk precomputation (parallel)}}
\For{$k=1$ \textbf{to} $M$ \textbf{in parallel}}
\State $a\gets a_k$, $b\gets b_k$, $p\gets \max(2,a-B+1)$, and let $J:=\{p,p{+}1,\dots,b\}$ (overlap length $|J|\le C+B-1$).
\State Slice $\Qb^{(k)}\gets \Qb_{a:b}$,\; $\mathbf{X}_{\rm ext}^{(k)}\gets \mathbf{X}_{J}$,\; $\Vb_{\rm ext}^{(k)}\gets \Vb_{J}$,\; $\boldsymbol{\eta}^{(k)}\gets \boldsymbol{\eta}_{a:b}$,\; and $\log\boldsymbol{\gamma}^{(k)}\gets (\log\boldsymbol{\gamma})_{a:b}$.
\State Compute log-prefix decays inside the chunk: $u_0\gets 0$ and $u_i\gets u_{i-1}+\log\gamma^{(k)}_i$ for $i=1,\ldots,C$ (scan); set $\delta_i\gets \exp(u_i)$ and $g^{(k)}\gets \delta_C$.
\State Build the causal decay matrix $D^{(k)}\in\mathbb{R}^{C\times C}$ with $D^{(k)}_{i,s}=\mathbb{I}[s\le i]\exp(u_i-u_s)$.
\State Build the $B$-banded window operator $A^{(k)}\in\mathbb{R}^{C\times |J|}$ (columns indexed by $j\in J$):
$A^{(k)}_{s,j}=\mathbb{I}[\,j\in\mathcal{I}_{a+s-1}\,]/B^+_{a+s-1}$.
\State Local mask $M^{(k)}\gets D^{(k)}\,\operatorname{Diag}(\boldsymbol{\eta}^{(k)})\,A^{(k)}\in\mathbb{R}^{C\times |J|}$.
\State Intra-chunk output $\Ob_{\rm intra}^{(k)}\gets \big((\Qb^{(k)}{\mathbf{X}_{\rm ext}^{(k)}}^\top)\odot M^{(k)}\big)\Vb_{\rm ext}^{(k)}$.
\State Chunk bias for boundary propagation:
$m_{\rm end}^{(k)}\gets M^{(k)}_{C,:}$ and $\Bb^{(k)}\gets {\mathbf{X}_{\rm ext}^{(k)}}^\top\big(\operatorname{Diag}(m_{\rm end}^{(k)})\Vb_{\rm ext}^{(k)}\big)$.
\State Cache $g^{(k)}$, $\Bb^{(k)}$, $\boldsymbol{\delta}^{(k)}:=(\delta_1,\ldots,\delta_C)$, and $\Ob_{\rm intra}^{(k)}$.
\EndFor

\State \Comment{\textbf{Phase 2: Inter-chunk recurrence (sequential over chunks or scan)}}
\State $\Sbb_{\rm in}^{(1)}\gets \Sbb_{\rm init}$
\For{$k=1$ \textbf{to} $M$}
\State Store $\Sbb_{\rm in}^{(k)}$ \Comment{Needed by Phase 3 to materialize the decayed-history term}
\State $\Sbb_{\rm out}^{(k)}\gets g^{(k)}\Sbb_{\rm in}^{(k)}+\Bb^{(k)}$
\State $\Sbb_{\rm in}^{(k+1)}\gets \Sbb_{\rm out}^{(k)}$
\EndFor

\State \Comment{\textbf{Phase 3: Materialize outputs (parallel)}}
\For{$k=1$ \textbf{to} $M$ \textbf{in parallel}}
\State $H^{(k)}\gets \Qb^{(k)}\Sbb_{\rm in}^{(k)}$ \Comment{$C\times d_v$}
\State $\Ob_{\rm hist}^{(k)}\gets \operatorname{Diag}(\boldsymbol{\delta}^{(k)})\,H^{(k)}$
\State $\Ob^{(k)}\gets \Ob_{\rm hist}^{(k)}+\Ob_{\rm intra}^{(k)}$
\State Write $\Ob_{a_k:b_k}\gets \Ob^{(k)}$.
\EndFor
\State \Return $\Ob$ (concatenate $\{\Ob^{(k)}\}_{k=1}^M$).
\end{algorithmic}
\end{breakablealgorithm}

We show that the recurrent sliding-window update in Section~\ref{sec:methods_sliding_ip} can be written as the sum of (i) a decayed-history term from the incoming boundary state and (ii) an (unnormalized) dot-product attention matrix with a structured causal mask. This view enables GPU-parallel training and matches the usual unrolling of gated linear recurrences into masked attention~\citep{sun2023retentive, qin2024lightning, yang2023gated, gu2023mamba,dao2024transformers}.

For the unclamped exposition, let $\lambda_s$ denote the actual shrinkage coefficient after any optional scale coupling, and define
\[
\eta_s := \frac{\beta_s}{\bar E_s^{(B)}+\lambda_s+\varepsilon},
\qquad
\gamma_s := 1-\eta_s\lambda_s,
\qquad
\delta_t:=\prod_{r=1}^{t}\gamma_r.
\]
Then the recurrence $\Sbb_s=\gamma_s \Sbb_{s-1}+\eta_s\bar{\Nb}_s^{(B)}$ unrolls to
\[
\Sbb_t=\delta_t \Sbb_0+\sum_{j=2}^{t} M_{t,j}\,\mathbf{x}_j\vb_j^\top,
\]
where, under the boundary convention $\mathbf{x}_1:=\mathbf{0}$ and $\eta_1=0$,
\[
M_{t,j}:=
\sum_{s=j}^{\min(t,j+B-1)}
\frac{\eta_s}{B_s}\prod_{r=s+1}^{t}\gamma_r,
\qquad 2\le j\le t,
\]
and $M_{t,j}:=0$ otherwise. Hence, the read-after-write output is
\[
\ob_t
=
\delta_t\,\Sbb_0^\top \phi(\qb_t)
\;+\;
\sum_{j=2}^{t} M_{t,j}\,\langle \phi(\qb_t), \mathbf{x}_{j} \rangle \vb_j.
\]

Equivalently, stacking query features $\Qb\in\mathbb{R}^{L\times d_x}$, write features $\mathbf{X}\in\mathbb{R}^{L\times d_x}$, values $\Vb\in\mathbb{R}^{L\times d_v}$, and mask $\Mb\in\mathbb{R}^{L\times L}$ with entries $\Mb_{t,j}=M_{t,j}$, we obtain
\[
\Ob
=
\operatorname{Diag}(\boldsymbol{\delta})\,\Qb\Sbb_0
\;+\;
(\Qb\mathbf{X}^\top\odot \Mb)\Vb,
\qquad
\boldsymbol{\delta}:=(\delta_1,\ldots,\delta_L)^\top.
\]
For the fresh-sequence default $\Sbb_0=\mathbf{0}$, only the masked-attention term remains. Appendix~\ref{app:Falcon3a_mask} gives the full derivation, the stationary special case, and the chunk-local log-space evaluation used by Algorithm~\ref{alg:Falcon3a_chunked_attn}.

\paragraph{Backward pass.}
The numerically stable backward pass is given in Appendix~\ref{app:Falcon3a_backward}; it differentiates through the structured mask, chunk-local log-decay renormalization, and normalized step-size computation.

\paragraph{Relation to prior internal-objective views.}
A brief comparison to test-time training and Titans- and ATLAS-style internal memory objectives is moved to Appendix~\ref{app:ttt_titans_relation}.

\section{Experiments}
\label{sec:experiments}

Our empirical evaluation emphasizes the scalar and sliding inner-product variants, \Falcon-1A and \Falcon-3A, and includes \Falcon-1.3 as a representative regression ablation. The 124M--130M-parameter language models are trained on FineWeb-Edu~\citep{penedo2024fineweb} with a matched 50B-token budget and evaluated by held-out perplexity and downstream task accuracy. We also use variable-length multi-digit addition as a controlled diagnostic for causal storage and length extrapolation. The per-column rules \Falcon-2 and \Falcon-2A, as well as the sliding regression rule \Falcon-3, are defined but not separately benchmarked in the main tables. Unless stated otherwise, fast-weight models use the scaled formulation of Section~\ref{sec:scaled_fast_weight} and the next-latent boundary convention $\mathbf{x}_1:=\mathbf{0}$. QK-RMSNorm and QK-$\ell_2$ denote the query/key normalization used inside the fast-weight block. In model names, ctx$\beta$, ctx$\eta$, and ctx$\lambda$ abbreviate context-conditioned $\beta$, $\eta$, and $\lambda$, respectively.

\subsection{Language modeling experiments}
\label{sec:results}

We train the 124M--130M-parameter models for 100{,}000 optimization steps with sequence length 1{,}024 and global batch size 480, for a total budget of approximately 49.2B tokens. The Transformer baseline uses a LLaMA-style architecture with RoPE and SwiGLU. Recurrent baselines include RetNet/LightningAttn, Mamba-2, DeltaNet, and Gated DeltaNet. Table~\ref{tab:perplexity_results} reports small-model perplexity; Tables~\ref{tab:accuracy_results_small} and~\ref{tab:accuracy_results_medium} report small- and medium-model downstream accuracy, respectively. The tables include scalar \Falcon-1A variants, \Falcon-3A.3 where available, and the regression ablation \Falcon-1.3. Unless explicitly ablated, fast-weight models use QK-RMSNorm and lightweight short convolutions on the attention projections. Full optimization and implementation details are deferred to Appendix~\ref{app:exp_setup}. We report teacher-forced perplexity and zero-shot / one-shot downstream accuracy.

On FineWeb-Edu validation perplexity (Table~\ref{tab:perplexity_results}), \Falcon-1.3 is strongest overall at 17.10; Gated DeltaNet is the strongest baseline at 17.32, and \Falcon-1A.3 is the best evaluated inner-product variant at 17.40. On the 124M--130M downstream evaluation (Table~\ref{tab:accuracy_results_small}), \Falcon-1A.2 has the best zero-shot average among the listed models (49.30), while \Falcon-1.3 has the best recurrent one-shot average (49.54). Within the small scalar inner-product ablation, QK-RMSNorm improves FineWeb-Edu perplexity relative to QK-$\ell_2$ normalization, and context-conditioned $\eta$ improves the small-model zero- and one-shot averages over the corresponding context-conditioned $\beta$ variant. The empirical takeaway is therefore not a uniform win, but that the proposed aligned, normalized updates preserve competitive language-model quality while providing the controlled arithmetic gains reported below.

\begin{table}[ht!]
\centering
\small
\setlength{\tabcolsep}{3.6pt}
\renewcommand{\arraystretch}{1.12}
\caption{
Comparison of models with 124M-130M parameters trained for a 50B-token budget on FineWeb-Edu. Lower is better. Boldface marks the best recurrent model in each column.
}
\label{tab:perplexity_results}
\begin{tabular}{l *{3}{S[table-format=2.2]}}
\toprule
\textbf{Model}
& \multicolumn{3}{c}{\textbf{Perplexity} $\downarrow$ \,} \\
\cmidrule(lr){2-4}
& \textbf{Wiki.} & \textbf{LMB.} & \textbf{FineEdu.} \\
\midrule
\multicolumn{4}{l}{\bfseries 124M-130M parameters}\\
\addlinespace[0.15em]

\hspace{2mm} (124M) Transformer (w.\ RoPE)
& 33.25 & 47.43 & 17.38 \\

\addlinespace[0.25em]\midrule\addlinespace[0.15em]

\hspace{2mm} (130M) RetNet/LightningAttn
& 36.86 & 65.16 & 18.79 \\
\hspace{2mm} (130M) Mamba-2
& 34.53 & 48.74 & 17.70 \\
\hspace{2mm} (130M) DeltaNet
& 34.19 & 52.84 & 17.84 \\
\hspace{2mm} (130M) Gated DeltaNet 
& \textbf{30.99} & \textbf{46.70} & 17.32 \\

\addlinespace[0.25em]\midrule\addlinespace[0.15em]
\multicolumn{4}{l}{\textit{Ours}}\\
\addlinespace[0.05em]
\rowcolor{black!2}
\hspace{2mm} (130M) \textbf{Falcon-1A.1 (QK-$\ell_2$-norm, ctx$\beta$-ctx$\lambda$)}
& 34.41 & 47.93 & 17.70 \\
\rowcolor{black!2}
\hspace{2mm} (130M) \textbf{Falcon-1A.2 (QK-$\ell_2$-norm, ctx$\eta$-ctx$\lambda$)}
& 34.20 & 51.01 & 17.70 \\
\rowcolor{black!2}
\hspace{2mm} (130M) \textbf{Falcon-1A.3 (QK-RMSNorm, ctx$\eta$-ctx$\lambda$)}
& 34.02 & 49.84 & 17.40 \\
\rowcolor{black!2}
\hspace{2mm} (130M) \textbf{Falcon-1.3 (QK-RMSNorm, ctx$\eta$-ctx$\lambda$)}
& 33.00 & 48.70 & \textbf{17.10} \\
\bottomrule
\end{tabular}
\end{table}

\begin{table*}[ht!]
\centering
\small
\setlength{\tabcolsep}{3.6pt}
\renewcommand{\arraystretch}{1.12}
\caption{
\textbf{Evaluation results on downstream tasks.}
Small models' zero-shot and one-shot accuracy. Accuracies use \texttt{acc}; tasks marked with $^*$ use \texttt{acc\_n} as in \texttt{lm-evaluation-harness}. Avg. is the unweighted average across the 8 tasks. Boldface marks the best recurrent model in each column.
}
\label{tab:accuracy_results_small}
\resizebox{\textwidth}{!}{%
\begin{tabular}{
l
*{8}{S[table-format=2.2]}
S[table-format=2.2]
}
\toprule
\textbf{Model}
& \multicolumn{9}{c}{\textbf{Accuracy} $\uparrow$} \\
\cmidrule(lr){2-10}
& \textbf{PIQA} & \textbf{Hella.}\textsuperscript{*} & \textbf{Wino.} & \textbf{ARC-e} & \textbf{ARC-c}\textsuperscript{*}
& \textbf{OBQA}\textsuperscript{*} & \textbf{Social IQA} & \textbf{SciQ} & \textbf{Avg.} \\
\midrule

\rowcolor{black!6}
\multicolumn{10}{c}{\bfseries Zero-shot (0-shot)}\\
\addlinespace[0.25em]
\multicolumn{10}{l}{\bfseries 124M-130M parameters}\\
\addlinespace[0.15em]

\hspace{2mm} (124M) Transformer (w.\ RoPE)
& 65.67 & 37.54 & 51.70 & 52.36 & 27.65 & 31.60 & 38.84 & 79.90 & 48.16 \\

\addlinespace[0.25em]\midrule\addlinespace[0.15em]
\hspace{2mm} (130M) RetNet/LightningAttn
& 64.91 & 35.36 & 49.64 & 57.62 & 26.28 & 32.20 & 37.97 & 80.40 & 48.05 \\
\hspace{2mm} (130M) Mamba-2
& 66.32 & 36.89 & 50.75 & 58.16 & 26.62 & 32.60 & 38.38 & 80.70 & 48.80 \\
\hspace{2mm} (130M) DeltaNet
& 66.38 & 37.15 & \textbf{52.33} & 57.37 & 26.79 & \textbf{34.00} & 39.00 & 78.00 & 48.88 \\
\hspace{2mm} (130M) Gated DeltaNet
& 65.51 & 37.75 & 49.72 & 58.88 & \textbf{27.90} & 31.60 & 38.28 & 80.60 & 48.78 \\

\addlinespace[0.25em]\midrule\addlinespace[0.15em]
\multicolumn{10}{l}{\textit{Ours}}\\
\addlinespace[0.05em]
\rowcolor{black!2}
\hspace{2mm} (130M) \textbf{Falcon-1A.1 (QK-$\ell_2$-norm, ctx$\beta$-ctx$\lambda$)}
& 66.05 & 37.27 & 50.67 & 57.62 & 27.65 & 31.60 & 38.95 & 81.10 & 48.86 \\
\rowcolor{black!2}
\hspace{2mm} (130M) \textbf{Falcon-1A.2 (QK-$\ell_2$-norm, ctx$\eta$-ctx$\lambda$)}
& \textbf{67.03} & 37.29 & \textbf{52.33} & 57.37 & 25.94 & 33.20 & 38.84 & \textbf{82.40} & \textbf{49.30} \\
\rowcolor{black!2}
\hspace{2mm} (130M) \textbf{Falcon-1A.3 (QK-RMSNorm, ctx$\eta$-ctx$\lambda$)}
& 66.10 & 37.55 & 50.12 & \textbf{59.01} & 26.79 & 32.00 & 37.82 & 82.20 & 48.95 \\
\rowcolor{black!2}
\hspace{2mm} (130M) \textbf{Falcon-3A.3 (QK-RMSNorm, ctx$\eta$-ctx$\lambda$)}
& 65.34 & 37.30 & 50.99 & 57.37 & 26.37 & 33.60 & \textbf{39.41} & 81.60 & 49.00 \\
\rowcolor{black!2}
\hspace{2mm} (130M) \textbf{Falcon-1.3 (QK-RMSNorm, ctx$\eta$-ctx$\lambda$)}
& 65.83 & \textbf{38.38} & 52.25 & 58.96 & 26.62 & 31.40 & 38.69 & 81.30 & 49.18 \\

\addlinespace[0.35em]
\specialrule{0.10em}{0.25em}{0.25em}
\addlinespace[0.10em]

\rowcolor{black!6}
\multicolumn{10}{c}{\bfseries One-shot (1-shot)}\\
\addlinespace[0.25em]

\hspace{2mm} (124M) Transformer (w.\ RoPE)
& 66.43 & 37.55 & 50.28 & 59.64 & 29.01 & 30.00 & 39.82 & 84.60 & 49.67 \\

\addlinespace[0.25em]\midrule\addlinespace[0.15em]
\hspace{2mm} (130M) RetNet/LightningAttn
& 65.13 & 35.19 & 49.64 & 56.78 & 25.85 & 28.80 & 36.44 & 81.50 & 47.42 \\
\hspace{2mm} (130M) Mamba-2
& 66.70 & 36.61 & 51.07 & 58.63 & 26.96 & \textbf{32.40} & 37.97 & 82.90 & 49.16 \\
\hspace{2mm} (130M) DeltaNet
& 66.27 & 36.67 & 50.36 & 57.70 & 27.39 & 32.00 & 37.72 & 79.90 & 48.50 \\
\hspace{2mm} (130M) Gated DeltaNet 
& 65.40 & 37.81 & 51.30 & 58.29 & 26.88 & 30.00 & 37.26 & 81.60 & 48.57 \\

\addlinespace[0.25em]\midrule\addlinespace[0.15em]
\multicolumn{10}{l}{\textit{Ours}}\\
\addlinespace[0.05em]
\rowcolor{black!2}
\hspace{2mm} (130M) \textbf{Falcon-1A.1 (QK-$\ell_2$-norm, ctx$\beta$-ctx$\lambda$)}
& \textbf{66.81} & 37.41 & 50.75 & 57.53 & 27.99 & \textbf{32.40} & 37.92 & 81.80 & 49.08 \\
\rowcolor{black!2}
\hspace{2mm} (130M) \textbf{Falcon-1A.2 (QK-$\ell_2$-norm, ctx$\eta$-ctx$\lambda$)}
& 66.38 & 36.97 & \textbf{52.96} & 57.32 & 27.82 & 31.40 & 37.56 & 83.20 & 49.20 \\
\rowcolor{black!2}
\hspace{2mm} (130M) \textbf{Falcon-1A.3 (QK-RMSNorm, ctx$\eta$-ctx$\lambda$)}
& 65.78 & 37.22 & 49.72 & 58.88 & 27.73 & 31.40 & 37.97 & 82.40 & 48.89 \\
\rowcolor{black!2}
\hspace{2mm} (130M) \textbf{Falcon-3A.3 (QK-RMSNorm, ctx$\eta$-ctx$\lambda$)}
& 65.72 & 36.47 & 51.78 & 58.29 & 26.54 & 32.00 & \textbf{38.23} & 83.20 & 49.03 \\
\rowcolor{black!2}
\hspace{2mm} (130M) \textbf{Falcon-1.3 (QK-RMSNorm, ctx$\eta$-ctx$\lambda$)}
& 65.67 & \textbf{38.09} & 52.80 & \textbf{59.55} & \textbf{29.01} & 30.40 & 37.46 & \textbf{83.40} & \textbf{49.54} \\
\bottomrule
\end{tabular}
}
\end{table*}

\subsection{Variable-length arithmetic addition tasks}
\label{sec:task_addition}

We use variable-length multi-digit addition as a controlled stress test for causal storage and extrapolation, following \citet{kaiser2015neural}. Each sample presents an $n$-digit prompt

\[
s_{\text{in}} := \texttt{"}\mathrm{digits}_n(a)\ \texttt{+}\ \mathrm{digits}_n(b)\ \texttt{=} \texttt{"},
\]
and asks the model to generate the reversed $(n{+}1)$-digit sum
\[
s_{\text{out}} := \texttt{"}\operatorname{rev}\big(\mathrm{digits}_{n+1}(a+b)\big)\texttt{"},
\]
so that the least significant digit is produced first. We train on widths sampled uniformly from $\{1,\dots,32\}$ and optimize masked next-token log-likelihood on the target suffix only.

Table~\ref{tab:arithmetics-digit-generalization} reports in-distribution validation accuracy and out-of-distribution teacher-forced target-suffix accuracy averaged over 33--48 digits. \Falcon-3A.3 achieves the best extrapolation performance, with 87.2 mean accuracy, followed by \Falcon-1A.3 at 85.9. Both outperform the baselines reported here, including RetNet/LightningAttn and the Transformer. We view this experiment as supporting evidence rather than the paper's primary result: it isolates the memory-writing behavior of the recurrent state and shows that the proposed shifted, normalized updates extrapolate well when storage and carry propagation dominate.

\begin{table}[ht!]
\centering
\tiny
\setlength{\tabcolsep}{3.6pt}
\renewcommand{\arraystretch}{1.12}
\caption{
Teacher-forced length generalization on variable-digit addition. Higher is better.
}
\label{tab:arithmetics-digit-generalization}
\begin{tabular}{l S[table-format=4.0] S[table-format=3.1] S[table-format=2.1] c}
\toprule
\textbf{Model}
& \textbf{Best step}
& \textbf{Val. acc.}
& \textbf{Mean acc.}
& \textbf{Acc@d33/d48} \\
\midrule

\hspace{2mm} Transformer (w.\ RoPE)
& 2000 & 100.0 & 65.8 & 97.0/49.0 \\

\addlinespace[0.25em]\midrule\addlinespace[0.15em]

\hspace{2mm} RetNet/LightningAttn
& 2000 & 99.7 & 82.9 & 99.0/63.0 \\
\hspace{2mm} Mamba-2
& 2000 & 100.0 & 75.2 & 100.0/51.0 \\

\addlinespace[0.25em]\midrule\addlinespace[0.15em]

\multicolumn{5}{l}{\textit{Ours}}\\
\addlinespace[0.05em]
\rowcolor{black!2}
\hspace{2mm}\textbf{Falcon-1A.1 (QK-$\ell_2$-norm, ctx$\beta$-ctx$\lambda$)}
& 1900 & 100.0 & 80.6 & 100.0/59.0 \\
\rowcolor{black!2}
\hspace{2mm}\textbf{Falcon-1A.2 (QK-$\ell_2$-norm, ctx$\eta$-ctx$\lambda$)}
& 2000 & 100.0 & 85.2 & 100.0/63.0 \\
\rowcolor{black!2}
\hspace{2mm}\textbf{Falcon-1A.3 (QK-RMSNorm, ctx$\eta$-ctx$\lambda$)}
& 1900 & 99.8 & 85.9 & 100.0/69.0 \\
\rowcolor{black!2}
\hspace{2mm}\textbf{Falcon-3A.3 (QK-RMSNorm, ctx$\eta$-ctx$\lambda$)}
& 2000 & 99.9 & \multicolumn{1}{c}{\bfseries 87.2} & 100.0/69.0 \\
\rowcolor{black!2}
\hspace{2mm}\textbf{Falcon-1.3 (QK-RMSNorm, ctx$\eta$-ctx$\lambda$)}
& 2000 & 100.0 & 68.8 & 100.0/48.0 \\
\bottomrule
\end{tabular}
\end{table}

\section{Related Work}

\paragraph{Efficient sequence models.}
Standard Transformers~\citep{vaswani2017attention} incur quadratic $\mathcal{O}(N^2)$ compute and memory costs due to the materialization of the attention matrix. To address this, Linear Attention~\citep{katharopoulos2020transformers} reorders the matrix multiplication via kernel feature maps, effectively treating the context as a recurrent accumulation of outer products. Other subquadratic attention replacements include random-feature approximations to softmax attention, such as Performer~\citep{choromanski2020rethinking}, long-convolution operators such as Hyena~\citep{poli2023hyena}, and retention-style recurrent attention~\citep{sun2023retentive} or RNN–Transformer hybrids~\citep{peng2023rwkv}. Parallel development in Structured State Space Models (SSMs), such as S4~\citep{gu2021efficiently} and Mamba~\citep{gu2023mamba}, utilized discretized continuous-time dynamics to achieve similar linear scaling.
Recently, Mamba-2~\citep{dao2024transformers} unified these paradigms under Structured State Space Duality (SSD), proving that selective SSMs are algorithmically dual to linear attention with semi-separable masks.
While Mamba-2 and related SSD-style models focus on hardware utilization (via chunk-parallel scans) and expressive discretizations, they largely retain additive or gated accumulation as the state write rule. In contrast, our work re-examines the objective that induces the state update. We derive autoregressively aligned, NLMS-stabilized regression updates and sliding-window variants that remain compatible with SSD-style chunk-parallel training.

\paragraph{Fast Weights and Delta Networks.}
Fast-weight models~\citep{schmidhuber1992learning} maintain a high-capacity short-term memory matrix updated by the input stream. In \emph{Linear Transformers Are Secretly Fast Weight Programmers}, \citet{schlag2021linear} introduced the Delta Network, whose state update is a gradient step on a value-reconstruction error rather than a purely Hebbian write. \citet{yang2024gated} extended this mechanism with Mamba-style gating. Building on that optimization view, our work focuses on autoregressive feature/target alignment, objective-matched normalization, per-column and sliding-window variants, and exact chunk-parallel implementations.

\paragraph{Adaptive Filtering and Online Regression.} 
The delta-rule updates used in fast-weight models are closely related to classical adaptive filtering, where the LMS delta rule and its normalized variant (NLMS) are standard algorithms for scale-robust online regression, and second-order methods such as recursive least squares (RLS) provide exact online ridge updates at higher cost \citep{sayed2011adaptive}. Our NLMS-based regression family can be viewed as importing these stability and normalization principles into fast weight attention under strict causality, while remaining compatible with modern parallel implementations of linear recurrences \citep{yang2024parallelizing, dao2024transformers, cirone2025parallelflow}.

\paragraph{In-Context Learning as Implicit Optimization.}
Recent work interprets in-context learning and test-time adaptation as implicit optimization in feature space~\citep{von2023transformers, ahn2023transformers, cheng2023transformers, sun2024learning, wang2025test}. MesaNet~\citep{von2025mesanet} makes this viewpoint explicit by solving least-squares problems inside the forward pass, while Titans and ATLAS optimize internal memory objectives over the stream~\citep{behrouz2024titans, behrouz2025atlas}. Our setting is stricter: the adapted object is a fixed-size fast-memory state, updated online under causal next-latent alignment. This yields normalized first-order and sliding-window rules that remain compatible with chunk-parallel training.

\section{Conclusion}

We recast recurrent sequence modeling as online continual learning with an explicit fast-memory objective. Under read-after-write semantics, the prefix-aligned causal training pair is $\phi(\mathbf{k}_{t-1})\to\mathbf{v}_t$, yielding a unified family of normalized fast-weight updates: \Falcon-1/\Falcon-2/\Falcon-3 for regression and \Falcon-1A/\Falcon-2A/\Falcon-3A for inner-product writes. The number denotes scalar, per-column, and sliding-window dynamics, respectively, while the suffix ``A'' denotes the inner-product objective. This viewpoint separates temporal alignment, plasticity, forgetting, and bounded rehearsal while remaining compatible with chunk-parallel training. In the current empirical study, representative scalar regression and scalar/sliding inner-product variants remain competitive in language modeling, and the best inner-product variants improve arithmetic length extrapolation.

% \section*{Acknowledgement}

% We thank Sanjeev Arora, Tri Dao, Karthik Narasimhan, Eric Song, and Gon Buzaglo for helpful discussions and constructive feedback. We thank Princeton AI Lab and Princeton Language Intelligence for providing computing resources. Steve Ta deeply appreciates Vikram Ramaswamy for believing in him, and Vikram's support helped Steve discover his own path in machine learning. We used Large Language Models to help refine this paper. Their role was limited to improving the clarity and readability of the text.

\vspace{5ex}

\bibliographystyle{plainnat}
\bibliography{reference}

\begin{thebibliography}{41}
\providecommand{\natexlab}[1]{#1}
\providecommand{\url}[1]{\texttt{#1}}
\expandafter\ifx\csname urlstyle\endcsname\relax
  \providecommand{\doi}[1]{doi: #1}\else
  \providecommand{\doi}{doi: \begingroup \urlstyle{rm}\Url}\fi

\bibitem[Ahn et~al.(2023)Ahn, Cheng, Daneshmand, and Sra]{ahn2023transformers}
Kwangjun Ahn, Xiang Cheng, Hadi Daneshmand, and Suvrit Sra.
\newblock Transformers learn to implement preconditioned gradient descent for in-context learning.
\newblock \emph{Advances in Neural Information Processing Systems}, 36:\penalty0 45614--45650, 2023.

\bibitem[Ba et~al.(2016)Ba, Hinton, Mnih, Leibo, and Ionescu]{ba2016using}
Jimmy Ba, Geoffrey~E Hinton, Volodymyr Mnih, Joel~Z Leibo, and Catalin Ionescu.
\newblock Using fast weights to attend to the recent past.
\newblock \emph{Advances in neural information processing systems}, 29, 2016.

\bibitem[Behrouz et~al.(2024)Behrouz, Zhong, and Mirrokni]{behrouz2024titans}
Ali Behrouz, Peilin Zhong, and Vahab Mirrokni.
\newblock Titans: Learning to memorize at test time.
\newblock \emph{arXiv preprint arXiv:2501.00663}, 2024.

\bibitem[Behrouz et~al.(2025)Behrouz, Li, Kacham, Daliri, Deng, Zhong, Razaviyayn, and Mirrokni]{behrouz2025atlas}
Ali Behrouz, Zeman Li, Praneeth Kacham, Majid Daliri, Yuan Deng, Peilin Zhong, Meisam Razaviyayn, and Vahab Mirrokni.
\newblock Atlas: Learning to optimally memorize the context at test time.
\newblock \emph{arXiv preprint arXiv:2505.23735}, 2025.

\bibitem[Bischof and Van~Loan(1987)]{bischof1987wy}
Christian Bischof and Charles Van~Loan.
\newblock The wy representation for products of householder matrices.
\newblock \emph{SIAM Journal on Scientific and Statistical Computing}, 8\penalty0 (1):\penalty0 s2--s13, 1987.

\bibitem[Cheng et~al.(2023)Cheng, Chen, and Sra]{cheng2023transformers}
Xiang Cheng, Yuxin Chen, and Suvrit Sra.
\newblock Transformers implement functional gradient descent to learn non-linear functions in context.
\newblock \emph{arXiv preprint arXiv:2312.06528}, 2023.

\bibitem[Choromanski et~al.(2020)Choromanski, Likhosherstov, Dohan, Song, Gane, Sarlos, Hawkins, Davis, Mohiuddin, Kaiser, et~al.]{choromanski2020rethinking}
Krzysztof Choromanski, Valerii Likhosherstov, David Dohan, Xingyou Song, Andreea Gane, Tamas Sarlos, Peter Hawkins, Jared Davis, Afroz Mohiuddin, Lukasz Kaiser, et~al.
\newblock Rethinking attention with performers.
\newblock \emph{arXiv preprint arXiv:2009.14794}, 2020.

\bibitem[Cirone and Salvi(2025)]{cirone2025parallelflow}
Nicola~Muca Cirone and Cristopher Salvi.
\newblock Parallelflow: Parallelizing linear transformers via flow discretization.
\newblock \emph{arXiv preprint arXiv:2504.00492}, 2025.

\bibitem[Dao and Gu(2024)]{dao2024transformers}
Tri Dao and Albert Gu.
\newblock Transformers are ssms: Generalized models and efficient algorithms through structured state space duality.
\newblock \emph{arXiv preprint arXiv:2405.21060}, 2024.

\bibitem[Fu et~al.(2022)Fu, Dao, Saab, Thomas, Rudra, and R{\'e}]{fu2022hungry}
Daniel~Y Fu, Tri Dao, Khaled~K Saab, Armin~W Thomas, Atri Rudra, and Christopher R{\'e}.
\newblock Hungry hungry hippos: Towards language modeling with state space models.
\newblock \emph{arXiv preprint arXiv:2212.14052}, 2022.

\bibitem[Grazzi et~al.(2024)Grazzi, Siems, Zela, Franke, Hutter, and Pontil]{grazzi2024unlocking}
Riccardo Grazzi, Julien Siems, Arber Zela, J{\"o}rg~KH Franke, Frank Hutter, and Massimiliano Pontil.
\newblock Unlocking state-tracking in linear rnns through negative eigenvalues.
\newblock \emph{arXiv preprint arXiv:2411.12537}, 2024.

\bibitem[Gu and Dao(2023)]{gu2023mamba}
Albert Gu and Tri Dao.
\newblock Mamba: Linear-time sequence modeling with selective state spaces.
\newblock \emph{arXiv preprint arXiv:2312.00752}, 2023.

\bibitem[Gu et~al.(2021)Gu, Goel, and R{\'e}]{gu2021efficiently}
Albert Gu, Karan Goel, and Christopher R{\'e}.
\newblock Efficiently modeling long sequences with structured state spaces.
\newblock \emph{arXiv preprint arXiv:2111.00396}, 2021.

\bibitem[Hinton and Plaut(1987)]{hinton1987using}
Geoffrey~E Hinton and David~C Plaut.
\newblock Using fast weights to deblur old memories.
\newblock In \emph{Proceedings of the ninth annual conference of the Cognitive Science Society}, pages 177--186, 1987.

\bibitem[Ho and K{\'a}lm{\'a}n(1966)]{ho1966effective}
BL~Ho and Rudolf~E K{\'a}lm{\'a}n.
\newblock Effective construction of linear state-variable models from input/output functions: Die konstruktion von linearen modeilen in der darstellung durch zustandsvariable aus den beziehungen f{\"u}r ein-und ausgangsgr{\"o}{\ss}en.
\newblock \emph{at-Automatisierungstechnik}, 14\penalty0 (1-12):\penalty0 545--548, 1966.

\bibitem[Hu et~al.(2025)Hu, Pan, Du, Lan, Tang, Wen, Liang, and Sun]{hu2025comba}
Jiaxi Hu, Yongqi Pan, Jusen Du, Disen Lan, Xiaqiang Tang, Qingsong Wen, Yuxuan Liang, and Weigao Sun.
\newblock Comba: Improving bilinear rnns with closed-loop control.
\newblock \emph{arXiv preprint arXiv:2506.02475}, 2025.

\bibitem[Kaiser and Sutskever(2015)]{kaiser2015neural}
{\L}ukasz Kaiser and Ilya Sutskever.
\newblock Neural gpus learn algorithms.
\newblock \emph{arXiv preprint arXiv:1511.08228}, 2015.

\bibitem[Kalman and Bucy(1961)]{kalman1961new}
Rudolph~E Kalman and Richard~S Bucy.
\newblock New results in linear filtering and prediction theory, 1961.

\bibitem[Kalman(1960)]{kalman1960new}
Rudolph~Emil Kalman.
\newblock A new approach to linear filtering and prediction problems, 1960.

\bibitem[Katharopoulos et~al.(2020)Katharopoulos, Vyas, Pappas, and Fleuret]{katharopoulos2020transformers}
Angelos Katharopoulos, Apoorv Vyas, Nikolaos Pappas, and Fran{\c{c}}ois Fleuret.
\newblock Transformers are rnns: Fast autoregressive transformers with linear attention.
\newblock In \emph{International conference on machine learning}, pages 5156--5165. PMLR, 2020.

\bibitem[Kung(1978)]{kung1978new}
Sun-Yuan Kung.
\newblock A new identification and model reduction algorithm via singular value decomposition.
\newblock In \emph{Proc. 12th asilomar conf. on circuits, systems and computer}, pages 705--714, 1978.

\bibitem[Kung and Lin(1981)]{kung1981state}
Sun-Yuan Kung and D~Lin.
\newblock A state-space formulation for optimal hankel-norm approximations.
\newblock \emph{IEEE Transactions on Automatic Control}, 26\penalty0 (4):\penalty0 942--946, 1981.

\bibitem[Liu et~al.(2024{\natexlab{a}})Liu, Wang, Wu, Feng, Stone, and Liu]{liu2024longhorn}
Bo~Liu, Rui Wang, Lemeng Wu, Yihao Feng, Peter Stone, and Qiang Liu.
\newblock Longhorn: State space models are amortized online learners.
\newblock \emph{arXiv preprint arXiv:2407.14207}, 2024{\natexlab{a}}.

\bibitem[Liu et~al.(2024{\natexlab{b}})Liu, Li, Wang, Wang, Liu, and Li]{liu2024short}
Zicheng Liu, Siyuan Li, Li~Wang, Zedong Wang, Yunfan Liu, and Stan~Z Li.
\newblock Short-long convolutions help hardware-efficient linear attention to focus on long sequences.
\newblock \emph{arXiv preprint arXiv:2406.08128}, 2024{\natexlab{b}}.

\bibitem[Penedo et~al.(2024)Penedo, Kydl{\'\i}{\v{c}}ek, Lozhkov, Mitchell, Raffel, Von~Werra, Wolf, et~al.]{penedo2024fineweb}
Guilherme Penedo, Hynek Kydl{\'\i}{\v{c}}ek, Anton Lozhkov, Margaret Mitchell, Colin Raffel, Leandro Von~Werra, Thomas Wolf, et~al.
\newblock The fineweb datasets: Decanting the web for the finest text data at scale.
\newblock \emph{Advances in Neural Information Processing Systems}, 37:\penalty0 30811--30849, 2024.

\bibitem[Peng et~al.(2023)Peng, Alcaide, Anthony, Albalak, Arcadinho, Biderman, Cao, Cheng, Chung, Grella, et~al.]{peng2023rwkv}
Bo~Peng, Eric Alcaide, Quentin Anthony, Alon Albalak, Samuel Arcadinho, Stella Biderman, Huanqi Cao, Xin Cheng, Michael Chung, Matteo Grella, et~al.
\newblock Rwkv: Reinventing rnns for the transformer era.
\newblock \emph{arXiv preprint arXiv:2305.13048}, 2023.

\bibitem[Poli et~al.(2023)Poli, Massaroli, Nguyen, Fu, Dao, Baccus, Bengio, Ermon, and R{\'e}]{poli2023hyena}
Michael Poli, Stefano Massaroli, Eric Nguyen, Daniel~Y Fu, Tri Dao, Stephen Baccus, Yoshua Bengio, Stefano Ermon, and Christopher R{\'e}.
\newblock Hyena hierarchy: Towards larger convolutional language models.
\newblock In \emph{International Conference on Machine Learning}, pages 28043--28078. PMLR, 2023.

\bibitem[Qin et~al.(2024)Qin, Sun, Li, Shen, Sun, and Zhong]{qin2024lightning}
Zhen Qin, Weigao Sun, Dong Li, Xuyang Shen, Weixuan Sun, and Yiran Zhong.
\newblock Lightning attention-2: A free lunch for handling unlimited sequence lengths in large language models.
\newblock \emph{arXiv preprint arXiv:2401.04658}, 2024.

\bibitem[Sayed(2011)]{sayed2011adaptive}
Ali~H Sayed.
\newblock \emph{Adaptive filters}.
\newblock John Wiley \& Sons, 2011.

\bibitem[Schlag et~al.(2021)Schlag, Irie, and Schmidhuber]{schlag2021linear}
Imanol Schlag, Kazuki Irie, and J{\"u}rgen Schmidhuber.
\newblock Linear transformers are secretly fast weight programmers.
\newblock In \emph{International conference on machine learning}, pages 9355--9366. PMLR, 2021.

\bibitem[Schmidhuber(1992)]{schmidhuber1992learning}
J{\"u}rgen Schmidhuber.
\newblock Learning to control fast-weight memories: An alternative to dynamic recurrent networks.
\newblock \emph{Neural Computation}, 4\penalty0 (1):\penalty0 131--139, 1992.

\bibitem[Sun et~al.(2024)Sun, Li, Dalal, Xu, Vikram, Zhang, Dubois, Chen, Wang, Koyejo, et~al.]{sun2024learning}
Yu~Sun, Xinhao Li, Karan Dalal, Jiarui Xu, Arjun Vikram, Genghan Zhang, Yann Dubois, Xinlei Chen, Xiaolong Wang, Sanmi Koyejo, et~al.
\newblock Learning to (learn at test time): Rnns with expressive hidden states.
\newblock \emph{arXiv preprint arXiv:2407.04620}, 2024.

\bibitem[Sun et~al.(2023)Sun, Dong, Huang, Ma, Xia, Xue, Wang, and Wei]{sun2023retentive}
Yutao Sun, Li~Dong, Shaohan Huang, Shuming Ma, Yuqing Xia, Jilong Xue, Jianyong Wang, and Furu Wei.
\newblock Retentive network: A successor to transformer for large language models.
\newblock \emph{arXiv preprint arXiv:2307.08621}, 2023.

\bibitem[Team et~al.(2025)Team, Zhang, Lin, Yao, Hu, Meng, Liu, Men, Yang, Li, et~al.]{team2025kimilinear}
Kimi Team, Yu~Zhang, Zongyu Lin, Xingcheng Yao, Jiaxi Hu, Fanqing Meng, Chengyin Liu, Xin Men, Songlin Yang, Zhiyuan Li, et~al.
\newblock Kimi linear: An expressive, efficient attention architecture.
\newblock \emph{arXiv preprint arXiv:2510.26692}, 2025.

\bibitem[Vaswani et~al.(2017)Vaswani, Shazeer, Parmar, Uszkoreit, Jones, Gomez, Kaiser, and Polosukhin]{vaswani2017attention}
Ashish Vaswani, Noam Shazeer, Niki Parmar, Jakob Uszkoreit, Llion Jones, Aidan~N Gomez, {\L}ukasz Kaiser, and Illia Polosukhin.
\newblock Attention is all you need.
\newblock \emph{Advances in neural information processing systems}, 30, 2017.

\bibitem[Von~Oswald et~al.(2023)Von~Oswald, Niklasson, Randazzo, Sacramento, Mordvintsev, Zhmoginov, and Vladymyrov]{von2023transformers}
Johannes Von~Oswald, Eyvind Niklasson, Ettore Randazzo, Jo{\~a}o Sacramento, Alexander Mordvintsev, Andrey Zhmoginov, and Max Vladymyrov.
\newblock Transformers learn in-context by gradient descent.
\newblock In \emph{International Conference on Machine Learning}, pages 35151--35174. PMLR, 2023.

\bibitem[von Oswald et~al.(2025)von Oswald, Scherrer, Kobayashi, Versari, Yang, Schlegel, Maile, Schimpf, Sieberling, Meulemans, et~al.]{von2025mesanet}
Johannes von Oswald, Nino Scherrer, Seijin Kobayashi, Luca Versari, Songlin Yang, Maximilian Schlegel, Kaitlin Maile, Yanick Schimpf, Oliver Sieberling, Alexander Meulemans, et~al.
\newblock Mesanet: Sequence modeling by locally optimal test-time training.
\newblock \emph{arXiv preprint arXiv:2506.05233}, 2025.

\bibitem[Wang et~al.(2025)Wang, Shi, and Fox]{wang2025test}
Ke~Alexander Wang, Jiaxin Shi, and Emily~B Fox.
\newblock Test-time regression: a unifying framework for designing sequence models with associative memory.
\newblock \emph{arXiv preprint arXiv:2501.12352}, 2025.

\bibitem[Yang et~al.(2023)Yang, Wang, Shen, Panda, and Kim]{yang2023gated}
Songlin Yang, Bailin Wang, Yikang Shen, Rameswar Panda, and Yoon Kim.
\newblock Gated linear attention transformers with hardware-efficient training.
\newblock \emph{arXiv preprint arXiv:2312.06635}, 2023.

\bibitem[Yang et~al.(2024{\natexlab{a}})Yang, Kautz, and Hatamizadeh]{yang2024gated}
Songlin Yang, Jan Kautz, and Ali Hatamizadeh.
\newblock Gated delta networks: Improving mamba2 with delta rule.
\newblock \emph{arXiv preprint arXiv:2412.06464}, 2024{\natexlab{a}}.

\bibitem[Yang et~al.(2024{\natexlab{b}})Yang, Wang, Zhang, Shen, and Kim]{yang2024parallelizing}
Songlin Yang, Bailin Wang, Yu~Zhang, Yikang Shen, and Yoon Kim.
\newblock Parallelizing linear transformers with the delta rule over sequence length.
\newblock \emph{arXiv preprint arXiv:2406.06484}, 2024{\natexlab{b}}.

\end{thebibliography}

\clearpage
\appendix

\renewcommand{\appendixpagename}{\centering \huge Appendix}
\appendixpage
\counterwithin{theorem}{section}

\startcontents[section]
\printcontents[section]{l}{1}{\setcounter{tocdepth}{2}}
\clearpage

\section{More on Background}

\noindent\textbf{Notation.}
We reuse the main-text regression notation: $\xb_t$ is the write feature, typically $\phi(\mathbf{k}_{t-1})$ under next-latent alignment; $\yb_t$ is the target, typically $\vb_t$; and $\Sbb_t$ is the fast-weight state. The residual is $\mathbf{r}_t:=\yb_t-\Sbb_{t-1}^\top\xb_t$.

\subsection{Recursive Least Squares}
Recursive least squares (RLS) realizes the exact cumulative ridge-regression solution online via the matrix inversion lemma. We record the constant-ridge, zero-prior-mean case $\lambda>0$ with $\Sbb_0=\mathbf{0}$. This is a reference solver, not the exact closed-form solution for the time-varying $\lambda_t$ recurrences used elsewhere in the paper.
Let
\begin{align*}
\Pb_{t-1} := \big(\lambda \Ib + \textstyle\sum_{s=1}^{t-1}\xb_s\xb_s^\top\big)^{-1},
\qquad \Pb_0:=\lambda^{-1}\Ib \ \ (\lambda>0).
\end{align*}
Define the gain and covariance updates
\begin{equation}
\label{eq:rls-gain}
\mathbf{g}_t=\frac{\Pb_{t-1}\xb_t}{1+\xb_t^\top \Pb_{t-1}\xb_t},
\qquad
\Pb_t=\Pb_{t-1}-\mathbf{g}_t \xb_t^\top \Pb_{t-1}.
\end{equation}
Then the parameter update is
\begin{equation}
\label{eq:rls-param}
\mathbf{r}_t := \yb_t-\Sbb_{t-1}^\top\xb_t,\qquad
\Sbb_t = \Sbb_{t-1} + \mathbf{g}_t\,\mathbf{r}_t^\top.
\end{equation}
These recursions cost $O(d_x^2 + d_x d_v)$ per step, where $d_x:=\dim(\xb_t)$, and satisfy that $\Sbb_t$ is the exact ridge-regression solution after processing $t$ samples, i.e. the minimizer of the cumulative ridge objective on $\{(\xb_s,\yb_s)\}_{s=1}^{t}$ with zero prior mean. A nonzero initial state corresponds instead to a prior-centered ridge objective. Exact online solvers of this form are much harder to parallelize than first-order recurrences, which is why MesaNet~\citep{von2025mesanet} uses an approximate chunk-parallel approach.

\subsection{First-Order Online Ridge}
\label{app:ridge_reference}

Algorithm~\ref{alg:ridge-sgd} records the reference sequential first-order update used throughout the main text.

\begin{algorithm}[t]
\caption{Online ridge via SGD with batch size $1$}
\label{alg:ridge-sgd}
\begin{algorithmic}[1]
\Require Keys $\{\mathbf{k}_t\}_{t=1}^{T}$, feature map $\phi$ (default identity), values $\{\vb_t\}_{t=1}^{T}$, boundary write feature $\mathbf{x}_1:=\mathbf{0}$ and $\mathbf{x}_t:=\phi(\mathbf{k}_{t-1})$ for $t\ge 2$, ridge coefficients $\{\lambda_t\}_{t=1}^{T}$ with $\lambda_t\ge 0$, gain $\beta_t\in(0,2)$, stabilizer $\varepsilon\ge 0$ (set $\eta_t{:=}0$ if the denominator is $0$), init $\Sbb_0=\mathbf{0}$.
\For{$t=1, \ldots, T$}
\State $\mathbf{x}_t \gets \mathbf{0}$ if $t=1$, else $\phi(\mathbf{k}_{t-1})$ \Comment{Write feature}
\State $\mathbf{y}_t \gets \vb_{t}$ \Comment{Target}
\State $\widehat{\mathbf{y}}_{t} \gets \Sbb_{t-1}^\top \mathbf{x}_{t}$ \Comment{Fast-memory prediction}
\State $\mathbf{r}_{t} \gets \mathbf{y}_{t} - \widehat{\mathbf{y}}_{t}$ \Comment{Residual}
\State $\eta_t \gets 0$ if $t=1$ or $\|\mathbf{x}_{t}\|_2^2+\lambda_t+\varepsilon=0$, else $\beta_t/(\|\mathbf{x}_{t}\|_2^2+\lambda_t+\varepsilon)$
\State $\Sbb_{t} \gets (1 - \eta_t \lambda_t)\Sbb_{t-1} + \eta_{t}\,\mathbf{x}_{t}\mathbf{r}_{t}^\top$
\EndFor
\end{algorithmic}
\end{algorithm}

\subsection{SSM Discretization}
\label{app:ssm_discretization}

Over step size $\Delta_t$, the exact zero-order-hold discretization of the continuous-time linear system is
\begin{equation}
\hb_t = \bar{\Ab}_t \hb_{t-1} + \bar{\Bb}_t x_t,
\qquad
\bar{\Ab}_t = e^{\Delta_t \Ab_t},
\qquad
\bar{\Bb}_t = \int_{0}^{\Delta_t} e^{(\Delta_t-s)\Ab_t}\Bb_t\,ds.
\end{equation}
When $\Ab_t$ is diagonal, the input integral admits the elementwise closed form
\[
\bar{\Bb}_t
= \left(\Ab_t^{-1}(e^{\Delta_t \Ab_t}-\Ib)\right)\Bb_t,
\]
where the factor $\Ab_t^{-1}(e^{\Delta_t \Ab_t}-\Ib)$ is interpreted elementwise and uses the limit $(e^{\Delta a}-1)/a\to\Delta$ as $a\to0$. A first-order approximation is
\[
\bar{\Bb}_t\approx \Delta_t\Bb_t,
\qquad
\hb_t\approx e^{\Delta_t \Ab_t}\hb_{t-1}+\Delta_t\Bb_t x_t,
\]
valid when $\|\Delta_t\Ab_t\|$ is small. 

\section{Experimental Setup}
\label{app:exp_setup}

\subsection{Language Model Training Setup}
For the language model runs in Section~\ref{sec:results}, all models are trained in bfloat16 with AdamW, tied input/output embeddings, Pre-Norm RMSNorm, no bias terms, and no dropout. We use $\mu$P-style width scaling, base learning rate $10^{-3}$ with cosine decay, 2{,}000 warmup steps, $(\beta_1,\beta_2)=(0.9,0.95)$, weight decay 0.1, and gradient clipping at 1.0. Each run is performed on a single 4-GPU node with NVIDIA H100 or H200 GPUs.

\section{Implementation Details}
\label{app:general_impl_positioning}

\subsection{Scaled Fast Weight Recurrences}
\label{app:scaled_fast_weight_details}

The normalized counterpart of the denominator-free scaled linear-attention recurrence in Section~\ref{sec:scaled_fast_weight} is
\[
\ob_t=\frac{\Sbb_t^\top \phi(\qb_t)}{\zb_t^\top \phi(\qb_t)+\varepsilon_{\rm attn}},
\qquad
\Sbb_t=\gamma_t\Sbb_{t-1}+\eta_t\,\mathbf{x}_t\vb_t^\top,
\qquad
\zb_t=\gamma_t\zb_{t-1}+\eta_t\,\mathbf{x}_t,
\]
\[
\alpha_t:=\min(\eta_t\lambda_t,1-\varepsilon_\gamma),
\qquad
\gamma_t:=1-\alpha_t.
\]
When the clamp is inactive, this is exactly the update with carry $1-\eta_t\lambda_t$. If the clamp activates, the shrinkage path is instead the positive-decay surrogate with effective coefficient $\alpha_t/\eta_t$ for $\eta_t>0$. With $\zb_0\ge 0$ and elementwise nonnegative feature maps, this normalized form remains attention-like only when $\gamma_t\ge 0$ for all $t$; the clamp enforces $\gamma_t\ge\varepsilon_\gamma>0$. The additive normalized case in Eq.~\eqref{eq:lin_attn_rec_shifted} is recovered by $\eta_t\equiv 1$ and $\lambda_t\equiv 0$.

For scaled non-sliding variants, let
\[
\bar{\lambda}_t := \lambda_{\text{scale}}\,\sigma(\tilde{\lambda}_t),
\qquad
\lambda_t^{\text{eff}} := \bar{\lambda}_t\,\|\mathbf{x}_t\|_2^2.
\]
Then
\[
\eta_t=\frac{\beta_t}{\|\mathbf{x}_t\|_2^2+\lambda_t^{\text{eff}}+\varepsilon},
\qquad
\gamma_t=1-\eta_t\lambda_t^{\text{eff}}.
\]
When $\varepsilon=0$ and $\|\mathbf{x}_t\|_2>0$, both the scalar decay fraction and the normalized rank-one edit term are invariant to uniform rescaling of $\mathbf{x}_t$. Sliding regression variants replace $\|\mathbf{x}_t\|_2^2$ with $\mu_t^{(B)}:=\lambda_{\max}(\bar{\Cb}_t^{(B)})$; sliding inner-product variants use the write-energy statistic $\bar E_t^{(B)}$.

\subsection{Signed-Feature Normalizers}
\label{app:normalized_linattn_caveat}

The normalized linear-attention read
\[
\ob_t=\frac{\Sbb_t^\top\phi(\qb_t)}{\zb_t^\top\phi(\qb_t)+\varepsilon_{\rm attn}}
\]
is mathematically safest when the read denominator is nonnegative, as with elementwise nonnegative feature maps satisfying $\phi(\cdot)\ge0$ and $\zb_t\ge0$. In signed-feature settings, adding a positive constant does not by itself prevent the denominator from vanishing or changing sign:
\[
\zb_t^\top\phi(\qb_t)+\varepsilon_{\rm attn}
\]
can still be zero or negative. Thus denominator-free inner-product reads are the default signed-feature interpretation in this paper. If a normalized signed-feature read is desired, the implementation must use an explicit safe denominator, clamp, or other sign-stable normalization.

\subsection{Step Size and Gating Conventions}
\label{app:deltanet_step_and_gating}

Prior DeltaNet literature often writes the raw gradient step size as $\beta_t$. We reserve $\eta_t$ for the actual step size and use $\beta_t$ for the dimensionless NLMS gain. This convention separates the learned plasticity control from the normalization statistic:
\[
\eta_t=\frac{\beta_t}{\|\mathbf{x}_t\|_2^2+\lambda_t+\varepsilon}
\]
for the rank-one regression rule, and analogously for the windowed and inner-product variants.

The standard Delta update
\[
\Sbb_t=(\Ib-\eta_t\mathbf{k}_t\mathbf{k}_t^\top)\Sbb_{t-1}
+\eta_t\mathbf{k}_t\vb_t^\top
\]
contains only implicit forgetting through the rank-one edit. Gated Delta Networks~\citep{yang2024gated} add an explicit scalar carry:
\begin{equation}
\Sbb_t
=g_t(\Ib-\eta_t\mathbf{k}_t\mathbf{k}_t^\top)\Sbb_{t-1}
+\eta_t\mathbf{k}_t\vb_t^\top,
\qquad
g_t\in[0,1].
\end{equation}
We write this gate as $g_t$ to avoid overloading the Falcon notation, where $\alpha_t:=\eta_t\lambda_t$ is a derived decay fraction and $\gamma_t:=1-\alpha_t$ is the carry. The gated Delta rule and Falcon-style ridge shrinkage both provide global state decay; the main distinction is that Falcon ties the carry to the local objective and the normalized step-size parameterization.

\subsection{Relation to Test-Time Training and Titans}
\label{app:ttt_titans_relation}

Test-Time Training (TTT) updates a subset of model parameters at inference by taking gradient steps on an internal, self-supervised objective, and recent work formalizes this mechanism as \emph{test-time regression} in a feature space~\citep{sun2024learning, wang2025test}. Our fast-weight update is a constrained, single-step instance of this paradigm: the recurrent state $\Sbb$ is the adapted parameter and Eq.~\eqref{eq:inst-loss} is the inner objective. Relative to generic TTT, we enforce strict autoregressive causality by updating only after $\vb_t$ is revealed and by writing it under the prefix write feature available at prediction time, $\mathbf{x}_t=\phi(\mathbf{k}_{t-1})$. Titans and ATLAS similarly view the recurrent state as fast memory optimized online with an internal objective over the stream~\citep{behrouz2024titans, behrouz2025atlas}; Falcon-3 can be read as a sliding-window specialization of that view. Our contribution is to make the write-feature/target alignment explicit and to derive normalized first-order rules and sliding-window variants compatible with SSD-style chunk-parallel training.

\subsection{Timing and Boundary Conventions}
\label{app:timing_boundary}

We use the read-after-write (RAW) indexing convention throughout the paper: after token $t$ is observed and written, the updated state $\Sbb_t$ is read to predict token $t{+}1$. This is the standard Transformer indexing shifted by one step relative to the read-before-write (RBW) view that predicts token $t$ from prefix $1{:}t{-}1$. Under next-latent alignment, the RAW update therefore uses the causal pair $(\phi(\mathbf{k}_{t-1}),\vb_t)$, or equivalently $(\phi(\mathbf{k}_i),\vb_{i+1})$ under standard indexing. Figure~\ref{fig:timing_ablation} visualizes the RAW/RBW timing conventions used throughout the paper.

The four cells in Fig.~\ref{fig:timing_ablation} correspond to different local-objective conventions. Under RAW, the shifted pairing $(\phi(\mathbf{k}_{t-1}),\vb_t)$ is aligned with the prefix feature that was available when $\vb_t$ was predicted. The unshifted RAW pairing $(\phi(\mathbf{k}_{t}),\vb_t)$ is still causal for next-token prediction, but it optimizes a different local objective based on same-step features. Likewise, the shifted RBW pairing is internally consistent, but it updates memory after the read that produced the position-$t$ prediction.

We initialize $\Sbb_0=\mathbf{0}$ and impose the feature-space boundary condition $\mathbf{x}_1:=\mathbf{0}$. In the identity-feature case this can be realized by a zero raw \textsc{BOS} key; for a generic feature map the condition is imposed directly in feature space. Since $\mathbf{x}_1=\mathbf{0}$ suppresses only the data write, not necessarily ridge shrinkage, we also set $\eta_1:=0$ at this boundary sentinel. Equivalently, $(\alpha_1,\gamma_1)=(0,1)$ in the derived-variable notation used for log-space unrolling. This convention is essential when a nonzero state is carried across segments.

With this convention, the internal fast-memory prediction at step $t$ is $\hat{\mathbf{y}}_t:=\Sbb_{t-1}^\top \mathbf{x}_t$, whereas the model readout at position $t$ uses the updated state $\Sbb_t$. The two quantities should not be conflated.

% =========================================================
% FIGURE: Formal Timing Semantics + 2x2 Ablation
% =========================================================
\begin{figure}[htbp!]
\centering
\resizebox{\linewidth}{!}{%
\begin{tikzpicture}[>=Latex, font=\small]

% ---------------------------------------------------------
% PANEL: Timing semantics (left)  -- balanced / explicit boxes
% ---------------------------------------------------------
\begin{scope}[local bounding box=timingPanel]

% Section label
\node[section_label] (timingTitle) at (7.2,3.95)
{\textbf{Autoregressive Fast Weight Updates}};

% =========================================================
% RBW row  (explicitly show READ and prediction box)
% =========================================================
\node[base_node, draw=appleGray!70, fill=appleLightGray, rounded corners=8pt,
      minimum width=1.8cm, font=\bfseries] (rbwTag) at (-0.2,2.35) {RBW};

\node[base_node, draw=appleGray!60, fill=white, minimum width=2.5cm] (rbwPrefix) at (2.1,2.35)
{Prefix $1{:}t-1$};

\node[state, scale=0.9, minimum width=2.3cm] (rbwRead) at (4.9,2.35)
{READ $\mathbf{S}_{t-1}$};

\node[base_node, dashed, draw=appleGray!70, fill=white, minimum width=2.2cm] (rbwPred) at (7.7,2.35)
{predict $t$};

\node[base_node, draw=appleRed!45, fill=appleRed!6, rounded corners=8pt,
      minimum width=2.6cm, font=\small] (rbwObs) at (10.7,2.35)
{observe $t$ \\ target $\mathbf{v}_t$};

\node[state, scale=0.9, minimum width=2.3cm] (rbwWrite) at (13.8,2.35)
{WRITE $\mathbf{S}_t$};

\draw[arrow] (rbwPrefix) -- (rbwRead);
\draw[arrow] (rbwRead) -- (rbwPred);
\draw[arrow] (rbwPred) -- (rbwObs);
\draw[arrow] (rbwObs) -- (rbwWrite);

% Pairing choices for RBW write
\node[tensor, scale=0.75] (rbwKshift) at (12.2,3.15) {$\mathbf{k}_{t-1}$};
\node[tensor, scale=0.75] (rbwKunsh)  at (12.2,1.55) {$\mathbf{k}_{t}$};

\draw[signal_arrow] (rbwKshift.east) -- ++(0.35,0) -| (rbwWrite.north);
\draw[signal_arrow] (rbwKunsh.east)  -- ++(0.35,0) -| (rbwWrite.south);

\node[font=\scriptsize, text=appleGray, anchor=west] at (15.35,2.95)
{shifted: $(\phi(\mathbf{k}_{t-1}),\mathbf{v}_t)$};
\node[font=\scriptsize, text=appleGray, anchor=west] at (15.35,1.75)
{unshifted: $(\phi(\mathbf{k}_{t}),\mathbf{v}_t)$};

% =========================================================
% RAW row (ours)  -- same number of explicit boxes
% =========================================================
\node[base_node, draw=appleBlue!80, fill=appleLightBlue, rounded corners=8pt,
      minimum width=1.8cm, font=\bfseries, text=appleBlue] (rawTag) at (-0.2,-0.35) {RAW};

% Left member of causal pair
\node[tensor, scale=0.9] (rawK) at (2.1,-0.35) {$\mathbf{k}_{t-1}$};

% Right member of causal pair (observation/target becomes available)
\node[base_node, draw=appleRed!55, fill=appleRed!8, rounded corners=8pt,
      minimum width=2.6cm, font=\small, align=center] (rawObs) at (4.9,-0.35)
{observe $t$ \\ target $\mathbf{v}_t$};

% Write / Read / Predict
\node[state, draw=appleBlue, fill=appleLightBlue, scale=0.9, minimum width=2.3cm] (rawWrite) at (8.1,-0.35)
{WRITE $\mathbf{S}_t$};

\node[state, draw=appleBlue, fill=appleLightBlue, scale=0.9, minimum width=2.2cm] (rawRead) at (10.9,-0.35)
{READ $\mathbf{S}_t$};

\node[base_node, dashed, draw=appleGray, fill=white, minimum width=2.4cm] (rawPred) at (13.8,-0.35)
{predict $t+1$};

% Arrows
\draw[arrow] (rawK) -- (rawObs);
\draw[arrow] (rawObs) -- (rawWrite);
\draw[arrow] (rawWrite) -- (rawRead);
\draw[arrow] (rawRead) -- (rawPred);

% Variable-role labels (put ABOVE the individual boxes)
\node[font=\bfseries\scriptsize, text=appleBlue, anchor=south] at ([yshift=2pt]rawK.north)
{$\mathbf{x}_t$};

\node[font=\bfseries\scriptsize, text=appleRed, anchor=south] at ([yshift=1pt]rawObs.north)
{$\mathbf{y}_t$};

% Optional: reinforce causal signal on the write edges
\draw[signal_arrow] (rawK.east) -- ++(0.25,0);
\draw[signal_arrow, draw=appleRed!70] (rawObs.east) -- ++(0.25,0);

% Dashed box enclosing the causal pair
\node[draw=appleBlue!70, dashed, rounded corners=8pt, inner sep=12pt, fit=(rawK)(rawObs)] (causalPair) {};

% Dashed-box title (relationship label, not a variable label)
\node[font=\bfseries\scriptsize, text=appleBlue, anchor=south] at ([yshift=6pt]causalPair.north)
{causal pair $(\mathbf{x}_t,\mathbf{y}_t)$ at step $t$};

% Timing note
\node[base_node, fill=white, draw=appleGray!50, rounded corners=8pt, inner sep=10pt, align=left]
(timingNote) at (7.7,-2.7) {%
$
\begin{aligned}
\textbf{Shifted (ours):}\;& \mathbf{x}_t=\phi(\mathbf{k}_{t-1}),\;\mathbf{y}_t=\mathbf{v}_t\\
\textbf{Unshifted:}\;& \mathbf{x}_t=\phi(\mathbf{k}_{t}),\;\mathbf{y}_t=\mathbf{v}_t
\end{aligned}
$
};

% Timing note
\node[base_node, fill=white, draw=appleGray!50, rounded corners=8pt, inner sep=10pt, align=left]
(capturedNode) at (0.7,-2.2) {%
$
\begin{aligned}
\text{Captured prefix }1:t-1 
\end{aligned}
$
};

\draw[arrow] (capturedNode) -- (rawK);

% Legend-ish note
\node[font=\scriptsize, text=appleGray, align=left, anchor=west] at (12.0,-2.7) {%
\textbf{RBW}: read then write\\
\textbf{RAW}: write then read (used in \Falcon)
};

\end{scope}

% ---------------------------------------------------------
% PANEL: 2x2 ablation (right)
% ---------------------------------------------------------
\begin{scope}[shift={(5.0,-9.2)}, local bounding box=ablationPanel]

\node[section_label] (ablTitle) at (3.9,3.45)
{\textbf{Pairings Ablation}};

% Outer box
\node[base_node, draw=appleGray!50, fill=white, rounded corners=10pt,
      minimum width=9.2cm, minimum height=6.2cm] (ablOuter) at (3.9,0.1) {};

% Column headers
\node[font=\bfseries\small, text=appleGray] at (3,2.2) {RBW};
\node[font=\bfseries\small, text=appleBlue] at (6.1,2.2) {RAW};

% Row headers
\node[font=\bfseries\scriptsize, text=appleGray, align=center] at (0.7,1.0)
{$\phi(\mathbf{k}_t)$ (unshifted)};
\node[font=\bfseries\scriptsize, text=appleBlue, align=center] at (0.7,-1.0)
{$\phi(\mathbf{k}_{t-1})$ (shifted)};

% Top-left cell
\node[base_node, draw=appleGray!50, fill=appleLightGray!40, rounded corners=8pt,
      minimum width=2.6cm, minimum height=1.4cm] (c11) at (3,1.0) {};
\node[font=\bfseries\scriptsize, text=appleGray] at (c11.center) {Run 1 (Original)};

% Top-right cell
\node[base_node, draw=appleGray!50, fill=appleLightGray!20, rounded corners=8pt,
      minimum width=2.6cm, minimum height=1.4cm] (c12) at (6.1,1.0) {};
\node[font=\bfseries\scriptsize, text=appleGray] at (c12.center) {Run 2};

% Bottom-left cell
\node[base_node, draw=appleGray!50, fill=appleLightGray!20, rounded corners=8pt,
      minimum width=2.6cm, minimum height=1.4cm] (c21) at (3,-1.0) {};
\node[font=\bfseries\scriptsize, text=appleGray] at (c21.center) {Run 3};

% Bottom-right cell (ours)
\node[base_node, draw=appleBlue!70, fill=appleLightBlue, rounded corners=8pt,
      minimum width=2.6cm, minimum height=1.4cm] (c22) at (6.1,-1.0) {};
\node[font=\bfseries\scriptsize, text=appleBlue] at (c22.center) {Run 4 (Ours)};

% Small arrows / outcome note
\draw[arrow, appleGray!60] (c11.south) -- ++(0,-0.35);
\draw[arrow, appleGray!60] (c12.south) -- ++(0,-0.35);
\draw[arrow, appleGray!60] (c21.south) -- ++(0,-0.35);
\draw[arrow, appleBlue!70] (c22.south) -- ++(0,-0.35);
\end{scope}

% ---------------------------------------------------------
% Backgrounds (same style)
% ---------------------------------------------------------
\begin{scope}[on background layer]
\node[draw=appleGray!30, left color=appleLightGray!30, right color=appleLightGray,
      rounded corners=16pt, fit=(timingPanel), inner sep=18pt] (timingBG) {};
\node[draw=appleGray!30, left color=appleLightGray!30, right color=appleLightGray,
      rounded corners=16pt, fit=(ablationPanel), inner sep=18pt] (ablBG) {};
\end{scope}

\end{tikzpicture}
} % end resizebox
\caption{\textbf{Formal timing conventions}.
Above: We explicitly distinguish \emph{read-before-write} (RBW) and \emph{read-after-write} (RAW) autoregressive fast-weight updates. Under the RAW convention used in \Falcon, the causally aligned training pair at step $t$ is $(\phi(\mathbf{k}_{t-1}), \mathbf{v}_t)$, i.e., a prefix write feature paired with the newly revealed target. Below: A 2$\times$2 design over timing convention $\{\text{RBW},\text{RAW}\}$ and write pairing \{same-step, shifted\} separates causal timing from mere index shifting; in the kernelized setting these pairings are understood in write-feature space, i.e.\ $\{\phi(\mathbf{k}_t),\phi(\mathbf{k}_{t-1})\}$.}
\label{fig:timing_ablation}
\end{figure}

\subsection{Implementation Notes and Related Update Rules}
\label{app:extra_impl_remarks}

\paragraph{Gain and step size.}
We use $\beta_t\in(0,2)$ for the dimensionless NLMS gain and $\eta_t$ for the resulting normalized step size. Some prior DeltaNet papers use $\beta_t$ for the raw step size; we keep the two quantities distinct. In the unclamped $\varepsilon=0$ analysis, $\beta_t>1$ induces a negative eigenvalue along the current write-feature direction, i.e., a stable sign flip that can improve state tracking~\citep{grazzi2024unlocking}.

\paragraph{Ridge as decay.}
For $\lambda_t>0$, the ridge term contributes the scalar carry $\gamma_t:=1-\eta_t\lambda_t$ in Eq.~\eqref{eq:ogd}. The full linear transition is $\Ab_t=\gamma_t\Ib_{d_x}-\eta_t\mathbf{x}_t\mathbf{x}_t^\top$, so the subspace orthogonal to $\mathbf{x}_t$ sees eigenvalue $\gamma_t$, while the write-feature direction sees eigenvalue $1-\eta_t(\|\mathbf{x}_t\|_2^2+\lambda_t)$. With $\varepsilon=0$ and Eq.~\eqref{eq:nlms-stepsize}, this directional eigenvalue equals $1-\beta_t$. When log-space unrolling requires positive decay, we clamp the derived fraction $\alpha_t:=\eta_t\lambda_t$; if the clamp activates, the implemented shrinkage may be interpreted as using an effective coefficient $\tilde\lambda_t:=\alpha_t/\eta_t\le\lambda_t$ whenever $\eta_t>0$ (and $0$ when $\eta_t=0$).

\paragraph{Log-space chunk-local renormalization.}
In long-context settings, direct products of positive decay factors can become numerically fragile in mixed precision. For scalar-decay recurrences of the form
\[
\Sbb_t
=
\big(\gamma_t\Ib-\eta_t\mathbf{x}_t\mathbf{x}_t^\top\big)\Sbb_{t-1}
\;+\;
\eta_t\,\mathbf{x}_t\mathbf{y}_t^\top,
\qquad
\gamma_t>0,
\]
with $\mathbf{y}_t$ the write target (e.g.\ $\mathbf{v}_t$), the unclamped carry is $\gamma_t=1-\eta_t\lambda_t$. Whenever positivity is not guaranteed a priori, we instead define
\[
\alpha_t:=\eta_t\lambda_t,
\qquad
\alpha_t\leftarrow \min(\alpha_t,1-\varepsilon_\gamma),
\qquad
\gamma_t:=1-\alpha_t,
\]
and compute
\[
\log\gamma_t:=\operatorname{log1p}(-\alpha_t)
\]
in fp32. Let
\[
c_0:=1,\qquad c_t:=\prod_{r=1}^{t}\gamma_r,\qquad \tilde{\Sbb}_t:=\Sbb_t/c_t.
\]
Then
\[
\tilde{\Sbb}_t
=
\big(\Ib-\tilde\eta_t\,\mathbf{x}_t\mathbf{x}_t^\top\big)\tilde{\Sbb}_{t-1}
\;+\;
\tilde\eta_t\,\mathbf{x}_t\tilde{\mathbf{y}}_t^\top,
\qquad
\tilde\eta_t:=\frac{\eta_t}{\gamma_t},
\qquad
\tilde{\mathbf{y}}_t:=\frac{\mathbf{y}_t}{c_{t-1}}.
\]
Hence exact equivalence requires \emph{both} step-size rescaling and inverse-decay rescaling of the write target; rescaling $\eta_t$ alone is not sufficient. In chunk-parallel implementations, we do not form global products $c_t$; instead, we accumulate $\log\gamma_t$ within each chunk and apply the equivalent local renormalization to boundary states and write targets. For normalized readouts of the form
\[
\ob_t=\frac{\Sbb_t^\top \phi(\qb_t)}{\zb_t^\top \phi(\qb_t)+\varepsilon_{\rm attn}},
\]
exact output preservation additionally requires forming the ratio after undoing the common rescaling, or equivalently, rescaling the stabilizer by the same factor. With a fixed $\varepsilon_{\rm attn}>0$ and directly using the renormalized states inside the ratio, the equivalence is only approximate. The denominator-free recurrence is unaffected by this caveat.

\paragraph{Optional local convolution.}
In some implementations, we apply a short causal convolution before forming $(\qb,\mathbf{k},\vb)$, following prior observations that local mixing can improve associative recall~\citep{fu2022hungry, liu2024short}. This choice is orthogonal to the fast-memory update itself and is not required by the theory.

\paragraph{Identity-matrix notation.}
We write $\Ib_{d_x}$ for the $d_x\times d_x$ identity in feature space, where $d_x\triangleq \dim(\mathbf{x}_t)$ (so $d_x=d$ for raw keys and $d_x=m$ for kernelized write-features), and $\Ib_{d_v}$ for the $d_v\times d_v$ identity in value space. When the dimension is unambiguous, we write $\Ib$.

\paragraph{Scale-coupled ridge.}
Appendix~\ref{app:scaled_fast_weight_details} gives the exact non-sliding scale-invariance identities. Sliding regression uses the smoothness statistic $\mu_t^{(B)}:=\lambda_{\max}(\bar{\Cb}_t^{(B)})$, while sliding inner-product writes use the window energy $\bar E_t^{(B)}$. In the current sliding-state implementation, this multiplier is treated as statistics-only (detached / stop-gradient) when forming the actual shrinkage coefficient; the step-size denominator still uses the live statistic.

\paragraph{Small-Gram smoothness for Falcon-3.}
For the sliding regression rule, the denominator uses
\[
\mu_t^{(B)}=\lambda_{\max}(\bar{\Cb}_t^{(B)})=\frac{\lambda_{\max}(\mathbf{X}_t^\top\mathbf{X}_t)}{B_t},
\]
where $\mathbf{X}_t\in\mathbb{R}^{d_x\times B_t}$ stacks the active-window write-features. Because $B_t\le B$ is small in the intended regime, one can either diagonalize the $B_t\times B_t$ Gram matrix exactly or approximate its top eigenvalue with a few power iterations; the theory in the main text is stated with the exact quantity.

\paragraph{Relation to RWKV-7 and Kimi Linear Attention.}
RWKV-style recurrent LLMs~\citep{peng2023rwkv} and more recent delta-rule variants (e.g., RWKV-7 and Kimi Linear Attention~\citep{team2025kimilinear}) enrich Delta-style fast weights with learned gating and in-context learning-rate control. Our formulation is complementary: we ground the write rule in a causal next-latent regression objective, make the required one-step key shift explicit, and stabilize learning with NLMS normalization. The same normalization/alignment can be used as a drop-in replacement for the local update inside delta-style blocks, independent of their particular gating or mixing parameterization.

\subsection{Sliding-Window Boundary State}
\label{app:sliding_window_remarks}

For Falcon-3 and Falcon-3A, the matrix state $\Sbb_t$ alone is not Markov for exact continuation. Exact continuation across a segment boundary also requires a fixed-width tail containing the last $B{-}1$ causal pairs $\{(\mathbf{x}_j,\vb_j)\}_{j=\max(2,t-B+2)}^{t}$, or an equivalent rolling-window representation. Resetting this tail changes the near-boundary update and should be treated as an explicit boundary reset.

For Falcon-3, Algorithm~\ref{alg:Falcon3_reference} is written for a fresh sequence with $\mathbf{x}_1=\mathbf{0}$. Algorithm~\ref{alg:Falcon3_forward} realizes the same overlap offline by gathering the extended write-pair set $J=\{p,\ldots,b\}$ at the chunk entry. Passing only the matrix boundary state between disjoint segments is therefore not exact.

For Falcon-3A, aggregated statistics such as $(\bar{\Nb}_t^{(B)},\bar E_t^{(B)},B_t)$ are not sufficient for indefinite continuation, because the next step must remove the oldest contribution exactly. The masked-attention unrolling in Section~\ref{sec:mb_ip_parallel} assumes the same overlap is available.

Away from boundaries, when $\lambda_t=0$ and $\eta_t=\eta$ is constant, each token appears in exactly $B$ consecutive window-averages with weight $1/B$, so its cumulative injection coefficient is $\sum_{s=j}^{j+B-1}\eta/B=\eta$, independent of $B$. This does not make the memory strictly local: after those $B$ direct injections, the token's effect can persist through later carry factors. The window therefore controls an \emph{effective} memory horizon rather than strict local support; strict locality would require explicit truncation or a hard zero-carry override.

\section{Falcon-3A Mask and Backward Pass}
\label{app:Falcon3a_details}

\subsection{Mask Identities and Chunk-Local Evaluation}
\label{app:Falcon3a_mask}

This subsection records the mask formulas underlying Section~\ref{sec:mb_ip_parallel}. Starting from
\[
\Ob
=
\operatorname{Diag}(\boldsymbol{\delta})\,\Qb\Sbb_0
\;+\;
(\Qb\mathbf{X}^\top\odot \Mb)\Vb,
\qquad
\boldsymbol{\delta}:=(\delta_1,\ldots,\delta_L)^\top,
\]
the fast-weight contribution is a masked-attention term and the boundary contribution is a decayed-history term. The causal mask is step dependent (and input-conditioned through $\eta_s$ when the write energy varies): each coefficient $M_{t,j}$ is a sum of at most $B$ decayed write coefficients $\eta_s/B_s$, and it is generally nonzero for all $j\le t$. Moreover, $M_{t,j}$ need not be monotone in the lag $t-j$ because token $j$ is re-injected into $\Nb_s^{(B)}$ for $s=j,\ldots,j+B-1$ before subsequent decay acts on the state.

In the stationary case $\eta_s=\eta$ and $\gamma_s=\gamma\in(0,1)$,
\[
M_{t,j}=\frac{\eta}{B}\sum_{u=\max(0,\,t-j-B+1)}^{t-j}\gamma^{u},
\]
for full-window interior positions where $B_s=B$ throughout the contributing range. Near sequence or segment boundaries, replace $B$ by the realized $B_s$ in the sum. This makes the moving-sum then exponential-tail structure explicit. All expressions here follow the paper-wide read-after-write convention: the readout at position $t$ uses the updated state $\Sbb_t$ and is used to predict token $t{+}1$.

\paragraph{Mask structure.}
The mask $M_{t,j}$ can be written explicitly as
\begin{align}
M_{t,j} = \sum_{s=j}^{\min(t, j+B-1)} \frac{\eta_s}{B_s}\prod_{r=s+1}^{t} (1-\eta_r\lambda_r),
\end{align}
with the implemented recurrence replacing $1-\eta_r\lambda_r$ by the clamped $\gamma_r$ when needed. This separates the window-average accumulation from the global multiplicative carry.

\paragraph{Log-space cumulative decay.}
Directly forming $\prod_{r=s+1}^{t}\gamma_r$ can underflow over long sequences, especially in reduced precision. Define the log cumulative decay
\[
\alpha_t := \min(\eta_t\lambda_t,\,1-\varepsilon_\gamma),
\qquad
\zeta_0 := 0,\qquad \zeta_t := \zeta_{t-1} + \log\gamma_t \quad (t\ge 1),
\]
with $\log\gamma_t=\operatorname{log1p}(-\alpha_t)$ computed in fp32, so that $\prod_{r=s+1}^{t}\gamma_r = \exp(\zeta_t-\zeta_s)$ and
\begin{align}
M_{t,j} = \sum_{s=j}^{\min(t, j+B-1)} \frac{\eta_s}{B_s}\,\exp(\zeta_t-\zeta_s).
\label{eq:mb_ip_mask_logspace}
\end{align}

\paragraph{Chunk-local renormalization.}
Even with the log representation, explicitly materializing global $\zeta_t$ and then forming $\exp(\zeta_t-\zeta_s)$ at very large lags can underflow in reduced precision and is unnecessary for chunked kernels. All chunk-parallel algorithms in this paper therefore reset the log-prefix to $0$ at chunk boundaries, use only within-chunk differences $u_i-u_s$ with $u_0=0$, and propagate boundary states via the chunk-exit decay $g^{(k)}=\prod_{t\in\text{chunk }k}\gamma_t$. This is exactly the local log-decay trick used by Algorithm~\ref{alg:Falcon3a_chunked_attn}.

\paragraph{Chunk-parallel evaluation.}
For implementation, it is helpful to expose the mask as the composition of (i) a causal decay kernel and (ii) a $B$-banded moving-average window operator:
\[
D_{t,s}:=\mathbb{I}[s\le t]\prod_{r=s+1}^{t}\gamma_r,
\qquad
B_s^+:=\max(1,B_s),
\qquad
A_{s,j}:=\frac{\mathbb{I}[j\in\mathcal{I}_s]}{B_s^+},
\qquad
\Mb = \Db\,\operatorname{Diag}(\boldsymbol{\eta})\,\Ab,
\]
where $\boldsymbol{\eta}:=(\eta_1,\ldots,\eta_L)^\top$. Under our boundary convention ($\mathcal{I}_1=\emptyset$ and $\eta_1=0$), the $s=1$ row of $\Ab$ is all zeros, and for all $s\ge 2$ we have $B_s^+=B_s$, so this agrees with $A_{s,j}=\mathbb{I}[j\in\mathcal{I}_s]/B_s$ while avoiding a $0/0$ definition at $s=1$. This yields the three-phase chunked algorithm used in the main text: per-chunk precomputation (parallel), inter-chunk boundary propagation (short recurrence/scan), and final output materialization (parallel).

\subsection{Backward Pass}
\label{app:Falcon3a_backward}

The backward pass is reverse-mode differentiation through Algorithm~\ref{alg:Falcon3a_chunked_attn}. In the final $(\beta,\lambda,\eta)$ parameterization, the only nontrivial scalar chain rules come from
\[
\eta_t = \frac{\beta_t}{d_t},
\qquad
d_t:=\bar E_t^{(B)}+\lambda_t+\varepsilon,
\qquad
\alpha_t = \min(\eta_t\lambda_t,1-\varepsilon_\gamma),
\qquad
\bar E_t^{(B)} = \frac{1}{B_t}\sum_{j\in\mathcal{I}_t}\|\mathbf{x}_j\|_2^2.
\]

Let $\bar\eta_t$ and $\bar\alpha_t$ denote the adjoints arriving from the structured-mask and log-decay paths, respectively. For unclamped steps ($\eta_t\lambda_t<1-\varepsilon_\gamma$),
\[
\bar\eta_t \mathrel{+}= \lambda_t\,\bar\alpha_t,
\qquad
\bar\lambda_t \mathrel{+}= \eta_t\,\bar\alpha_t.
\]
Then
\[
\bar\beta_t \mathrel{+}= \frac{\bar\eta_t}{d_t},
\qquad
\bar d_t = -\frac{\beta_t}{d_t^2}\,\bar\eta_t,
\qquad
\bar\lambda_t \mathrel{+}= \bar d_t,
\qquad
g_t^{(E)} \mathrel{+}= \bar d_t,
\]
where $g_t^{(E)}$ denotes the adjoint of $\bar E_t^{(B)}$.

Finally, the window-energy path contributes
\[
\bar{\mathbf{x}}_j \mathrel{+}= 2\left(\sum_{t:\,j\in\mathcal{I}_t}\frac{g_t^{(E)}}{B_t}\right)\mathbf{x}_j.
\]
Equivalently, in the chunked notation of Algorithm~\ref{alg:Falcon3a_chunked_attn},
\[
\bar{\mathbf{X}}_{\rm ext}^{(k)}
\mathrel{+}= 2\big((\Ab^{(k)\top}\mathbf g^{(k,E)})\mathbf 1_{d_x}^\top\big)\odot \mathbf{X}_{\rm ext}^{(k)},
\]
where $\mathbf g^{(k,E)}\in\mathbb{R}^{C}$ stores the adjoints of the chunk-local window energies.

If the same scale-coupled base-ridge parameterization as in \Falcon-3 is used, first run the backward pass above with respect to the actual coefficients $\lambda_t$. Then apply the outer chain rule through
\[
\lambda_t=\bar\lambda_t\,\bar E_t^{(B)}.
\]
Thus, the base-ridge gradient is the gradient with respect to the actual coefficient multiplied by $\bar E_t^{(B)}$; if the multiplier is live, add the corresponding $\bar\lambda_t$-scaled term to the energy gradient. With the statistics-only / stop-gradient multiplier used in the current sliding-state implementation, only the base-ridge path remains, and $\bar E_t^{(B)}$ is treated as detached.

When the clamp is active, multiply the $\bar\alpha_t$ path by the indicator $\mathbb{I}[\eta_t\lambda_t<1-\varepsilon_\gamma]$. The remaining derivatives through the block-causal mask, local log-prefix decays, boundary propagation, and output materialization are unchanged once $\eta_t$ and $\gamma_t$ have been materialized. Because the boundary step is hard-set to $\eta_1=0$ and $\gamma_1=1$, it carries no gradient to $(\beta_1,\lambda_1)$.

\section{DeltaNet WY Parallelization}
\label{sec:delta-parallel}

The basic Delta Network update in Section~\ref{sec:autoregressive} has the form
\begin{align}
\Sbb_t
&= \big(\Ib - \eta_{t}\mathbf{x}_{t}\mathbf{x}_{t}^\top\big)\Sbb_{t-1}
+ \eta_{t}\mathbf{x}_{t}\vb_t^\top,
\label{eq:delta-basic-affine}
\end{align}
where $\mathbf{x}_t:=\phi(\mathbf{k}_{t-1})$ is the shifted write feature. This corresponds to the no-ridge rank-one case $\lambda_t=0$ of Eq.~\eqref{eq:ogd}. When ridge regularization / weight decay is enabled, Eq.~\eqref{eq:ogd} changes the transition to $\Ab_t=\gamma_t\Ib-\eta_t\mathbf{x}_{t}\mathbf{x}_{t}^\top$ with $\gamma_t:=1-\eta_t\lambda_t$. For $\gamma_t>0$ one can factor $\Ab_t=\gamma_t(\Ib-\ub'_t{\ub'_t}^\top)$ where $\ub'_t:=\sqrt{\eta_t/\gamma_t}\,\mathbf{x}_{t}$, and absorb $\prod \gamma_t$ into a cumulative rescaling of boundary states and bias terms. For clarity, the WY derivation below presents the no-ridge rank-one case $\lambda_t=0$.

\paragraph{Explicit reduction to the no-ridge rank-one case (for $\lambda_t>0$).}
Because $\gamma_t$ is a \emph{scalar}, one can reuse the WY machinery unchanged via a simple rescaling.
Let $c_0:=1$ and $c_t:=\prod_{r=1}^{t}\gamma_r$ and define
$\tilde{\Sbb}_t:=\Sbb_t/c_t$, $\tilde{\eta}_t:=\eta_t/\gamma_t$, and $\tilde{\vb}_t:=\vb_t/c_{t-1}$.
Then the decayed recurrence is equivalent to the no-ridge rank-one recurrence
\[
\tilde{\Sbb}_t
= \big(\Ib-\tilde{\eta}_t\,\mathbf{x}_{t}\mathbf{x}_{t}^\top\big)\tilde{\Sbb}_{t-1} + \tilde{\eta}_t\,\mathbf{x}_{t}\tilde{\vb}_t^\top,
\qquad \ob_t=c_t\tilde{\ob}_t.
\]
\noindent\textbf{Chunk-local renormalization.}
Forming $c_t$ (or $c_{t-1}^{-1}$) over long sequences can underflow or overflow; all chunk-wise kernels therefore apply this rescaling \emph{within each chunk} by resetting the log-prefix decay to $0$ at chunk boundaries (see Algorithm~\ref{alg:deltanet2-parallelized-decay} and Algorithm~\ref{alg:deltanet2_lite_fwd}).

\paragraph{Step-size convention.}
Throughout this appendix, $\eta_t$ denotes the actual step size used in the rank-one update (e.g., the NLMS step size from Eq.~\eqref{eq:nlms-stepsize}).
In the \emph{no-ridge rank-one} setting treated here ($\lambda_t=0$), a gain $\beta_t\in(0,2)$ corresponds to
\[
\eta_t=\frac{\beta_t}{\|\mathbf{x}_{t}\|_2^2+\varepsilon},
\]
with the usual convention $\eta_t:=0$ when the denominator vanishes. Hence $\eta_t\ge 0$ and $\sqrt{\eta_t}$ is well-defined.

\subsection{Affine WY Form}
We explicitly define the affine structure. Write
\begin{align}
\Ab_t := \Ib - \eta_{t}\mathbf{x}_{t}\mathbf{x}_{t}^\top\in\mathbb{R}^{d_x\times d_x},
\qquad
\Bb_t := \eta_{t}\mathbf{x}_{t}\vb_t^\top\in\mathbb{R}^{d_x\times d_v},
\label{eq:delta-A-B-def}
\end{align}
so that the recurrence becomes $\Sbb_t = \Ab_t \Sbb_{t-1} + \Bb_t$ (no-ridge rank-one; $\lambda_t=0$).
Eq.~\eqref{eq:delta-A-B-def} reveals a rank-one structure $\Ab_t = \Ib - \ub_t\ub_t^\top$ where $\ub_t := \sqrt{\eta_t}\,\mathbf{x}_{t}$.
\citet{bischof1987wy} show that the product of such matrices over a block can be written as a single WY transform (or UT transform). For a chunk of size $C$, the accumulated transition matrix $\widehat{\Ab} = \Ab_{t+C} \cdots \Ab_{t+1}$ is:
\begin{align}
\widehat{\Ab} = \Ib - \Ub \Tb \Ub^\top,
\label{eq:wy-block}
\end{align}
where $\Ub \in\mathbb{R}^{d_x\times C}$ collects the update vectors $\ub$, and $\Tb \in \mathbb{R}^{C \times C}$ is a triangular mixing matrix.

\subsection{Chunk-Parallel Forward Pass}

Similar to \citet{yang2024parallelizing}, we have the following three-phase parallel algorithm. Phase 1 computes the local operators in parallel. Phase 2 performs a fast recurrence over chunk boundaries. Phase 3 materializes the outputs by combining the propagated state with a local causal attention mechanism.

\begin{algorithm}[H]
\caption{Parallel DeltaNet Training (Forward, $\lambda_t=0$)}
\label{alg:deltanet-parallel}
\begin{algorithmic}[1]
\Require Shifted write features $\mathbf{x}\in\mathbb{R}^{L\times d_x}$ (with $\mathbf{x}_t:=\phi(\mathbf{k}_{t-1})$), queries $\Qb\in\mathbb{R}^{L\times d_x}$, values $\Vb\in\mathbb{R}^{L\times d_v}$, step sizes $\boldsymbol{\eta}\in\mathbb{R}^{L}$ with $\eta_t\ge 0$ (so $\sqrt{\eta_t}$ is defined), initial state $\Sbb_{\text{init}}\in\mathbb{R}^{d_x\times d_v}$ (default $\mathbf{0}$). Chunk size $C$ (assume $C\mid L$). \Comment{No-ridge rank-one case $\lambda_t=0$.}
\State Reshape inputs into $M = L/C$ chunks: $\Ub_{\rm raw}^{(k)}, \Qb^{(k)} \in \mathbb{R}^{d_x \times C}$ and $\Vb^{(k)} \in \mathbb{R}^{C \times d_v}$, with $\boldsymbol{\eta}^{(k)} \in \mathbb{R}^{C}$.
\State \Comment{Here $\mathbf{x}_t$ is the shifted write-feature stream.}

\State \Comment{\textbf{Phase 1: Intra-chunk structural pre-computation (Parallel)}}
\For{$k = 1$ \textbf{to} $M$ \textbf{in parallel}}
\State $\Ub^{(k)} \gets \Ub_{\rm raw}^{(k)} \cdot \operatorname{diag}(\sqrt{\boldsymbol{\eta}^{(k)}})$ \Comment{Scale key-columns}
\State $\tilde{\Vb}^{(k)} \gets \operatorname{diag}(\sqrt{\boldsymbol{\eta}^{(k)}})\,\Vb^{(k)}$ \Comment{Scale value-rows}
\State $\Gb^{(k)} \gets {\Ub^{(k)}}^\top \Ub^{(k)}$ \Comment{Gram Matrix}
\State $\Lb^{(k)} \gets \operatorname{tril}(\Gb^{(k)}, -1) + \Ib$ \Comment{Implicit $\Tb^{-1}$}
\EndFor

\State \Comment{\textbf{Phase 2: Inter-chunk Recurrence (Sequential)}}
\State $\Sbb^{(1)}_{\text{in}} \gets \Sbb_{\text{init}}$
\For{$k = 1$ \textbf{to} $M$}
\State $\Pb^{(k)} \gets {\Ub^{(k)}}^\top \Sbb^{(k)}_{\text{in}}$
\State $\Rb^{(k)} \gets \tilde{\Vb}^{(k)} - \Pb^{(k)}$
\State $\Bb^{(k)} \gets \operatorname{solve\_triangular}(\Lb^{(k)}, \Rb^{(k)})$ \Comment{single residual solve}
\State $\Sbb^{(k+1)}_{\text{in}} \gets \Sbb^{(k)}_{\text{in}} + \Ub^{(k)}\Bb^{(k)}$
\EndFor
\State $\Sbb_{\text{final}}\gets \Sbb^{(M+1)}_{\text{in}}$ \Comment{Final state after processing length $L$}

\State \Comment{\textbf{Phase 3: Materialize Outputs (Parallel)}}
\For{$k = 1$ \textbf{to} $M$ \textbf{in parallel}}
\State $\Hb^{(k)} \gets {\Qb^{(k)}}^\top \Sbb^{(k)}_{\text{in}}$ \Comment{$C \times d_v$: Query interaction with incoming state}
\State $\boldsymbol{\Lambda}^{(k)} \gets {\Qb^{(k)}}^\top \Ub^{(k)}$ \Comment{$C \times C$: Query-Key interaction matrix}
\State $\Mb^{(k)} \gets \operatorname{tril}(\boldsymbol{\Lambda}^{(k)}, 0)$ \Comment{Causal scores including diagonal}
\State $\Ob^{(k)} \gets \Hb^{(k)} + \Mb^{(k)} \Bb^{(k)}$
\EndFor
\State \Return $(\Ob,\Sbb_{\text{final}})$
\end{algorithmic}
\end{algorithm}

\paragraph{Single residual solve.}
The two-solve decomposition is algebraically unnecessary because the local value path and the projected incoming-state path share the same unit-lower-triangular matrix.
Let
\[
\Pb^{(k)}:={\Ub^{(k)}}^\top\Sbb^{(k)}_{\text{in}},
\qquad
\Rb^{(k)}:=\tilde{\Vb}^{(k)}-\Pb^{(k)},
\qquad
\Bb^{(k)}:=\Lb^{(k)-1}\Rb^{(k)}.
\]
By linearity,
\[
\Bb^{(k)}
=
\Lb^{(k)-1}\tilde{\Vb}^{(k)}-\Lb^{(k)-1}\Pb^{(k)},
\]
so $\Bb^{(k)}$ is exactly the difference between the older ``intra'' and ``history'' solves.
Hence
\[
\Sbb^{(k)}_{\text{out}}
=
\Sbb^{(k)}_{\text{in}}+\Ub^{(k)}\Bb^{(k)},
\qquad
\Ob^{(k)}
=
\Hb^{(k)}+\Mb^{(k)}\Bb^{(k)},
\]
which removes one triangular solve per chunk in the forward pass without changing the recurrence.

\subsection{Chunk-Parallel Backward Pass}
\label{app:deltanet-parallel-bwd}

\begin{breakablealgorithm}
\caption{Parallel DeltaNet Training (Backward, $\lambda_t=0$)}
\label{alg:deltanet-parallel-bwd}
\small
\begin{algorithmic}[1]
\Require Upstream gradient $\bar{\Ob}\in\mathbb{R}^{L\times d_v}$ and cached forward tensors from Algorithm~\ref{alg:deltanet-parallel} (no-ridge rank-one, $\lambda_t=0$) for each chunk $k\in\{1,\dots,M\}$:
unscaled keys $\Ub_{\rm raw}^{(k)}\in\mathbb{R}^{d\times C}$, queries $\Qb^{(k)}\in\mathbb{R}^{d\times C}$, values $\Vb^{(k)}\in\mathbb{R}^{C\times d_v}$, step sizes $\boldsymbol{\eta}^{(k)}\in\mathbb{R}^{C}$,
scaled keys $\Ub^{(k)}=\Ub_{\rm raw}^{(k)}\operatorname{diag}(\sqrt{\boldsymbol{\eta}^{(k)}})$,
scaled values $\tilde{\Vb}^{(k)}=\operatorname{diag}(\sqrt{\boldsymbol{\eta}^{(k)}})\Vb^{(k)}$,
$\Lb^{(k)}\in\mathbb{R}^{C\times C}$, residual solves $\Bb^{(k)}=\operatorname{solve\_triangular}\!\big(\Lb^{(k)},\,\tilde{\Vb}^{(k)}-{\Ub^{(k)}}^\top\Sbb^{(k)}_{\text{in}}\big)$,
and boundary states $\Sbb^{(k)}_{\text{in}}\in\mathbb{R}^{d\times d_v}$.
\Ensure Gradients $\bar{\Qb}\in\mathbb{R}^{L\times d}$, $\bar{\Ub}_{\rm raw}\in\mathbb{R}^{L\times d}$, $\bar{\Vb}\in\mathbb{R}^{L\times d_v}$, $\bar{\boldsymbol{\eta}}\in\mathbb{R}^{L}$, and $\bar{\Sbb}_{\text{init}}\in\mathbb{R}^{d\times d_v}$.

\State Reshape $\bar{\Ob}$ into chunks $\bar{\Ob}^{(k)}\in\mathbb{R}^{C\times d_v}$, $M=L/C$.
\State Initialize $\bar{\Qb}^{(k)},\bar{\Ub}^{(k)},\bar{\tilde{\Vb}}^{(k)},\bar{\Lb}^{(k)},\bar{\Sbb}^{(k)}_{\text{in}},\bar{\Bb}^{(k)},\bar{\Pb}^{(k)},\bar{\Rb}^{(k)} \gets \mathbf{0}$ for all $k$.
\State Set $\bar{\Sbb}^{(M+1)}_{\text{in}}$ to the upstream gradient on the final chunk-exit state (default $\mathbf{0}$ if unused by the loss).

\State \Comment{\textbf{Adjoint identity (triangular solve).} Let $L$ be unit-lower triangular with fixed diagonal. If $X=L^{-1}B$,}
\State \Comment{then $\bar{B}=L^{-T}\bar{X}$ and $\bar{L}\mathrel{-{=}}\operatorname{tril}(\bar{B}X^\top,-1)$.}

\State \Comment{\textbf{Adjoint identity (Gram).} If $G=U^\top U$, then $\bar{U}\gets \bar{U}+U(\bar{G}+\bar{G}^\top)$.}

\State \Comment{\textbf{Phase 3$^\top$: Backprop through output materialization (Parallel)}}
\For{$k=1$ \textbf{to} $M$ \textbf{in parallel}}
\State \Comment{No-ridge rank-one identity: $\Ob^{(k)}=\Hb^{(k)}+\Mb^{(k)}\Bb^{(k)}$ with $\Hb^{(k)}=(\Qb^{(k)})^\top \Sbb^{(k)}_{\text{in}}$ and $\Mb^{(k)}=\operatorname{tril}((\Qb^{(k)})^\top \Ub^{(k)},0)$.}
\State $\bar{\Hb}^{(k)} \gets \bar{\Ob}^{(k)}$
\State $\boldsymbol{\Lambda}^{(k)} \gets (\Qb^{(k)})^\top \Ub^{(k)}$
\State $\Mb^{(k)} \gets \operatorname{tril}(\boldsymbol{\Lambda}^{(k)}, 0)$
\State $\bar{\Mb}^{(k)} \gets \bar{\Ob}^{(k)}(\Bb^{(k)})^\top$
\State $\bar{\Mb}^{(k)} \gets \operatorname{tril}(\bar{\Mb}^{(k)}, 0)$ \Comment{$\Mb^{(k)}$ is causal (includes diagonal)}
\State $\bar{\Bb}^{(k)} \gets \bar{\Bb}^{(k)} + (\Mb^{(k)})^\top \bar{\Ob}^{(k)}$

\State \Comment{Backprop $\Hb^{(k)}=(\Qb^{(k)})^\top \Sbb^{(k)}_{\text{in}}$}
\State $\bar{\Qb}^{(k)} \gets \bar{\Qb}^{(k)} + \Sbb^{(k)}_{\text{in}}(\bar{\Hb}^{(k)})^\top$
\State $\bar{\Sbb}^{(k)}_{\text{in}} \gets \bar{\Sbb}^{(k)}_{\text{in}} + \Qb^{(k)}\bar{\Hb}^{(k)}$

\State \Comment{Backprop $\Mb^{(k)}=\operatorname{tril}(\boldsymbol{\Lambda}^{(k)},0)$ and $\boldsymbol{\Lambda}^{(k)}=(\Qb^{(k)})^\top \Ub^{(k)}$}
\State $\bar{\boldsymbol{\Lambda}}^{(k)} \gets \bar{\Mb}^{(k)}$
\State $\bar{\Qb}^{(k)} \gets \bar{\Qb}^{(k)} + \Ub^{(k)}(\bar{\boldsymbol{\Lambda}}^{(k)})^\top$
\State $\bar{\Ub}^{(k)} \gets \bar{\Ub}^{(k)} + \Qb^{(k)}\bar{\boldsymbol{\Lambda}}^{(k)}$
\EndFor

\State \Comment{\textbf{Phase 2$^\top$: Backprop through inter-chunk recurrence (Sequential reverse)}}
\For{$k=M$ \textbf{down to} $1$}
\State $\bar{\Sbb}^{(k)}_{\text{next}} \gets \bar{\Sbb}^{(k+1)}_{\text{in}}$

\State \Comment{State update: $\Sbb^{(k+1)}_{\text{in}}=\Sbb^{(k)}_{\text{in}}+\Ub^{(k)}\Bb^{(k)}$}
\State $\bar{\Sbb}^{(k)}_{\text{in}} \gets \bar{\Sbb}^{(k)}_{\text{in}} + \bar{\Sbb}^{(k)}_{\text{next}}$
\State $\bar{\Ub}^{(k)} \gets \bar{\Ub}^{(k)} + \bar{\Sbb}^{(k)}_{\text{next}}(\Bb^{(k)})^\top$
\State $\bar{\Bb}^{(k)} \gets \bar{\Bb}^{(k)} + (\Ub^{(k)})^\top \bar{\Sbb}^{(k)}_{\text{next}}$

\State \Comment{Solve: $\Bb^{(k)}=\operatorname{solve\_triangular}(\Lb^{(k)},\Rb^{(k)})$}
\State $\bar{\Rb}^{(k)}_{\text{solve}} \gets \operatorname{solve\_triangular}(\Lb^{(k)\top}, \bar{\Bb}^{(k)})$
\State $\bar{\Rb}^{(k)} \gets \bar{\Rb}^{(k)} + \bar{\Rb}^{(k)}_{\text{solve}}$
\State $\bar{\Lb}^{(k)} \gets \bar{\Lb}^{(k)} - \operatorname{tril}\big(\bar{\Rb}^{(k)}_{\text{solve}}(\Bb^{(k)})^\top,\,-1\big)$

\State \Comment{Residual input: $\Rb^{(k)}=\tilde{\Vb}^{(k)}-\Pb^{(k)}$}
\State $\bar{\tilde{\Vb}}^{(k)} \gets \bar{\tilde{\Vb}}^{(k)} + \bar{\Rb}^{(k)}$
\State $\bar{\Pb}^{(k)} \gets \bar{\Pb}^{(k)} - \bar{\Rb}^{(k)}$

\State \Comment{Projected incoming state: $\Pb^{(k)}=(\Ub^{(k)})^\top\Sbb^{(k)}_{\text{in}}$}
\State $\bar{\Ub}^{(k)} \gets \bar{\Ub}^{(k)} + \Sbb^{(k)}_{\text{in}}(\bar{\Pb}^{(k)})^\top$
\State $\bar{\Sbb}^{(k)}_{\text{in}} \gets \bar{\Sbb}^{(k)}_{\text{in}} + \Ub^{(k)}\bar{\Pb}^{(k)}$
\EndFor
\State $\bar{\Sbb}_{\text{init}} \gets \bar{\Sbb}^{(1)}_{\text{in}}$

\State \Comment{\textbf{Phase 1$^\top$: Backprop through local structure and input scaling (Parallel)}}
\For{$k=1$ \textbf{to} $M$ \textbf{in parallel}}
\State $\bar{\Lb}^{(k)} \gets \operatorname{tril}(\bar{\Lb}^{(k)},-1)$ \Comment{diag of $\Lb^{(k)}$ is constant ($=\Ib$)}
\State $\bar{\Gb}^{(k)} \gets \bar{\Lb}^{(k)}$ \Comment{since $\Lb^{(k)}=\operatorname{tril}(\Gb^{(k)},-1)+\Ib$}
\State $\bar{\Ub}^{(k)} \gets \bar{\Ub}^{(k)} + \Ub^{(k)}\big(\bar{\Gb}^{(k)}+(\bar{\Gb}^{(k)})^\top\big)$ \Comment{backprop $\Gb^{(k)}=(\Ub^{(k)})^\top\Ub^{(k)}$}

\State $\boldsymbol{\sigma}^{(k)} \gets \sqrt{\boldsymbol{\eta}^{(k)}}$
\State \Comment{Unscale: $\Ub^{(k)}=\Ub_{\rm raw}^{(k)}\operatorname{diag}(\boldsymbol{\sigma}^{(k)})$ and $\tilde{\Vb}^{(k)}=\operatorname{diag}(\boldsymbol{\sigma}^{(k)})\Vb^{(k)}$}
\State $\bar{\Ub}_{\rm raw}^{(k)} \gets \bar{\Ub}^{(k)}\operatorname{diag}(\boldsymbol{\sigma}^{(k)})$
\State $\bar{\Vb}^{(k)} \gets \operatorname{diag}(\boldsymbol{\sigma}^{(k)})\,\bar{\tilde{\Vb}}^{(k)}$

\State \Comment{Column/row dot definitions:
($\operatorname{col\_dot}(A,B))_i=\langle A_{:,i},B_{:,i}\rangle$,
($\operatorname{row\_dot}(A,B))_i=\langle A_{i,:},B_{i,:}\rangle$.}
\State $\bar{\boldsymbol{\sigma}}^{(k)} \gets \operatorname{col\_dot}(\Ub_{\rm raw}^{(k)},\bar{\Ub}^{(k)}) + \operatorname{row\_dot}(\Vb^{(k)},\bar{\tilde{\Vb}}^{(k)})$
\State $\mathbf{m}^{(k)} \gets \mathbb{I}[\boldsymbol{\sigma}^{(k)} > 0]$ \Comment{element-wise mask}
\State $\bar{\boldsymbol{\eta}}^{(k)} \gets \mathbf{m}^{(k)} \odot \big(\bar{\boldsymbol{\sigma}}^{(k)} \,/\, (2\boldsymbol{\sigma}^{(k)})\big)$
\Comment{Safe when some $\eta_t=0$ (e.g., boundary $\xb_t=\mathbf{0}$): set $\bar{\eta}_t=0$ instead of forming $0/0$.}
\EndFor

\State \Return $\bar{\Qb},\bar{\Ub}_{\rm raw},\bar{\Vb},\bar{\boldsymbol{\eta}},\bar{\Sbb}_{\text{init}}$ (reshape chunk grads back to length $L$).
\end{algorithmic}
\end{breakablealgorithm}

\section{Falcon-2 Parallel Implementation}
\label{app:deltanet2_parallel}

This appendix details the parallel implementation of \textbf{Falcon-2}
in the multi-head form used in our experiments and codebase.
We write the derivation for a single head, the full multi-head attention (MHA) layer applies the same update independently to each head.
Unless explicitly stated otherwise, the symbols $d$ and $d_v$ in this appendix therefore denote the per-head key/query feature dimension and value
dimension, respectively.
Within one head, Falcon-2 introduces column-wise adaptive learning rates
$\boldsymbol{\eta}_t$.
This implies that every value channel $j \in \{1, \dots, d_v\}$ follows a unique
trajectory in the head-local state space, seemingly preventing the use of shared
block-transition matrices.

However, we observe that while the update learning rate varies per column, the geometry of the updates is determined solely by the keys $\mathbf{k}_t$,
which are shared across all value channels within the head.
By exploiting this structure, we can vectorize the chunk-wise recurrence while
retaining per-dimension adaptivity.

\subsection{Per-Channel Dynamics and Shared Geometry}

For one head, let $\Sbb_t \in \mathbb{R}^{d \times d_v}$ be the head-local state.
The Falcon-2 update for the $j$-th column $\mathbf{s}_{t,j}$ is:
\begin{align}
\mathbf{s}_{t,j}
&= \mathbf{s}_{t-1,j} + \eta_{t,j}\,\mathbf{k}_{t-1}\,\Big(v_{t,j} - \langle \mathbf{k}_{t-1}, \mathbf{s}_{t-1,j}\rangle\Big) \nonumber\\
&= \Big(\Ib - \eta_{t,j}\,\mathbf{k}_{t-1}\mathbf{k}_{t-1}^\top\Big)\mathbf{s}_{t-1,j} + \eta_{t,j}\, v_{t,j}\, \mathbf{k}_{t-1},
\end{align}
where we assume $\eta_{t,j}\ge 0$ so that $\sqrt{\eta_{t,j}}$ is well-defined. This makes explicit that the rank-one update matrix $\Ib-\eta_{t,j}\mathbf{k}_{t-1}\mathbf{k}_{t-1}^\top$ depends on the per-channel step size $\eta_{t,j}$, which determines the chunk-local WY system. To parallelize this over a chunk of size $C$, we utilize the dual (Gram) formulation.

\paragraph{No-ridge rank-one form.}
For clarity, we present the no-ridge rank-one case (no explicit per-column decay factor). When ridge/weight decay is enabled in the main text, the unclamped $j$-th column update includes a scalar decay $\gamma_{j,t}:=1-\lambda_t\eta_{j,t}$:
\[
\mathbf{s}_{t,j}=(\gamma_{t,j}\Ib-\eta_{t,j}\mathbf{k}_{t-1}\mathbf{k}_{t-1}^\top)\mathbf{s}_{t-1,j}+\eta_{t,j}v_{t,j}\mathbf{k}_{t-1},
\qquad \gamma_{t,j}:=1-\lambda_t\eta_{t,j}.
\]
The implementation uses the clamped positive-decay version described below whenever this unclamped carry is not guaranteed positive.
Since each value channel evolves independently, this decay can be removed by a channel-wise cumulative rescaling (equivalently, right-multiplying $\Sbb_t$ by a diagonal matrix of inverse cumulative decays), reducing back to the no-ridge rank-one form; cf.\ Appendix~\ref{app:deltanet2_lite} for the shared-decay special case.

\paragraph{Expanded edit decomposition.}
Expanding the compact main-text update gives
\[
\Sbb_t
= \Sbb_{t-1}\Big(\Ib_{d_v}-\lambda_t\,\operatorname{Diag}(\boldsymbol{\eta}_t)\Big)
- \mathbf{x}_t\Big(\mathbf{x}_t^\top \Sbb_{t-1}\operatorname{Diag}(\boldsymbol{\eta}_t)\Big)
+ \mathbf{x}_t\big(\boldsymbol{\eta}_t \odot \mathbf{y}_t\big)^\top.
\]
This isolates the per-column right-multiplicative shrinkage from the feature-direction edit, but it is algebraically identical to the compact update in the main text.

\paragraph{Shared Gram Matrix.}
Let $\mathbf{K} \in \mathbb{R}^{d \times C}$ be the matrix of (shifted) write keys in the current chunk. The core correlation structure is given by the Gram matrix $\mathbf{G} = \mathbf{K}^\top \mathbf{K} \in \mathbb{R}^{C \times C}$. Crucially, $\mathbf{G}$ is independent of the value dimension $d_v$ and needs to be computed only once per chunk.
For channel-wise step sizes, the actual triangular system differs per channel via a simple element-wise modulation of this shared base Gram.

\paragraph{Batched Triangular Solve.}
The variations in per-channel step sizes $\eta_{t,j}$ modulate this Gram matrix. For each channel $j$, the implicit system matrix $\mathbf{L}_j$ for the WY representation is:
\begin{equation}
\mathbf{L}_j = \operatorname{tril}\left( \mathbf{G} \odot (\boldsymbol{\sqrt{\eta}}_j \boldsymbol{\sqrt{\eta}}_j^\top), -1 \right) + \Ib_C,
\end{equation}
where $\boldsymbol{\sqrt{\eta}}_j \in \mathbb{R}^C$ is the vector of step-size square-roots for channel $j$ over the chunk, and $\odot$ denotes element-wise multiplication (broadcasting).
Equivalently, letting $\mathbf{U}_j := \mathbf{K}\operatorname{diag}(\boldsymbol{\sqrt{\eta}}_j)$, we have
\[
\mathbf{G}_j := \mathbf{U}_j^\top \mathbf{U}_j
= \mathbf{G}\odot(\boldsymbol{\sqrt{\eta}}_j \boldsymbol{\sqrt{\eta}}_j^\top),
\qquad
\mathbf{L}_j = \operatorname{tril}(\mathbf{G}_j,-1)+\Ib_C.
\]

Because the chunk size $C$ is typically small (e.g., $C=64$ or $128$), we can efficiently perform a batched triangular solve over the batch dimension $d_v$.

\subsection{Chunk-Parallel Algorithm}

Algorithm~\ref{alg:deltanet2-parallelized-decay} gives the single-head chunk-wise
forward pass in a form directly parallel to the scalar-step-size WY kernel in
Appendix~\ref{sec:delta-parallel}, but batched over value channels and written with
the positive-decay renormalization used in the implementation.
When $\lambda_t=0$, the decay factors are identically one, and the algorithm
reduces to the no-ridge rank-one WY/Gram kernel. The only difference from the
shared-step-size case is that the unit-lower-triangular system $\mathbf{L}_j$ is
channel-dependent (via $\boldsymbol{\eta}_{:,j}$), so the single residual solve is
carried out in batch over $j\in\{1,\ldots,d_v\}$.

\noindent\textbf{Chunk read/write formulas.}
Let $\Sbb_{\mathrm{in}}\in\mathbb{R}^{d\times d_v}$ be the incoming state for one
head in a chunk, and let $\mathbf{K},\mathbf{Q}\in\mathbb{R}^{d\times C}$ and
$\mathbf{V}\in\mathbb{R}^{C\times d_v}$ be the chunk keys, queries, and values
(with causal read-after-write ordering inside the chunk).
Define layout convention: inside a chunk, we use the feature-major
(column-major) layout $\mathbf{K},\mathbf{Q}\in\mathbb{R}^{d\times C}$ with tokens
along columns.
Concretely, if $\mathbf{K}_{(m)},\mathbf{Q}_{(m)}\in\mathbb{R}^{C\times d}$ are the standard time-major slices for chunk $m$, then we set $\mathbf{K}:=\mathbf{K}_{(m)}^\top$ and $\mathbf{Q}:=\mathbf{Q}_{(m)}^\top$.
We keep $\mathbf{V}\in\mathbb{R}^{C\times d_v}$ time-major, and similarly $\boldsymbol{\eta}\in\mathbb{R}^{C\times d_v}$ for per-channel step sizes.

\[
\mathbf{G}:=\mathbf{K}^\top\mathbf{K}\in\mathbb{R}^{C\times C},\qquad
\mathbf{M}:=\operatorname{tril}(\mathbf{Q}^\top\mathbf{K},0)\in\mathbb{R}^{C\times C},\qquad
\mathbf{H}:=\mathbf{Q}^\top\Sbb_{\mathrm{in}}\in\mathbb{R}^{C\times d_v},
\]
and the projected incoming state $\mathbf{P}:=\mathbf{K}^\top\Sbb_{\mathrm{in}}\in\mathbb{R}^{C\times d_v}$. With per-channel step sizes $\boldsymbol{\eta}\in\mathbb{R}^{C\times d_v}$ and $\boldsymbol{\Sigma}:=\sqrt{\boldsymbol{\eta}}$, define for each channel $j$ the unit-lower-triangular
\[
\mathbf{L}_j := \Ib_C + \operatorname{tril}\big(\mathbf{G}\odot(\boldsymbol{\Sigma}_{:,j}\boldsymbol{\Sigma}_{:,j}^\top),-1\big).
\]
Throughout we include the diagonal (read-after-write) via $\operatorname{tril}(\cdot,0)$.

\paragraph{Unscaled scores.}
Note that $\mathbf{M}$ is intentionally formed with the unscaled keys $\mathbf{K}$. All $\sqrt{\eta}$ factors are absorbed into the coefficient matrix $\mathbf{B}$ below; equivalently, for each channel $j$, $\mathbf{M}\big(\boldsymbol{\Sigma}_{:,j}\odot \cdot\big)=\operatorname{tril}\big(\mathbf{Q}^\top\mathbf{K}\operatorname{diag}(\boldsymbol{\Sigma}_{:,j}),0\big)(\cdot)$.

\paragraph{One-TriSolve reduction.}
A two-path implementation would solve
\[
\mathbf{a}_j:=\mathbf{L}_j^{-1}\big(\boldsymbol{\Sigma}_{:,j}\odot\mathbf{V}_{:,j}\big),
\qquad
\mathbf{c}_j:=\mathbf{L}_j^{-1}\big(\boldsymbol{\Sigma}_{:,j}\odot\mathbf{P}_{:,j}\big),
\]
and then form $\mathbf{B}_{:,j}=\boldsymbol{\Sigma}_{:,j}\odot(\mathbf{a}_j-\mathbf{c}_j)$.
Because both paths share the same $\mathbf{L}_j$, linearity gives the exact merged solve
\[
\widetilde{\mathbf{b}}_j
:=
\mathbf{L}_j^{-1}\Big(\boldsymbol{\Sigma}_{:,j}\odot(\mathbf{V}_{:,j}-\mathbf{P}_{:,j})\Big),
\qquad
\mathbf{B}_{:,j}:=\boldsymbol{\Sigma}_{:,j}\odot\widetilde{\mathbf{b}}_j.
\]
This is the form used throughout the paper.
It is algebraically identical to the older two-solve decomposition, but removes one batched TriSolve per chunk in the forward pass.
The chunk outputs and chunk-exit state are then
\[
\mathbf{O} = \mathbf{H} + \mathbf{M}\mathbf{B}\in\mathbb{R}^{C\times d_v},\qquad
\Sbb_{\mathrm{out}} = \Sbb_{\mathrm{in}} + \mathbf{K}\mathbf{B}\in\mathbb{R}^{d\times d_v}.
\]

\paragraph{Pseudocode.}
Algorithm~\ref{alg:deltanet2-parallelized-decay} in the main text gives the single consolidated one-TriSolve forward pass; this appendix focuses on the derivation and complexity.

\paragraph{Efficiency Analysis.}
Per chunk, forming the shared Gram and score matrices costs $\mathcal{O}(dC^2)$ for $\mathbf{K}^\top\mathbf{K}$ and $\mathbf{Q}^\top\mathbf{K}$.
Constructing the channel-specific systems $\{\mathbf{L}_j\}$ and performing the single batched residual solve cost $\mathcal{O}(d_v C^2)$.
Materializing the chunk outputs via $\mathbf{M}\mathbf{B}$ also costs $\mathcal{O}(d_v C^2)$, and the chunk-exit state update $\mathbf{K}\mathbf{B}$ costs $\mathcal{O}(d\,d_v\,C)$.
Relative to the older two-solve presentation, the asymptotic complexity is unchanged, but the dominant forward triangular-solve count drops from two to one per chunk.
Compared to a naive per-channel implementation that would require explicit $d\times d$ operators per value dimension, this confines the rate-dependence to small $C\times C$ systems.
Nevertheless, the $\mathcal{O}(d_v C^2)$ terms can still be substantial for large $d_v$, motivating the shared-dynamics Falcon variant in Appendix~\ref{app:deltanet2_lite}.

\paragraph{Including ridge/weight decay ($\lambda_t>0$).}
The main-text Algorithm~\ref{alg:deltanet2-parallelized-decay} already includes the $\lambda_t>0$ path. Algebraically, the decayed recurrence is reduced to the no-ridge rank-one WY/Gram kernel by a channel-wise cumulative rescaling.
With per-channel step sizes $\eta_{t,j}$ and scalar ridge $\lambda_t$, the $j$-th column evolves as
\[
    \mathbf{s}_{t,j}
    =\big((1-\lambda_t\eta_{t,j})\Ib-\eta_{t,j}\mathbf{k}_{t-1}\mathbf{k}_{t-1}^\top\big)\mathbf{s}_{t-1,j} + \eta_{t,j}v_{t,j}\mathbf{k}_{t-1}.
\]
For the unclamped derivation, define the per-channel decay $\gamma_{t,j}:=1-\lambda_t\eta_{t,j}$ and cumulative products
$c_{0,j}:=1$, $c_{t,j}:=\prod_{s=1}^{t}\gamma_{s,j}$, and $\Db_t:=\operatorname{Diag}(c_{t,1},\ldots,c_{t,d_v})$. In the implemented log-space path, $\gamma_{t,j}$ is replaced by the clamped positive carry defined in the note below.
Then the column-wise rescaled state $\tilde{\Sbb}_t:=\Sbb_t \Db_t^{-1}$ satisfies a no-ridge rank-one recurrence:
\[
    \tilde{\mathbf{s}}_{t,j}
    =\big(\Ib-\tilde{\eta}_{t,j}\mathbf{k}_{t-1}\mathbf{k}_{t-1}^\top\big)\tilde{\mathbf{s}}_{t-1,j}
    + \tilde{\eta}_{t,j}\tilde{v}_{t,j}\mathbf{k}_{t-1},
    \qquad
    \tilde{\eta}_{t,j}:=\eta_{t,j}/\gamma_{t,j},
    \qquad
    \tilde{v}_{t,j}:=v_{t,j}/c_{t-1,j}.
\]
Equivalently, the chunk kernel runs the no-ridge rank-one WY/Gram equations on the rescaled variables
$\tilde{\boldsymbol{\eta}}$ and $\tilde{\Vb}$ (working in the $\tilde{\Sbb}$ domain), and then recovers the original outputs and state by
$\Ob_t=\tilde{\Ob}_t\Db_t$ and $\Sbb_t=\tilde{\Sbb}_t\Db_t$ (i.e., element-wise multiplication by $c_{t,j}$ on each channel).

\noindent\textbf{Log-space note.} First form $\alpha_{t,j}:=\lambda_t\eta_{t,j}$, clamp $\alpha_{t,j}\leftarrow\min(\alpha_{t,j},1-\varepsilon_\gamma)$, and then compute $\log\gamma_{t,j}=\operatorname{log1p}(-\alpha_{t,j})$ in fp32 so that $\gamma_{t,j}>0$. In chunk-wise kernels, do \emph{not} form global products $c_{t,j}$ (or $\exp(\pm\zeta_{t,j})$) over the full sequence; instead, reset the log-prefix at each chunk boundary and use only chunk-local prefixes (Algorithm~\ref{alg:deltanet2-parallelized-decay}).
 
\paragraph{Main-text algorithm.}
The consolidated positive-decay pseudocode appears in Algorithm~\ref{alg:deltanet2-parallelized-decay}; no separate no-ridge rank-one pseudocode is needed because setting $\lambda_t=0$ makes all local decay prefixes equal to one.

\section{Falcon with Shared Dynamics}
\label{app:deltanet2_lite}

Appendix~\ref{app:deltanet2_parallel} derives the full per-channel Falcon-2 parallelization. That formulation is expressive but requires a distinct step-size trajectory for every value channel $j \in \{1, \dots, d_v\}$.

This creates a computational bottleneck: the implicit inverse matrix $\Tb$ (or the Cholesky factor $\Lb$) depends on the learning rates. Consequently, the full Falcon-2 requires solving $d_v$ distinct triangular systems of size $C \times C$ for every chunk. Although $d_v$ denotes the per-head value dimension in this appendix, the aggregate number of value channels over heads can be large, which prevents the use of a single efficient block solve and leads to high memory bandwidth usage.

In this section, we introduce \textbf{Falcon-1}, a hardware-efficient variant that enforces a shared adaptive learning rate across all value channels.
This constraint reduces the number of distinct $C\times C$ triangular systems that must be constructed from $d_v$ to $1$, enabling a single multi-right-hand-side triangular solve per chunk while keeping the per-step rank-one dynamics shared across value channels.
This shared-dynamics reduction is orthogonal to head factorization: in the practical MHA setting used here, one first applies the per-head factorization and then decides whether the step-size dynamics inside each head are full (per value channel) or shared.

Asymptotically, applying the state update still costs $\mathcal{O}(d\,d_v\,C)$ and solving a triangular system with $d_v$ right-hand sides costs $\mathcal{O}(d_v C^2)$.
The primary savings come from eliminating the need to construct and factor $d_v$ distinct $C\times C$ systems: Falcon-1 uses one shared system per chunk and a single high-throughput solve with many right-hand sides.

\subsection{Scalar versus Vector Step Sizes}

In the full Falcon-2, the state update is driven by a vector rate $\boldsymbol{\eta}_t \in \mathbb{R}^{d_v}$. In Falcon-1, we constrain this to a scalar $\eta_t \in \mathbb{R}$, derived from the global energy of the write feature $\mathbf{x}_t$:
\begin{equation}
\label{eq:lite_lr}
\eta_t = \frac{\beta_t}{\|\mathbf{x}_{t}\|_2^2 + \lambda_t + \varepsilon},
\qquad \beta_t\in(0,2),\ \lambda_t\ge 0,\ \varepsilon\ge 0,
\end{equation}
where $\beta_t$ is a learnable scalar NLMS gain, $\lambda_t$ is the ridge coefficient (default $\lambda_t>0$), and $\varepsilon$ is a small stabilizer (cf.\ Eq.~\eqref{eq:nlms-stepsize}).
The autoregressive update rule for the state matrix $\Sbb_t \in \mathbb{R}^{d \times d_v}$ simplifies to:
\begin{equation}
\Sbb_t
= \underbrace{\left((1-\eta_t\lambda_t)\Ib - \eta_t \mathbf{x}_{t}\mathbf{x}_{t}^\top\right)}_{\Ab_t}\Sbb_{t-1}
+ \eta_t\,\mathbf{x}_{t}\vb_t^\top,
\qquad \gamma_t^{\rm raw}:=1-\eta_t\lambda_t.
\end{equation}
Crucially, the transition matrix $\Ab_t$ is now \textit{shared} across all columns of $\Sbb_t$. This implies that while different features store different contents (values), they share the same dynamics (write/forget speeds).

\paragraph{Including weight decay.}
Falcon-2 uses $\lambda_t>0$ by default, inducing the unclamped shrinkage factor $\gamma_t^{\rm raw}:=1-\eta_t\lambda_t$.
For the log-space reduction below, we require a positive carry $\gamma_t>0$ (equivalently, the realized decay fraction must be $<1$).
Under our default RMSNorm scaling (so $\|\mathbf{x}_{t}\|_2^2\approx d$) and $\lambda_t\in[0,1]$, Eq.~\eqref{eq:lite_lr} implies $\eta_t\lambda_t\lesssim \beta_t/(d+1)<1$ and therefore $\gamma_t\in(0,1]$ automatically away from the boundary sentinel.
More generally (and in finite precision), we clamp the decay fraction $\alpha_t:=\eta_t\lambda_t$ as
$\alpha_t\leftarrow \min(\alpha_t,1-\varepsilon_\gamma)$, i.e.\ $\gamma_t\leftarrow \max(1-\eta_t\lambda_t,\varepsilon_\gamma)$, so that $\gamma_t\ge \varepsilon_\gamma>0$.
For the identity-feature Falcon setting, one may realize the boundary no-op via $\mathbf{k}_0=\mathbf{0}$ and hence $\xb_1=\mathbf{0}$; in the general kernelized setting, the mathematically correct convention is imposed directly as $\xb_1:=\mathbf{0}$. We take $\eta_1=0$ and $\gamma_1=1$ (no write/decay), matching the main-text convention that updates start at $t=2$. Because $\gamma_t$ is a scalar, one can reduce the decayed recurrence to the no-ridge rank-one case by rescaling.
Let $c_0:=1$ and $c_t:=\prod_{r=1}^t \gamma_r$, and define $\tilde{\Sbb}_t:=\Sbb_t/c_t$ (so under read-after-write, reads satisfy $\ob_t=c_t\tilde{\ob}_t$).
Then the decayed update is equivalent to running the no-ridge rank-one WY kernels on
\[
\tilde{\eta}_t:=\eta_t/\gamma_t,\qquad \tilde{\vb}_t:=\vb_t/c_{t-1},
\]
and rescaling outputs by $c_t$.
Algorithms~\ref{alg:deltanet2_lite_fwd}-\ref{alg:deltanet2_lite_bwd} implement this reduction explicitly.

\paragraph{Numerical stability.}
The algebraic reduction to the no-ridge rank-one WY kernels rescales targets by the inverse cumulative decay, $\tilde{\vb}_t=\vb_t/c_{t-1}$ with $c_t=\prod_{r\le t}\gamma_r$.
While correct in exact arithmetic, in finite precision $c_t$ may become extremely small when $\gamma_t$ is noticeably below $1$, making $c_{t-1}^{-1}=\exp(-\zeta_{t-1})$ overflow and yielding $\infty\times 0$ patterns (and NaNs) when rescaling back by $\exp(\zeta_t)$.

\paragraph{Stable chunk-local renormalization.}
To avoid ever forming $\exp(\pm\zeta_t)$ over the full sequence, our implementation performs the same reduction within each WY chunk of length $C$.
For a chunk $[a,b]$ (where $b=a{+}C{-}1$), define local log-prefix decays $u_0:=0$ and $u_i:=\sum_{r=0}^{i-1}\log\gamma_{a+r}$ for $i=1,\ldots,C$, with $\delta_i:=\exp(u_i)$.
We run the no-ridge rank-one WY kernel on the rescaled values $\hat{\vb}_{a+i-1}:=\vb_{a+i-1}\exp(-u_{i-1})=\vb_{a+i-1}/\delta_{i-1}$ and step sizes $\hat{\eta}_t:=\eta_t/\gamma_t$ (compute $\log\gamma_t=\operatorname{log1p}(-\alpha_t)$ in fp32 with $\alpha_t:=\min(\eta_t\lambda_t,1-\varepsilon_\gamma)$, then $\hat{\eta}_t=\eta_t\exp(-\log\gamma_t)$).
Outputs and the chunk-exit state are rescaled back by the \emph{forward} local decays:
$\ob_{a+i-1}=\delta_i\hat{\ob}_{a+i-1}$ and $\Sbb_{b}=\delta_C\hat{\Sbb}_{b}$.
This is algebraically equivalent to the global reduction but bounds exponentials by $O(C)$ rather than $O(L)$.

\subsection{Shared-WY Parallelization}

Falcon-1 uses a \emph{single scalar} step size $\eta_t$ shared across all value channels.
Consequently, within each chunk, the WY factors are shared across columns of $\Sbb$:
the per-chunk system matrix $\Lb^{(k)}$ is \emph{identical} for all $d_v$ columns.
This is precisely the setting of Appendix~\ref{sec:delta-parallel} (no-ridge rank-one, $\lambda_t=0$), where the WY implementation yields
one unit-lower-triangular matrix $\Lb^{(k)}\in\mathbb{R}^{C\times C}$ per chunk and, after merging the write and history paths into a single residual right-hand side, one standard \emph{multi-RHS} triangular solve with $d_v$ right-hand sides.

\paragraph{Forward/backward reuse.}
With shared dynamics, the chunk-parallel attention forward pass is exactly Algorithm~\ref{alg:deltanet-parallel}
and the backward pass through the WY machinery is Algorithm~\ref{alg:deltanet-parallel-bwd}.
The only additional ingredient in Falcon-1 is that the step size $\eta_t$ is not a free input: it is produced by the NLMS normalization
\begin{equation}
\eta_t = \frac{\beta_t}{\|\xb_t\|_2^2 + \lambda_t + \varepsilon},
\qquad \xb_t:=\phi(\mathbf{k}_{t-1}) \text{ is the shifted write feature},
\label{eq:lite_lr_repeat}
\end{equation}
and the default $\lambda_t>0$ introduces the additional scalar pathway $\gamma_t=1-\eta_t\lambda_t$.
Backpropagation therefore includes chain-rule steps through both $\eta_t$ and the cumulative decay $c_t$ (Algorithm~\ref{alg:deltanet2_lite_bwd}).

\subsection{Forward and Backward Algorithms}

\paragraph{Forward.}
Compute $(\eta_t,\log\gamma_t)$ (with clamping) and run the single-solve no-ridge rank-one WY kernel on chunk-locally renormalized inputs:
within each chunk compute local prefixes $u_i$ and $\delta_i=\exp(u_i)$, use $\hat\eta_t=\eta_t/\gamma_t$ and $\hat\vb_t=\vb_t/\delta_{i-1}$ inside that chunk, and rescale outputs/state by $\delta_i$ (Algorithm~\ref{alg:deltanet2_lite_fwd}).

\begin{algorithm}[H]
\caption{Falcon-1 (Forward)}
\label{alg:deltanet2_lite_fwd}
\begin{algorithmic}[1]
\Require Shifted write features $\xb\in\mathbb{R}^{L\times d}$ (boundary $\xb_1=\mathbf{0}$), queries $\Qb\in\mathbb{R}^{L\times d}$, values $\Vb\in\mathbb{R}^{L\times d_v}$, gains $\beta\in(0,2)^L$, ridge $\lambda\in\mathbb{R}_{\ge 0}^L$, stabilizer $\varepsilon>0$, clamp $\varepsilon_\gamma>0$, dummy $\varepsilon_\eta>0$, chunk size $C$ (assume $C\mid L$), initial state $\Sbb_{\rm init}$.
\Ensure Outputs $\Ob$ for the decayed recurrence with $\eta_t=\beta_t/(\|\xb_t\|_2^2+\lambda_t+\varepsilon)$, $\alpha_t:=\min(\eta_t\lambda_t,1-\varepsilon_\gamma)$, and $\gamma_t:=1-\alpha_t>0$. If the clamp is active, the shrinkage path corresponds to the effective coefficient $\alpha_t/\eta_t$ for $\eta_t>0$.
\State Compute $\eta_1\gets 0$, and for $t\ge 2$: $d_t\gets \|\xb_t\|_2^2+\lambda_t+\varepsilon$, $\eta_t\gets \beta_t/d_t$ (fp32 for norms/ratios).
\State $\alpha^{\rm raw}_t\gets \eta_t\lambda_t$, $\alpha_t\gets \min(\alpha^{\rm raw}_t,1-\varepsilon_\gamma)$ for $t\ge 2$.
\State $\log\gamma_1\gets 0$ and for $t\ge 2$: $\log\gamma_t\gets \operatorname{log1p}(-\alpha_t)$ (fp32), $\gamma_t\gets \exp(\log\gamma_t)$.
\State $\hat\eta_1\gets \varepsilon_\eta$ (dummy; no effect since $\xb_1=\mathbf{0}$), and for $t\ge 2$: $\hat\eta_t\gets \eta_t/\gamma_t=\eta_t\exp(-\log\gamma_t)$.
\State Partition into $M=L/C$ chunks $[a_k,b_k]=[(k{-}1)C{+}1,kC]$. Set $\Sbb_{\rm in}\gets \Sbb_{\rm init}$.
\For{$k=1,\ldots,M$}
\State $a\gets a_k$, $b\gets b_k$.
\State $u_0\gets 0$ \Comment{Local log-prefix decays inside the chunk}
\For{$i=1,\ldots,C$}
\State $u_i\gets u_{i-1}+\log\gamma_{a+i-1}$,\quad $\delta_i\gets \exp(u_i)$ \Comment{fp32}
\EndFor
\State $\hat\Vb_{a+i-1}\gets \Vb_{a+i-1}\exp(-u_{i-1})$ for $i=1,\ldots,C$ \Comment{$\hat\vb=\vb/\delta_{i-1}$}
\State Run the single-solve no-ridge rank-one WY kernel on the length-$C$ sequence $(\xb_{a:b},\Qb_{a:b},\hat\Vb_{a:b},\hat\eta_{a:b})$ with init $\Sbb_{\rm in}$ (Algorithm~\ref{alg:deltanet-parallel} with $L=C$), producing $(\hat\Ob_{a:b},\hat\Sbb_{\rm out})$.
\State $\Ob_{a+i-1}\gets \delta_i\,\hat\Ob_{a+i-1}$ for $i=1,\ldots,C$.
\State $\Sbb_{\rm in}\gets \delta_C\,\hat\Sbb_{\rm out}$ \Comment{Propagate original chunk-exit state}
\EndFor
\State \Return $\Ob$.
\end{algorithmic}
\end{algorithm}

\paragraph{Backward.}
Algorithm~\ref{alg:deltanet-parallel-bwd} provides gradients for the one-solve no-ridge rank-one kernel while treating $(\xb,\boldsymbol{\eta})$ as independent.
Falcon additionally backpropagates through the NLMS map $\eta_t=\beta_t/(\|\xb_t\|_2^2+\lambda_t+\varepsilon)$ and through the chunk-local log-decay renormalization (the local prefixes $u_i$ and $\delta_i$) used to handle $\gamma_t$ (Algorithm~\ref{alg:deltanet2_lite_bwd}).

\begin{breakablealgorithm}
\caption{Falcon-1 (Backward)}
\label{alg:deltanet2_lite_bwd}
\small
\begin{algorithmic}[1]
\Require Upstream gradient $\bar{\Ob}$, chunk-local caches from Algorithm~\ref{alg:deltanet2_lite_fwd} for each chunk $k$: local prefixes $\{u_i,\delta_i\}_{i=0}^C$, $\log\gamma_{a_k:b_k}$, $(\eta,\beta,\lambda)$, clamp mask $m_t=\mathbb{I}[\eta_t\lambda_t<1-\varepsilon_\gamma]$, and WY caches from running the no-ridge rank-one kernel on $(\xb_{a_k:b_k},\Qb_{a_k:b_k},\hat\Vb_{a_k:b_k},\hat\eta_{a_k:b_k})$.
\Ensure Gradients $(\bar{\Qb},\bar{\xb},\bar{\Vb},\bar{\beta},\bar{\lambda},\bar{\Sbb}_{\rm init})$.
\State Initialize all gradients to zero. Set boundary $\bar{\Sbb}_{\rm in}^{(M+1)}\gets \mathbf{0}$.
\For{$k=M,\ldots,1$} \Comment{reverse over chunks}
\State $a\gets a_k$, $b\gets b_k$.
\State Initialize $\bar{u}_i\gets 0$ for $i=0,\ldots,C$, $\bar{\log\gamma}_t\gets 0$ and $\bar{\eta}_t\gets 0$ for $t=a,\ldots,b$.
\State \Comment{Unscale outputs: $\Ob_t=\delta_i\,\hat\Ob_t$}
\For{$t=a,\ldots,b$}
\State $i\gets t-a+1$
\State $\bar{\hat\Ob}_t \gets \delta_i\,\bar{\Ob}_t$
\State $\bar{u}_i \mathrel{+}= \langle \bar{\hat\Ob}_t,\hat\Ob_t\rangle$ \Comment{$\partial \delta_i/\partial u_i=\delta_i$}
\EndFor
\State \Comment{Boundary: next chunk sees $\Sbb_{\rm in}^{(k+1)}=\delta_C\,\hat\Sbb_{\rm out}$}
\State $\bar{\hat\Sbb}_{\rm out} \gets \delta_C\,\bar{\Sbb}_{\rm in}^{(k+1)}$
\State $\bar{u}_C \mathrel{+}= \langle \bar{\hat\Sbb}_{\rm out},\hat\Sbb_{\rm out}\rangle$
\State Run the no-ridge rank-one WY backward on this length-$C$ chunk (Algorithm~\ref{alg:deltanet-parallel-bwd} with $L=C$), using upstream $(\bar{\hat\Ob}_{a:b},\bar{\hat\Sbb}_{\rm out})$, to obtain $(\bar{\Qb}_{a:b},\bar{\xb}_{a:b},\bar{\hat\Vb}_{a:b},\bar{\hat\eta}_{a:b},\bar{\Sbb}_{\rm in}^{(k)})$.
\State \Comment{Unscale values: $\hat\Vb_t=\Vb_t\,\exp(-u_{i-1})$}
\For{$t=a,\ldots,b$}
\State $i\gets t-a+1$
\State $s \gets \exp(-u_{i-1})$
\State $\bar{\Vb}_t \mathrel{+}= s\,\bar{\hat\Vb}_t$
\State $\bar{u}_{i-1} \mathrel{+}= -\langle s\,\bar{\hat\Vb}_t,\Vb_t\rangle$
\EndFor
\State \Comment{Unscale step sizes: $\hat\eta_t=\eta_t\,\exp(-\log\gamma_t)$}
\For{$t=\max(a,2),\ldots,b$}
\State $g^{-1}\gets \exp(-\log\gamma_t)$
\State $\bar{\eta}_t \mathrel{+}= g^{-1}\,\bar{\hat\eta}_t$
\State $\bar{\log\gamma}_t \mathrel{+}= -(\eta_t\,g^{-1})\,\bar{\hat\eta}_t$
\EndFor
\State \Comment{Backprop local prefix sums: $u_i=u_{i-1}+\log\gamma_{a+i-1}$}
\For{$i=C$ \textbf{down to} $1$}
\State $t\gets a+i-1$
\If{$t\ge 2$}
\State $\bar{\log\gamma}_t \mathrel{+}= \bar{u}_i$
\EndIf
\State $\bar{u}_{i-1} \mathrel{+}= \bar{u}_i$
\EndFor
\State \Comment{Backprop $\log\gamma_t=\operatorname{log1p}(-\alpha_t)$ with $\alpha_t=\min(\eta_t\lambda_t,1-\varepsilon_\gamma)$}
\For{$t=\max(a,2),\ldots,b$}
\State $\bar{\alpha}_t \gets -\exp(-\log\gamma_t)\,\bar{\log\gamma}_t$
\State $\bar{\alpha}_t \gets m_t\,\bar{\alpha}_t$ \Comment{zero gradient when clamped}
\State $\bar{\eta}_t \mathrel{+}= \lambda_t\,\bar{\alpha}_t$
\State $\bar{\lambda}_t \mathrel{+}= \eta_t\,\bar{\alpha}_t$
\EndFor
\State \Comment{Backprop $\eta_t=\beta_t/(\|\xb_t\|_2^2+\lambda_t+\varepsilon)$}
\For{$t=\max(a,2),\ldots,b$}
\State $d_t\gets \|\xb_t\|_2^2+\lambda_t+\varepsilon$
\State $\bar{\beta}_t \mathrel{+}= \bar{\eta}_t/d_t$
\State $\bar{d}_t \gets -\beta_t\,\bar{\eta}_t/d_t^2$
\State $\bar{\xb}_t \mathrel{+}= 2\,\bar{d}_t\,\xb_t$
\State $\bar{\lambda}_t \mathrel{+}= \bar{d}_t$
\EndFor
\State \Comment{Boundary constants $\hat\eta_1:=\varepsilon_\eta$ and $\log\gamma_1:=0$ carry no gradient.}
\EndFor
\State \Return $(\bar{\Qb},\bar{\xb},\bar{\Vb},\bar{\beta},\bar{\lambda},\bar{\Sbb}_{\rm init})$ with $\bar{\Sbb}_{\rm init}=\bar{\Sbb}_{\rm in}^{(1)}$.
\end{algorithmic}
\end{breakablealgorithm}

\subsection{Complexity Analysis}

Table~\ref{tab:complexity_lite} summarizes the per-chunk costs.
The residual formulation removes one forward TriSolve per chunk in both Falcon-2 and Falcon-1 kernels, but Falcon-1 retains the larger practical gain because it also shares the system build across all value channels.
Sharing the dynamics therefore eliminates the rate-dependent system construction overhead that scales as $\mathcal{O}(d_v C^2)$ in the full Falcon-2 (building/factorizing $d_v$ distinct $C\times C$ triangular systems).
The remaining triangular solve still scales as $\mathcal{O}(d_v C^2)$ because the state has $d_v$ columns, but it is a single highly optimized multi-RHS solve, which is substantially more GPU-friendly in practice.

\begin{table}[ht!]
\centering
\caption{Complexity per Chunk (chunk size $C$, feature dim $d$, value dim $d_v$).}
\label{tab:complexity_lite}
\begin{tabular}{l l l}
\toprule
\textbf{Component} & \textbf{Falcon-2 (Full)} & \textbf{Falcon-1} \\
\midrule
Gram Matrix & $\mathcal{O}(d C^2)$ & $\mathcal{O}(d C^2)$ \\
\textbf{Rate-dependent system build} ($\Lb$) & $\mathbf{\mathcal{O}(d_v C^2)}$ & $\mathbf{\mathcal{O}(C^2)}$ \\
\textbf{Forward residual TriSolve} & $\mathcal{O}(d_v C^2)$ & $\mathcal{O}(d_v C^2)$ \\
State update / output projection & $\mathcal{O}(d d_v C)$ & $\mathcal{O}(d d_v C)$ \\
\bottomrule
\end{tabular}
\end{table}

Falcon-1 removes the need to construct and factor $d_v$ distinct $C\times C$ systems by enforcing shared dynamics.
After the residual merge, the remaining forward solve is a single high-throughput multi-RHS triangular solve, which is substantially more GPU-friendly in practice.

%%%
\section{Falcon-3 ParallelFlow Implementation}
\label{app:deltanet3_parallel}

Falcon-3 uses a sliding regression window of size $B$, which turns the rank-one DeltaNet-style update into a rank-$B$ affine recurrence. To train this recurrence without materializing dense state-transition matrices or scanning token by token, we use the ParallelFlow framework~\citep{cirone2025parallelflow}.

This appendix states the chunk map used in the scan, summarizes the low-rank controlled-differential-equation view, and maps the Falcon-3 update to the \texttt{tensorInv} solve used in Algorithm~\ref{alg:Falcon3_forward}.

\paragraph{Mask conventions.}
Within \texttt{tensorInv}, we use a block-strict-causal mask over $(\text{time},\text{rank})$ pairs, so same-time rank components do not interact, and every residual is evaluated at the pre-update state $\Sbb_{t-1}$. Output materialization then uses an inclusive block-causal mask to implement read-after-write inside each chunk.

\subsection{Affine Chunk Map}
\label{app:Falcon3_chunk_map}

To make the scan claim explicit, define for chunk $k$
\[
\mathbf{M}^{(k)} := \delta_C^{(k)}\big(\Ib_{d_x}+\mathbf{A}^{(k)}\mathbf{W}^{(k)}\big),
\qquad
\mathbf{b}^{(k)} := \delta_C^{(k)}\mathbf{A}^{(k)}\mathbf{U}^{(k)}.
\]
Then the boundary update is the affine map
\[
\Sbb_{\rm in}^{(k+1)} = \mathbf{M}^{(k)}\Sbb_{\rm in}^{(k)} + \mathbf{b}^{(k)}.
\]
These chunk maps compose associatively,
\[
(\mathbf{M}_2,\mathbf{b}_2)\circ(\mathbf{M}_1,\mathbf{b}_1)
=
(\mathbf{M}_2\mathbf{M}_1,\,\mathbf{M}_2\mathbf{b}_1+\mathbf{b}_2),
\]
so Phase~2 of Algorithm~\ref{alg:Falcon3_forward} may be implemented either as the simple sequential loop shown in the main text or as an associative scan over chunk maps.

\subsection{Matrix-Valued CDEs and Low-Rank Drivers}

Standard linear recurrences can be viewed as discretizations of a continuous-time process. Let $\Sbb_t \in \mathbb{R}^{d \times d_v}$ be the hidden state. We adopt the left-multiplicative convention and model its evolution as a matrix-valued CDE driven by $\boldsymbol{\omega}$ and $\boldsymbol{\xi}$:
\begin{equation}
\label{eq:cde_main}
d\Sbb_t = d\boldsymbol{\omega}_t \Sbb_t + d\boldsymbol{\xi}_t,
\end{equation}
where $\boldsymbol{\omega}_t \in \mathbb{R}^{d \times d}$ and $\boldsymbol{\xi}_t \in \mathbb{R}^{d \times d_v}$ are matrix-valued paths. The solution over an interval $[s, t]$ factors through a linear propagator (flow) $\Pb_{t\leftarrow s}$:
\begin{equation}
\Sbb_t = \Pb_{t\leftarrow s}\Sbb_s + \int_s^t \Pb_{t\leftarrow r}\, d\boldsymbol{\xi}_r,
\quad \text{where } \Pb_{t\leftarrow s} = \Ib + \int_s^t d\boldsymbol{\omega}_r \Pb_{r\leftarrow s}.
\end{equation}
This decomposition separates the temporal dynamics ($\Pb$) from the computation. ParallelFlow parallelizes this by partitioning the sequence into chunks $[t_{k-1}, t_k]$. The system computes local propagators and accumulated input injections for each chunk in parallel, then links them via a global associative scan.

\paragraph{Transpose-equivalent form.}
ParallelFlow is often written after transposing the state. Defining $\widehat{\Sbb}_t:=\Sbb_t^\top\in\mathbb{R}^{d_v\times d}$, the same low-rank update becomes the shared-right-factor recurrence
\[
\widehat{\Sbb}_{t+1}
= \widehat{\Sbb}_t
+ \widehat{\Sbb}_t\,\mathbf{B}_t\mathbf{A}_t^\top
+ \tilde{\mathbf{B}}_t\mathbf{A}_t^\top.
\]
Transposing back yields the shared-left-factor form
\[
\Sbb_{t+1} = \Sbb_t + \mathbf{A}_t\mathbf{B}_t^\top\Sbb_t + \mathbf{A}_t\tilde{\mathbf{B}}_t^\top.
\]
We use the shared-left-factor form throughout this appendix because the sliding regression update naturally shares the window matrix of write features.

\paragraph{Low-Rank Drivers.}
The core difficulty lies in computing the propagator $\Pb_{s \to t}$, which is typically an expensive matrix ODE. However, efficient computation is possible if the drivers possess a low-rank structure. We assume rank-$R$ drivers with a shared left factor:
\begin{equation}
\label{eq:low_rank_drivers}
d\boldsymbol{\omega}_t = \vec{\Ab}_t \vec{\Bb}_t^\top dt, \qquad
d\boldsymbol{\xi}_t = \vec{\Ab}_t \tilde{\vec{\Bb}}_t^\top dt,
\end{equation}
with $\vec{\Ab}_t,\vec{\Bb}_t\in\mathbb{R}^{d\times R}$ and $\tilde{\vec{B}}_t\in\mathbb{R}^{d_v\times R}$. This structure allows for an efficiently computable, compact representation of the flow.
For discrete tokens, under a forward-Euler discretization (with the step size absorbed into $\vec{B},\tilde{\vec{B}}$, equivalently $\Delta t=1$), the update becomes:
\begin{equation}
\Sbb_{t_{k+1}} = \Sbb_{t_k} + \vec{\Ab}_{t_k}\vec{\Bb}_{t_k}^\top \Sbb_{t_k} + \vec{\Ab}_{t_k}\tilde{\vec{\Bb}}_{t_k}^\top.
\end{equation}

\paragraph{The \texttt{tensorInv} Algorithm.}
Solving this low-rank system over a chunk of length $L_c$ can be reduced to a Triangular Tensor Inversion. Let $\mathcal{C} \in \mathbb{R}^{(L_c \times R) \times (L_c \times R)}$ be a tensor representing causal interactions between rank-components. For indices $s, t \in \{1, \dots, L_c\}$ and ranks $i, j \in \{1, \dots, R\}$:
\begin{equation}
[\mathcal{C}]_{t, i}^{s, j} = \delta_{s,t}\delta_{i,j} - \mathbb{I}(s < t) \left[ \vec{\Bb}_{t}^\top \vec{\Ab}_{s} \right]_{i, j}.
\end{equation}
\citet{cirone2025parallelflow} shows that the propagator can be computed by applying the structured inverse tensor $\mathcal{D}=\mathcal{C}^{-1}$ under the tensor-contraction product, thereby avoiding explicit dense $d_x\times d_x$ chunk propagators. In our setting, the computation is governed by an $(L_cR)\times(L_cR)$ structured causal solve together with matrix multiplications involving $d_x$ and $d_v$; the compressed asymptotic $\mathcal{O}(L_c^2R+d)$ is therefore misleading here because it hides the dominant dimension-dependent terms.

\paragraph{\texttt{tensorInv} outputs (the objects used in Algorithm~\ref{alg:Falcon3_forward}).}
In practice, we do not materialize the full inverse tensor $\mathcal{D}$.
Instead, ParallelFlow applies $\mathcal{C}^{-1}$ to two right-hand sides corresponding to the low-rank drivers.
For a chunk of length $L_c$ and rank $R$, stack (flatten time$\times$rank) the discrete drivers as
\[
\vec{\Ab}\in\mathbb{R}^{d_x\times (L_cR)},\qquad
\vec{\Bb}\in\mathbb{R}^{d_x\times (L_cR)},\qquad
\tilde{\vec{B}}\in\mathbb{R}^{d_v\times (L_cR)}.
\]
Let $\vec{M}\in\{0,1\}^{(L_cR)\times(L_cR)}$ denote the strictly-causal block mask that is $1$ for interactions from earlier times $s<t$ and $0$ otherwise; equivalently, in $(t,i)$ indexing it is exactly the indicator $\mathbb{I}(s<t)$ used in the definition of $\mathcal{C}$ above. In particular, after flattening time$\times$rank, same-time rank components do not interact inside the triangular solve.
Then the chunk-local triangular system can be written compactly as
\[
\mathcal{C} = \mathbf{Id} - \vec{M}\odot(\vec{B}^\top \vec{A}),
\]
and the two solves returned by \texttt{tensorInv} are
\begin{equation}
\label{eq:pf_WU_def}
\vec{W} := \mathcal{C}^{-1}\vec{B}^\top \in \mathbb{R}^{(L_cR)\times d_x},
\qquad
\vec{U} := \mathcal{C}^{-1}\tilde{\vec{B}}^\top \in \mathbb{R}^{(L_cR)\times d_v}.
\end{equation}
We denote these solutions by
\[
(\vec{W},\vec{U})=\texttt{tensorInv}(\vec{A},\vec{B},\tilde{\vec{B}}).
\]
Given an incoming chunk-boundary state $\Sbb_{\rm in}\in\mathbb{R}^{d_x\times d_v}$, define
$\vec{Z}:=\vec{U}+\vec{W}\Sbb_{\rm in}\in\mathbb{R}^{(L_cR)\times d_v}$.
Then the chunk-exit state is
\begin{equation}
\label{eq:pf_chunk_update}
\Sbb_{\rm out}=\Sbb_{\rm in}+\vec{A}\vec{Z},
\end{equation}
which is the form used by Algorithm~\ref{alg:Falcon3_forward} (with the additional scalar decay $\gamma_t$ handled there via chunk-local renormalization).

\subsection{Mapping Falcon-3 to ParallelFlow}

We now explicitly map the \textbf{Falcon-3} (Sliding Regression) update rule derived in Section~\ref{sec:methods} to this low-rank CDE framework.

\paragraph{The Falcon-3 Recurrence.}
Recall the mini-batch update with window size $B$ (ignoring the scalar decay $\gamma_t = 1 - \eta_t \lambda_t$, which our chunked implementations handle via the same chunk-local log-decay renormalization used throughout):
\begin{align}
\label{eq:deltanet3_recurrence}
\Sbb_t &= \underbrace{\left( \Ib - \eta_t \bar{\Cb}_t^{(B)} \right)}_{\text{Transition Matrix}} \Sbb_{t-1} + \underbrace{\eta_t \bar{\Nb}_t^{(B)}}_{\text{Input Injection}} \nonumber \\
&= \left( \Ib - \frac{\eta_t}{B_t}\sum_{j \in \mathcal{I}_t} \mathbf{x}_{j}\mathbf{x}_{j}^\top \right) \Sbb_{t-1} + \frac{\eta_t}{B_t}\sum_{j \in \mathcal{I}_t} \mathbf{x}_{j}\vb_{j}^\top,
\end{align}
where $B_t:=|\mathcal{I}_t|\le B$, $\mathbf{x}_j:=\phi(\mathbf{k}_{j-1})$ is the shifted write feature, $\bar{\Cb}_t^{(B)}:=\Cb_t^{(B)}/B_t$, and $\bar{\Nb}_t^{(B)}:=\Nb_t^{(B)}/B_t$. The unkernelized case has $\phi$ equal to the identity.

\paragraph{Identification of Drivers.}
To match the fixed-rank \texttt{tensorInv} kernel used in Algorithm~\ref{alg:Falcon3_forward}, we zero-pad every active window to width $B$. Let $\mathcal{I}_t=\{j \mid \max(2,t-B+1)\le j\le t\}$, let $B_t:=|\mathcal{I}_t|\le B$, and enumerate $\mathcal{I}_t$ increasingly as $(j_1,\dots,j_{B_t})$. Define the padded write-feature and value blocks
\begin{align}
\widetilde{\vec X}_t
&:=
\begin{bmatrix}
\mathbf{x}_{j_1} & \mathbf{x}_{j_2} & \dots & \mathbf{x}_{j_{B_t}} & \mathbf{0} & \dots & \mathbf{0}
\end{bmatrix}
\in \mathbb{R}^{d_x\times B},\\
\widetilde{\vec V}_t
&:=
\begin{bmatrix}
\vb_{j_1} & \vb_{j_2} & \dots & \vb_{j_{B_t}} & \mathbf{0} & \dots & \mathbf{0}
\end{bmatrix}
\in \mathbb{R}^{d_v\times B},
\end{align}
where padded columns are zero. Then, for $t\ge 2$,
\[
\bar{\Cb}_t^{(B)}=\frac{1}{B_t}\widetilde{\vec X}_t\widetilde{\vec X}_t^\top,
\qquad
\bar{\Nb}_t^{(B)}=\frac{1}{B_t}\widetilde{\vec X}_t\widetilde{\vec V}_t^\top,
\]
since the padded columns vanish identically.
Substituting into Eq.~\eqref{eq:deltanet3_recurrence} yields the canonical left-multiplicative low-rank form
\begin{equation}
\Sbb_t
= \Sbb_{t-1}
+ \underbrace{\Big(\frac{\eta_t}{B_t}\widetilde{\vec X}_t\Big)}_{\vec{A}_t}
\underbrace{(-\widetilde{\vec X}_t^\top)}_{\vec{B}_t^\top} \Sbb_{t-1}
+ \underbrace{\Big(\frac{\eta_t}{B_t}\widetilde{\vec X}_t\Big)}_{\vec{A}_t}
\underbrace{\widetilde{\vec V}_t^\top}_{\tilde{\vec{B}}_t^\top},
\qquad (t\ge 2).
\end{equation}
Thus the \texttt{tensorInv} kernel always uses the fixed rank parameter $R=B$ after padding:
\[
\vec{A}_t = \frac{\eta_t}{B_t}\widetilde{\vec X}_t \in \mathbb{R}^{d_x\times B},\qquad
\vec{B}_t = -\widetilde{\vec X}_t \in \mathbb{R}^{d_x\times B},\qquad
\tilde{\vec{B}}_t = \widetilde{\vec V}_t \in \mathbb{R}^{d_v\times B}.
\]
For the boundary step $t=1$, we set $\vec{A}_1=\mathbf{0}$ so the update is a strict no-op. As in the main text, the stacked triangular solve uses the block-strict-causal mask over $(\text{time},\text{rank})$ pairs, so same-time padded-window columns do not interact, and every residual is still evaluated at the pre-update state $\Sbb_{t-1}$.
In Algorithm~\ref{alg:Falcon3_forward}, after the positive-decay renormalization of Eq.~\eqref{eq:deltanet3_residual_renorm}, one simply replaces $\eta_t$ by $\hat\eta_t$ and rescales the entire step-$t$ value block by the common factor $(\delta^{(k)}_{i-1})^{-1}$ inside chunk $k$.

\paragraph{Complexity.}
Because the discrete Falcon-3 update is exactly a rank-$B$ low-rank affine recurrence after padding (and may be viewed as a forward-Euler discretization of a corresponding CDE), we can utilize the \texttt{tensorInv} algorithm without approximation at the discrete level. For the small sliding windows used in our experiments (e.g., $B=4$), the rank-$B$ overhead is modest in practice and remains far cheaper than full $\mathcal{O}(L^2)$ attention while avoiding dense $d_x\times d_x$ state propagation. This derivation shows that Falcon-3 can use chunk-parallel scan methods while retaining the sliding-window regressor update exactly at the discrete level.

\end{document}